\pdfoutput=1
\documentclass[twoside]{article}
\usepackage[accepted]{aistats2021}
\usepackage{microtype}
\usepackage{graphicx}
\usepackage{subfigure}
\usepackage{algorithm}
\usepackage{algorithmic}
\usepackage{booktabs} % for professional tables
\usepackage{color}
\usepackage{subfigure}
\usepackage{enumitem}
\usepackage{natbib}

\usepackage{caption}
\usepackage{graphicx}
\usepackage{amsthm}

\usepackage{mathtools}

\usepackage{tikz}
\newcommand*\circled[1]{\tikz[baseline=(char.base)]{
            \node[shape=circle,draw,inner sep=0.4pt] (char) {#1};}}
\newcommand{\RomanNumeralCaps}[1]
    {\MakeUppercase{\romannumeral #1}}
\usepackage{amsmath,amsfonts,bm}

\newtheorem{theorem}{Theorem}
\newtheorem{lemma}{Lemma}

\newtheorem{corollary}{Corollary}

\newtheorem{assumption}{Assumption}

\def\eqref#1{equation~\ref{#1}}
\def\1{\bm{1}}

\DeclareMathAlphabet{\mathsfit}{\encodingdefault}{\sfdefault}{m}{sl}
\SetMathAlphabet{\mathsfit}{bold}{\encodingdefault}{\sfdefault}{bx}{n}

\def\gC{{\mathcal{C}}}

\def\gL{{\mathcal{L}}}

\def\gN{{\mathcal{N}}}

\def\gS{{\mathcal{S}}}

\def\gZ{{\mathcal{Z}}}

\def\sR{{\mathbb{R}}}

\newcommand{\E}{\mathbb{E}}

\newcommand{\supp}{\text{supp}}
\DeclareMathOperator*{\argmin}{arg\,min}

\usepackage{hyperref}
\usepackage{autonum}
\usepackage{amsmath}
\usepackage{wrapfig}

\usepackage[utf8]{inputenc} % allow utf-8 input
\usepackage[T1]{fontenc}    % use 8-bit T1 fonts
\usepackage{hyperref}       % hyperlinks
\usepackage{url}            % simple URL typesetting
\usepackage{booktabs}       % professional-quality tables
\usepackage{amsfonts}       % blackboard math symbols
\usepackage{nicefrac}       % compact symbols for 1/2, etc.
\usepackage{microtype}      % microtypography
\usepackage{amsmath}

\runningtitle{Bayesian Coresets: Revisiting the Nonconvex Optimization Perspective}

\runningauthor{Jacky Y. Zhang, Rajiv Khanna, Anastasios Kyrillidis, Oluwasanmi Koyejo}
\begin{document}
\twocolumn[

\aistatstitle{Bayesian Coresets: \\ Revisiting the Nonconvex Optimization Perspective}

\aistatsauthor{ Jacky Y. Zhang$^1$ \And Rajiv Khanna$^2$ \And Anastasios Kyrillidis$^3$ \And Oluwasanmi Koyejo$^1$ }
% \aistatsaddress{ UIUC \And UC Berkeley \And Rice University \And UIUC} 
\aistatsaddress{
\texttt{yiboz@illinois.edu} $\quad \ \ $ \texttt{rajivak@berkeley.edu} $\quad \ \ $ \texttt{anastasios@rice.edu} $\quad \ \ $
\texttt{sanmi@illinois.edu}\\
$^1$University of Illinois at Urbana-Champaign $\qquad$ $^2$University of California, Berkeley $\qquad$ $^3$Rice University}
]

%\maketitle

\begin{abstract}
Bayesian coresets have emerged as a promising approach for implementing scalable Bayesian inference. 
The Bayesian coreset problem involves selecting a (weighted) subset of the data samples, such that the posterior inference using the selected subset closely approximates the posterior inference using the full dataset. 
This manuscript revisits Bayesian coresets through the lens of sparsity constrained optimization. 
Leveraging recent advances in accelerated optimization methods, we propose and analyze a novel algorithm for coreset selection. 
We provide explicit convergence rate guarantees and present an empirical evaluation on a variety of benchmark datasets to highlight our proposed algorithm's superior performance compared to state-of-the-art on speed and accuracy.
\end{abstract}

\section{Introduction}
\label{sec:introduction}

Bayesian coresets have emerged as a promising approach for scalable Bayesian inference~\citep{huggins2016coresets, campbell2018bayesian, campbell2019automated, campbell2019sparse}.
The key idea is to select a (weighted) subset of the data such that the posterior inference using the selected subset closely approximates the posterior inference using the full dataset.
This creates a trade-off, where using Bayesian coresets as opposed to the full dataset exchanges approximation accuracy for computational speedups. 
%Compare this to running the inference algorithm on full data, since the complexity of such inference algorithms often scales at least linearly in the dataset size.
We study Bayesian coresets as they are easy to implement, effective in practice, and come with useful theoretical guarantees that relate the coreset size with the approximation quality. 

The main technical challenge in the Bayesian coreset problem lies in handling the combinatorial constraints -- we desire to select a few data points out of many  as the coreset. % sparsity is non-convex.
In terms of optimization, previous approaches mainly rely on two ideas: \emph{convexification} and \emph{greedy methods}.
In convexification~\citep{campbell2019automated}, the sparsity constraint -- i.e., selection of $k$ data samples -- is relaxed into a convex $\ell_1$-norm constraint. 
This allows them to use out-of-the-box solvers such as Frank-Wolfe (FW) type-of methods \citep{frank1956algorithm, jaggi2013revisiting}.
An alternative approach is by using greedy methods~\citep{campbell2018bayesian}, which constructs a sparse weight vector based on local decisions to greedily optimize the approximation problem~\citep{tropp2007signal, needell2009cosamp}. The resulting method, greedy iterative geodesic ascent (GIGA), achieves linear convergence with no hyper-parameter tuning and optimal scaling~\citep{campbell2018bayesian}. More recently, sparse variational inference (SparseVI) is considered for Bayesian coreset construction. SparseVI also employs a greedy algorithm to minimize a KL divergence objective. The method achieves state-of-the-art accuracy, but at a cost of higher computational requirements. Therefore, existing work illustrates the trade-off between accuracy and efficiency, opening a gap for improvements.

We revisit Bayesian coresets through the lens of sparsity constrained optimization. 
Sparsity, a kind of nonconvexity, appears in a variety of applications in machine learning and statistics. 
For instance, compressed sensing~\citep{donoho2006compressed, candes2008restricted} is an example where sparsity is used as a complexity measure for signal representation. Leveraging and building upon recent advances in non-convex optimization, we solve the Bayesian coreset problem based on hard thresholding algorithms~\citep{blumensath2009iterative} that directly work on the non-convex sparsity constraint. Hard-thresholding schemes are highly flexible, and easily accommodate variations such as subspace exploration~\citep{dai2009subspace}, de-bias steps \citep{needell2009cosamp}, adaptive step size selections \citep{kyrillidis2011recipes}, as well as different types of sparsity constraints, such as group sparsity \citep{baldassarre2016group}, sparsity within groups \citep{kyrillidis2015structured}, and generic structured sparsity \citep{baraniuk2010model}. The thresholding step involves a projection onto the $k$-sparsity constraint set to determine the selected sample set in each iteration. 
%which is to be made in every iteration to possibly re-build the currently selected set of data points. 
While we achieve state-of-the-art accuracy using direct application of this algorithm, re-building the set in every iteration makes it slower than previous works. To fix this, we employ line search for step size selection and momentum based techniques~\citep{khanna2018iht} to accelerate the algorithm, also achieving state-of-the-art speed.

\textbf{Contributions.}
In this paper, we adapt accelerated iterative hard thresholding schemes to the Bayesian coreset problem.  
Despite directly attacking the non-convex optimization problem, we provide strong convergence guarantees. 
To summarize our contributions:  \vspace{-0.6cm}
\begin{itemize}[leftmargin=0.5cm]
\item We revisit the Bayesian coreset problem via a non-convex (sparse) optimization lens, and provide an IHT-based algorithm that combines hard thresholding and momentum steps; \vspace{-0.1cm}
\item We analyze its convergence %akin to optimization guarantees, 
based on standard assumptions; \vspace{-0.1cm}
\item We provide extensive empirical evaluation\footnote{Code available at \url{https://github.com/jackyzyb/bayesian-coresets-optimization}} 
to show superior performance of the proposed method vis-\`a-vis state-of-the-art algorithms in terms of approximation accuracy as well as speed.\vspace{-0.1cm}
\end{itemize}

\section{Problem Formulation}
\label{sec:background}

Given $n$ observations, one can compute the log-likelihood $\gL_i(\theta)$ of each of the observations, parameterized by $\theta$. Assuming observations are conditionally independent given $\theta$, one can represent the likelihood of all the observations as the sum of individual log-likelihoods, i.e., $\gL(\theta) = \sum_{i=1}^n \gL_i(\theta)$. With prior density $\pi_0(\theta)$, the posterior density can be derived as:
\begin{align}
    \pi(\theta) := \tfrac{1}{Z} \cdot &e^{\gL(\theta)} \cdot \pi_0(\theta), 
\end{align}
where $ Z=\int e^{\gL(\theta)}\pi_0(\theta)  d\theta  $ is a normalization factor. 

However, for most applications, exact posterior estimation is intractable; \emph{i.e.}, $\pi$ is too hard to evaluate exactly. 
Practitioners use algorithms for approximate inference that may approximate the $\pi$ in a closed-form (\emph{e.g.}, using variational inference), or allow for sampling from the posterior without providing a closed-form expression (\emph{e.g.}, MCMC methods). 
Such algorithms often scale at least linearly with the size of the dataset $n$, which makes them prohibitively expensive for large datasets. As such, designing algorithms to speed up inference is an area of active research.

One solution to the scalability problem is to use coresets.
%Here is where the idea of Bayesian coresets comes to rescue. 
Coresets approximate the empirical log-likelihood $\gL=\sum_{i=1}^n\gL_i$ using a %(potentially)
\emph{weighted sum of a subset of all the log-likelihoods $\gL_i$}. In other words, we use
%\begin{align}
    $\gL_w=\sum_{i=1}^n w_i\gL_i$
%\end{align}
to approximate the true $\gL$,  where $w \in \mathbb{R}_+^n$ is a non-negative sparse vector. It will be useful to view that $\gL, \gL_i$ and $\gL_w$ are functions in a Hilbert space, and we will use $L^2$-norm to denote the 2-norm defined in function space, differentiating with the $\ell_2$-norm defined in Euclidean space.
We enforce the sparsity constraint as $\|w\|_0 \leq k$, for $k < n$; here $\|\cdot\|_0$ denotes the pseudo-norm that counts the number of non-zero entries.

When $k < n$, posterior estimation (e.g., using MCMC or variational inference) is less expensive on the coreset as opposed to the entire dataset. However, sparsifying $w$ involves dropping some samples, which in turn implies deviating from the best performance possible from using the full dataset. The Bayesian coreset problem is formulated to minimize this loss in performance.

\textbf{The Bayesian Coreset Problem.} \emph{The Bayesian coreset problem is to control the deviation of coreset log-likelihood from true log-likelihood via sparsity:}
\begin{equation}
\begin{aligned}
& \underset{w \in \mathbb{R}^n}{\argmin}
& & f(w) := \textsc{Dist}(\mathcal{L},\; \mathcal{L}_w) \\
&  \qquad \text{s.t.}
& & \|w\|_0 \leq k, \; w_i \geq 0, \forall i. \vspace{-0.1cm}
\end{aligned} \label{eq:abstract_opt}
\end{equation}
Key components are $(i)$ the weights $w \in \mathbb{R}_+^n$ over $n$ data points, $(ii)$ the function $f(\cdot)$ that controls the deviation between the full-dataset log-likelihood $\mathcal{L}$ and the coreset log-likelihood $\mathcal{L}_w$ using the distance functional $\textsc{Dist}(\cdot,\; \cdot)$, and $(iii)$ the non-convex sparsity constraint that restricts the number of nonzeros in $w$, thus constraining the number of active data points in the coreset.
Examples of $\textsc{Dist}(\cdot,\; \cdot)$ include the weighted $L^2$-norm \citep{campbell2019automated} and the KL-divergence \citep{campbell2019sparse}. In this manuscript, we consider the $L^2(\hat\pi)$-norm as the distance metric in the embedding Hilbert space, \emph{i.e.}, 
\begin{align}
    \textsc{Dist}(\mathcal{L},\; \mathcal{L}_w)^2 &= \|\mathcal{L}- \mathcal{L}_w\|^2_{\hat\pi, 2}\\
    &=\E_{\theta \sim \hat\pi}\left[(\mathcal{L}(\theta)- \mathcal{L}_w(\theta))^2 \right],\label{eq:dist-1}
\end{align}
where $\hat\pi$ is a weighting distribution that has the same support as true posterior $\pi$. Ideally, $\hat{\pi}$ is the true posterior, which is obviously unknown. However, one can employ Laplace approximation to derive an inexpensive and reasonable approximation for $\hat{\pi}$~\citep{campbell2019automated}.

To account for the shift invariance, we write $g_i=\gL_i-\E_{\theta\sim \hat\pi} \gL_i(\theta)$, so the equivalent optimization problem is now: minimize $\|\sum_{i=1}^n g_i - \sum_{i=1}^n w_i g_i\|^2_{\hat\pi, 2}$. 
Further, noting that the $L^2(\hat\pi)$-norm is in the form of expectation (equation~(\ref{eq:dist-1})), it can be approximated by a finite-dimensional $\ell_2$-norm which replaces the function with a vector of sampled evaluations $\theta\sim \hat \pi$, \emph{i.e.}, its Monte Carlo approximation. Thus, given $S$ samples $\{\theta_j\}_{j=1}^S, \theta_j\sim \hat\pi$, and using 
\[ \hat{g}_i = \tfrac{1}{\sqrt{S}} \cdot \left[\gL_i(\theta_1) - \bar{\gL_i}, \dots, \gL_i(\theta_S) - \bar{\gL_i}\right]^\top \in \mathbb{R}^S \]
as projections from function space to standard Euclidean space, where $\bar{\gL_i}=\frac{1}{S}\sum_{j=1}^S \gL_i(\theta_j)$, the Bayesian coreset problem~(\ref{eq:abstract_opt}) becomes a \emph{finite-dimensional sparse regression problem:}
\begin{align}
\begin{aligned}
& \underset{w \in \mathbb{R}^{n}} {\arg \min } \quad 
f(w) := \left\|\sum_{i=1}^{n} \hat{g}_{i}-\sum_{i=1}^{n} w_{i} \hat{g}_{i}\right\|_{2}^{2}  \\
& \quad \text { s.t. } \quad \|w\|_{0} \leq k, \quad w_i \geq 0, \forall i.
\end{aligned}\label{prob:l2}
\end{align}
The resulting sparse regression problem is non-convex due to the combinatorial nature of the constraints. Previous methods that use this $\ell_2$-norm formulation~\citep{campbell2019automated,campbell2018bayesian} offers less satisfactory approximation accuracy compared to the state-of-the-art sparse variational inference method~\citep{campbell2019sparse}. However, the high computational cost of the latter method makes it impractical for real-world large datasets. Nonetheless, as we will show, our approach for solving equation~(\ref{prob:l2}) using a variant of iterative hard thresholding, achieves better accuracy and speed.

\section{Our approach}

\vspace{-0.3cm}
\begin{algorithm}[H]
	\begin{algorithmic}[1]
		\INPUT{Objective $f:\mathbb{R}^n\to \mathbb{R}$; sparsity $k$; step size $\mu$}
		\STATE Initialize $w$
		\REPEAT
		\STATE $w \leftarrow \Pi_{\mathcal{C}_k\cap \sR^n_+}\left(w-\mu \nabla f(w)\right) $
		\UNTIL Stop criteria met
		\STATE {\bfseries return} $w$
	\end{algorithmic}
		\caption{Vanilla IHT} 
	\label{alg:iht}
	\end{algorithm} \vspace{-0.3cm}
For clarity of exposition, we gradually build up our approach for solving the optimization problem~(\ref{prob:l2}). 
The fundamental ingredient of our approach is the vanilla Iterative Hard Thresholding (IHT) method presented in Algorithm~\ref{alg:iht}.
We develop our approach by augmenting IHT with momentum updates, step size selection for line search and active subspace expansion techniques to accelerate and automate the algorithm (Algorithms~\ref{alg:a-iht} \&~\ref{alg:a-iht-2}). 
Details follow.

\subsection{Iterative Hard Thresholding (IHT)}\label{subsec:method-iht}

The classical IHT~\citep{blumensath2009iterative} is a projected gradient descent method that performs a gradient descent step and then projects the iterate onto the non-convex $k$-sparsity constraint set.
We denote the orthogonal projection of a given $z \in \mathbb{R}^{n}$ to a space $\mathcal{C}\subseteq \mathbb{R}^n$ as: %$\Pi_\mathcal{C}(z)$:
%\begin{equation}  
$\Pi_\mathcal{C}(z) := \argmin_{w \in \mathcal{C}} \|w - z\|_2$.
%\end{equation}
Define the sparsity restricted space as:
%\begin{align}
$\mathcal{C}_k = \big \lbrace w \in \mathbb{R}^{n} :  \left| \supp(w) \right|\le k  \big \rbrace$, where $\supp(w)=\{i|w_i\neq 0\}$ denotes the support set of $w$.
%\end{align}
%With a little abuse of notation, for a set $\mathcal{C}\subseteq [n]$, we use projection $\Pi_\mathcal{C}(z)$ as
%\begin{align}
% \Pi_\mathcal{C}(z) = \argmin_{w: \supp(w)\subseteq\mathcal{C} } \|w - z\|_2.
%\end{align}
Here, we describe the plain sparsity case, but one can consider different realizations of $\mathcal{C}_k$ as in \citep{baldassarre2016group, kyrillidis2015structured, baraniuk2010model}.
The projection step in the classical IHT, \emph{i.e.,}, $\Pi_{\gC_k}$, can be computed easily by selecting the top-$k$ elements in $O(n\log k)$ time; but projection can be more challenging for more complex constraint sets, e.g., if the variable is a distribution on a lattice~\citep{Zhang2019IHT}. 

For our problem, we require that the projected sparse vector only has non-negative values. For vector variate functions, the projection step in Algorithm~\ref{alg:iht}, i.e., $\Pi_{\mathcal{C}_k \cap \mathbb{R}^n_+}(w)$ is also straightforward; it can be done optimally in $O(n\log k)$ time by simply picking the top $k$ largest \emph{non-negative} elements. More discussions about the projections are presented in section~\ref{sec:theory} in appendix.

\subsection{Accelerated IHT}\label{subsec:acc-iht}

For clarity, we rewrite the problem in equation~(\ref{prob:l2}) as:
\begin{align}\label{prob:l2-iht-obj}
w^*=\argmin_{w\in \mathcal{C}_k\cap \mathbb{R}^n_+} \quad f(w) := \|y-\Phi w\|_2^2, 
\end{align}
where $y=\sum_{i=1}^n \hat{g}_{i}$ and $\Phi = [\hat{g}_{1},\dots, \hat{g}_{n}]$.
In this case, $\nabla f(w) \equiv -2 \Phi^\top (y - \Phi w)$. %We denote the optimal solution for the problem as $w^*$, \emph{i.e.}, $w^*=\argmin_{w\in \mathcal{C}_k\cap \mathbb{R}^n_+}  f(w)$.

\noindent
\textbf{Step size selection in IHT:}
Classical results on the performance of IHT algorithms come with rigorous convergence guarantees (under regularity conditions) \citep{blumensath2009iterative, foucart2011hard}. 
However, these results require step size assumptions that either do not work in practice, or rely on strong assumptions. 
For example, in \citep{blumensath2009iterative, foucart2011hard} strong isometry constant bounds are assumed to allow step size $\mu = 1$ for all the iterations, and thus remove the requirement of hyper-parameter tuning. 
Moreover, the authors in \citep{blumensath2010normalized} present toy examples by carefully selecting $\Phi$ so that the vanilla IHT algorithm diverges without appropriate step size selection. In this work, given the quadratic objective $f(w)$, we perform exact line search to obtain the best step size per iteration~\citep{blumensath2010normalized, kyrillidis2011recipes}: $\mu_t := \|\widetilde{\nabla}_t\|_2^2 / 2\|\Phi \widetilde{\nabla}_t\|^2_2$; details in Algorithm \ref{alg:a-iht}.

    \begin{algorithm}[t]
	\caption{Automated Accelerated IHT (A-IHT)}\label{alg:a-iht}
	\begin{algorithmic}[1]
		\INPUT{Objective $f(w)=\|y-\Phi w\|_2^2$; sparsity $k$}
		\STATE $t=0$, $z_0=0$, $w_0=0$
		\REPEAT
		\STATE $\mathcal{Z} = \text{supp}(z_t)$
		\STATE $\mathcal{S}=\supp(\Pi_{\mathcal{C}_k \setminus \mathcal{Z}} \left( \nabla f(z_t) \right)) \cup \mathcal{Z}$ where $|\mathcal{S}| \leq 3k$
		\STATE $\widetilde{\nabla}_t = \nabla f(z_t) \big|_{\mathcal{S}}$
		\STATE $\mu_t=\argmin_\mu f(z_t-\mu \widetilde{\nabla}_t )= \frac{\|\widetilde{\nabla}_t\|_2^2}{2\|\Phi \widetilde{\nabla}_t\|^2_2}$
		\STATE $w_{t+1} = \Pi_{\mathcal{C}_k \cap\mathbb{R}^n_+}\left(z_t-\mu_t \nabla f(z_t) \right)$
		\STATE $\tau_{t+1} = \argmin_\tau f(w_{t+1}+\tau (w_{t+1}-w_t))$ \\ \qquad ~$= \frac{\langle y-\Phi w_{t+1}, \Phi (w_{t+1}-w_t)\rangle}  {2\|\Phi (w_{t+1}-w_t)\|^2_2}$
		\STATE $z_{t+1}=w_{t+1}+\tau_{t+1} (w_{t+1}-w_t)$
		\STATE $t=t+1$
		\UNTIL Stop criteria met
		\STATE {\bfseries return} $w_t$
	\end{algorithmic}%\vspace{-1em}
	\end{algorithm} 

\noindent
\textbf{Memory in vanilla IHT:} Based upon the same ideas as step size selection, we propose to include adaptive momentum acceleration; we select the momentum term as the minimizer of the objective: $\tau_{t+1} = \argmin_\tau f(w_{t+1}+\tau (w_{t+1}-w_t)) = \frac{\langle y-\Phi w_{t+1}, \Phi (w_{t+1}-w_t)\rangle}{2\|\Phi (w_{t+1}-w_t)\|^2_2}$, which also comes out as a closed-form solution. The step $z_{t+1}=w_{t+1}+\tau_{t+1} (w_{t+1}-w_t)$ at the end of the algorithm captures memory in the algorithm based on the results on acceleration by \cite{nesterov1983method} for convex optimization. 

\noindent 
\textbf{Automated Accelerated IHT for coreset selection:} 
Combining the ideas above leads to Automated Accelerated IHT, as presented in Algorithm~\ref{alg:a-iht}. 
The algorithm alternates between the projection step (steps 6 and 7) after the gradient updates, and the momentum acceleration step (step 8). 
It thus maintains two sets of iterates that alternatively update each other in each iteration at only a constant factor increase in per iteration complexity. 
The iterate $w_t$ at iteration $t$ is the most recent estimate of the optimizer, while the iterate $z_t$ models the effect of momentum or ``memory" in the iterates. 
We have shown exact line search that solves one dimensional problems to automate the step size selection ($\mu$) and the momentum parameter ($\tau$) for acceleration. 
In practice, these parameters can also be selected using a backtracking line search. 
%We defer the theoretical analysis to Section~\ref{sec:theory}.

\noindent
\textbf{Using de-bias steps in Automated Accelerated IHT:}
%\subsection{Accelerated IHT- II}
Based on pursuit methods for sparse optimization~\citep{needell2009cosamp, dai2009subspace, kyrillidis2014matrix}, 
we propose a modification that improves upon Algorithm~\ref{alg:a-iht} both in speed and accuracy in empirical evaluation. The modified algorithm is presented in Algorithm~\ref{alg:a-iht-2} in section~\ref{sec:iht2} in appendix due to space limitations.
%We hope this motivates a study of a unified analysis of pursuit methods that could potentially provide a roadmap for constructing such recipes in a principled way.
The key differences of Algorithm~\ref{alg:a-iht-2} from Algorithm~\ref{alg:a-iht} are that, with additional de-bias steps, one performs another gradient step and a line search in the sparsified space in each iteration for further error reduction.
We omit these steps in the algorithmic description to maintain clarity, since these steps do not provide much intellectual merit to the existing algorithm, but help boost the practical performance of Automated Accelerated IHT.

\textbf{Time complexity analysis.}
Here, we analyze the time complexity of IHT in terms of the dataset size $n$ and coreset size $k$, and show that IHT is faster than previous methods for Bayesian coreset construction. We take Algorithm~\ref{alg:a-iht} as an example and let the stopping criteria be a constant constraint on number of iterations; the time complexity for all the three versions of IHT (\emph{i.e.}, Algorithm~\ref{alg:iht},~\ref{alg:a-iht},~\ref{alg:a-iht-2}) are the same. 
As the dimension of $z_t, w_t$ is $n$, and the matrix multiplication $\Phi w$ has complexity $O(n)$, we can see that each line in Algorithm~\ref{alg:a-iht} except for the projection steps (line~4 and line~7) have complexity $O(n)$. The projection steps, as we have discussed in subsection~\ref{subsec:method-iht}, can be done in $O(n\log k )$. Therefore, the total time complexity of IHT is $O(n\log k)$. In comparison, previous state-of-the-art algorithms GIGA~\citep{campbell2018bayesian} and SparseVI~\citep{campbell2019sparse} have time complexity $O(nk)$, which is exponentially slower than IHT in terms of coreset size $k$. We note that some other factors play a role in the time complexity, \emph{e.g.}, the number of samples from posterior for IHT, GIGA and SparseVI; the number of iterations of the stochastic gradient descent in SparseVI. However, unlike $n$ and $k$ defined by the problem, those factors are chosen parameters specific to each algorithm. Therefore, we treat them as pre-specified constants, and focus on the complexity \emph{w.r.t.} dataset size $n$ and coreset size $k$.

\vspace{-0.3em}
\subsection{Theoretical Analysis of Convergence}\vspace{-0.3em}

In this subsection, we study the convergence properties of our main algorithm Automated Accelerated IHT in Algorithm~\ref{alg:a-iht}. We make a standard assumption about the objective -- the Restricted Isometry Property or RIP (Assumption~\ref{assum:rip}), which is a standard assumption made for analysis of IHT and its variants. 

\begin{assumption}[Restricted Isometry Property (RIP)]\label{assum:rip}
    The matrix $\Phi$ in the objective function satisfies the RIP property, \emph{i.e.}, for $\forall w\in \mathcal{C}_k$
    \begin{align}
        \alpha_k \|w\|_2^2\leq \|\Phi w\|_2^2 \leq \beta_k \|w\|^2_2.
    \end{align}
\end{assumption}
In RIP, $\alpha_k$ reflects the convexity and $\beta_k$ reflects the smoothness of the objective in some sense \citep{khanna2018iht, kyrillidis2014matrix}. 
We note that the assumption may not be necessary but is sufficient to show convergence theoretically.
For example, if the number of samples required to exactly construct $\hat g$ is less than the coreset size ($a_k=0$ in RIP), so that the system becomes under-determined, then a local minimum can also be global achieving zero error without assuming that the RIP holds.
On the other hand, when the number of samples goes to infinity, RIP is saying that the restricted eigenvalues of covariance matrix, $cov[\gL_i(\theta), \gL_j(\theta)]$ where $\theta\sim\hat \pi$, are upper bounded and lower bounded away from $0$. It is an active area of research in random matrix theory to quantify RIP constants e.g. see~\citep{Baraniuk2008ASP}.

RIP generalizes to restricted strong convexity and smoothness~\citep{chen2010general}; thus our results could potentially be extended to general convex $f(\cdot)$ functions. 
We present our main result next, and defer the details of the theory to  section~\ref{sec:theory} in the appendix.
\begin{theorem}\label{the:1}
In the worst case scenario, with Assumption~\ref{assum:rip}, the solutions path found by Automated Accelerated IHT satisfies the following iterative invariant. 

\begin{align}
        \|w_{t+1}-w^*\|_2 &\leq \rho|1+\tau_t|\cdot \|w_t - w^*\|_2\\
         &+ \rho |\tau_t| \cdot {\|w_{t-1}-w^*\|_2}  + 2\beta_{3k} \sqrt{\beta_{2k}} \| \epsilon\|_2,
\end{align}
where $\rho = \left( 2 \max\{ \frac{\beta_{2k}}{\alpha_{3k}} - 1, 1-\frac{\alpha_{2k}}{\beta_{3k}}\} + \frac{\beta_{4k}-\alpha_{4k}}{\alpha_{3k}} \right)$, and $\|\epsilon\|_2 = \|y-\Phi w^*\|_2$ is the optimal error.
\end{theorem}
The theorem provides an upper bound invariant among consecutive iterates of the algorithm. To have a better sense of convergence rate, we can derive linear convergence from our iterative invariant, as shown in Corollary~\ref{corollary:1}. 
\begin{corollary}\label{corollary:1}
	Given the iterative invariant as stated in Theorem~\ref{the:1}, and assuming the optimal solution achieves $\|\epsilon\|_2=0$, the solution found by Algorithm~\ref{alg:a-iht} satisfies:
	 \begin{align}
	     f(w_{t+1})-f(w^\star)\leq \phi^t\left( \frac{\beta_{2k}}{\alpha_{2k}} f(w_1) + \frac{\rho\tau\beta_{2k}}{\phi\alpha_{k}} f(w_0) \right),
	 \end{align}
	 where $\phi = (\rho(1+\tau) +\sqrt{\rho^2(1+\tau)^2+4\rho \tau})/2$ and $\tau = \max_{i\in [t]} |\tau_i|$. It is sufficient to show linear convergence to the global optimum, when $\phi<1$, or equivalently $\rho<1/(1+2\tau)$.
\end{corollary}

We note that Theorem~1 holds more generally, and we chose the simplifying condition of $\|\epsilon\|_2=0$ for Corollary~\ref{corollary:1} to clearly highlight the main result of linear convergence. If $\|\epsilon\|_2>0$, 
%we can distribute the error term in Theorem 1 into each iterate, and then
the linear convergence (up to an error) can be proved in the same way but with more complicated expressions. 

Thus, Algorithm~\ref{alg:a-iht} generates a sequence of iterates that decrease the quadratic objective in equation~(\ref{prob:l2}) at a geometric rate. The quadratic objective can upper bound the symmetric KL divergence, \emph{i.e.}, the sum of forward KL and reverse KL divergences, between the constructed coreset posterior and the true posterior under certain conditions, as shown in Proposition 2 by \cite{campbell2019sparse}, which further justifies our approach of using this objective.

Our theory and algorithm differ from the work by~\cite{khanna2018iht} in several ways. The non-negative constraint is unique to the Bayesian coreset problem, and extending the analysis from the original IHT to our setting is non-trivial (see Section~\ref{sec:theory} in appendix). 
Further, the new analysis we present does not work with the restricted gradient used by~\citet{khanna2018iht}, which is why we choose to use the full gradient instead (line 7 in Algorithm~\ref{alg:a-iht}). We note that the restricted gradient refers to the $\nabla f(z_t)|_\mathcal{S}$ in Algorithm~\ref{alg:a-iht}. We also observe empirically in our experiments that using the full gradient performs better for the coreset problem. The high-level idea is that, during the iterations, it is not guaranteed that $\mathcal{S}$ (line 4 in Algorithm~\ref{alg:a-iht}) contains the optimal support, while the full gradient is guaranteed to provide information on the optimal support.  Further, we also automated the step-size selection, the momentum selection, and the de-bias step selection to minimize the need of tuning. Recall that vanilla IHT (Algorithm~\ref{alg:iht}) is much slower than the greedy approach by~\cite{campbell2018bayesian}, and so the enhancements we propose are crucial to ensure that the overall algorithm is both faster as well as better performing than the state-of-the-art.

\section{Related Work}
\begin{figure*}
\centering
    \begin{minipage}[c]{0.29\linewidth}\centering
        \includegraphics[width=1\linewidth]{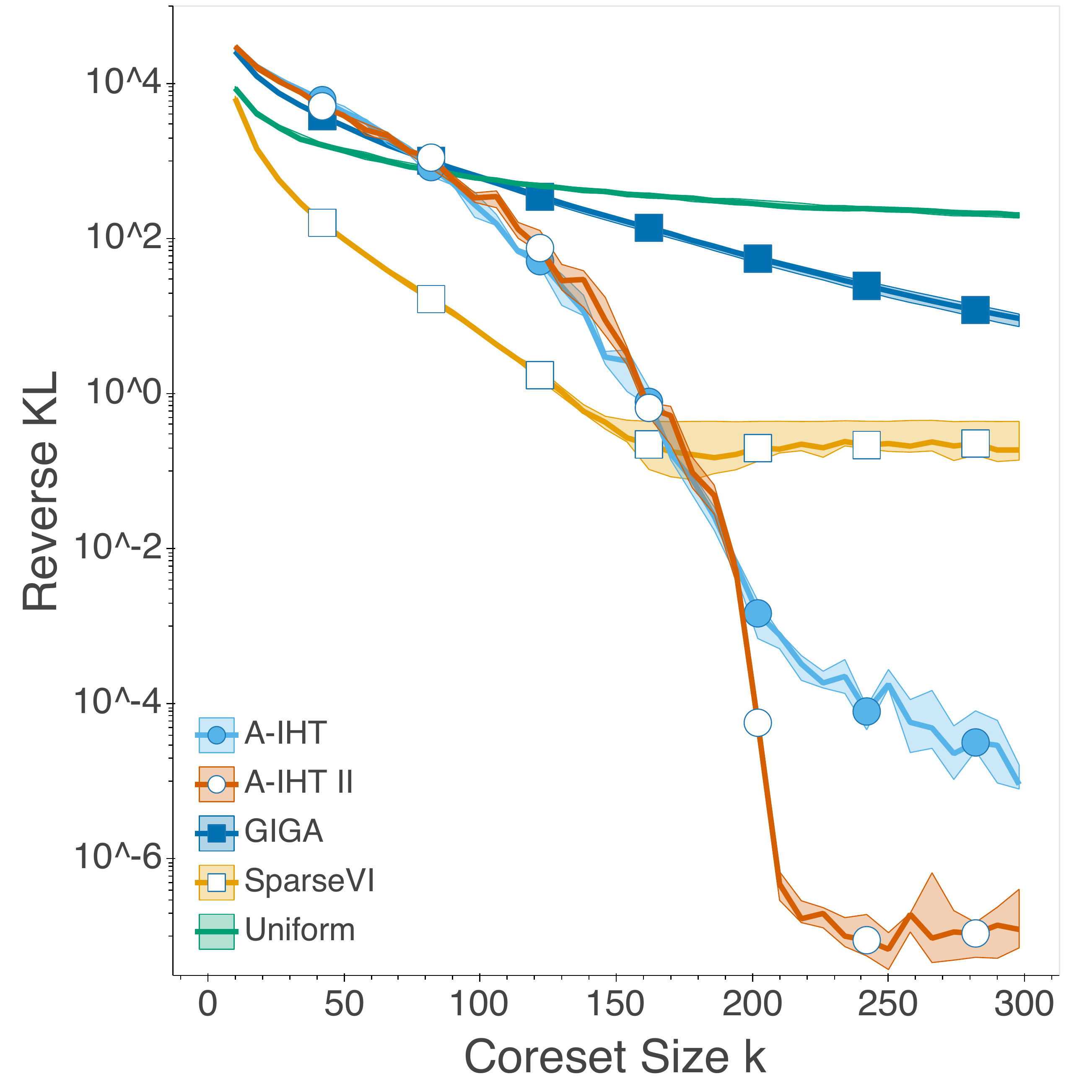}
        
        \small(a) %synthetic Gaussian posterior inference
        \vspace{-0.3em}
    \end{minipage}
	\begin{minipage}[c]{0.58\linewidth}\centering
		%\vspace{-2em}
		\includegraphics[width=0.49\linewidth]{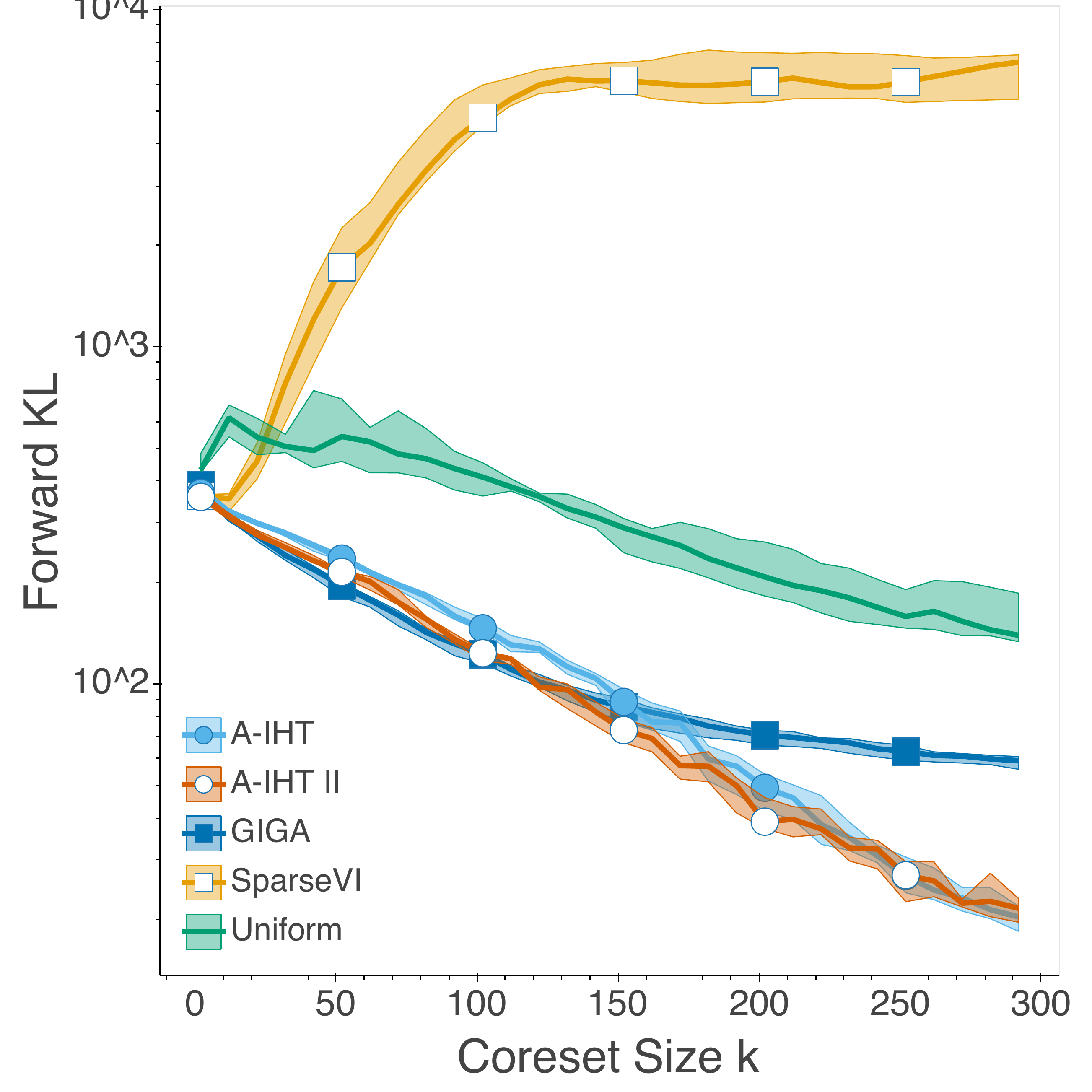}
		\includegraphics[width=0.49\linewidth]{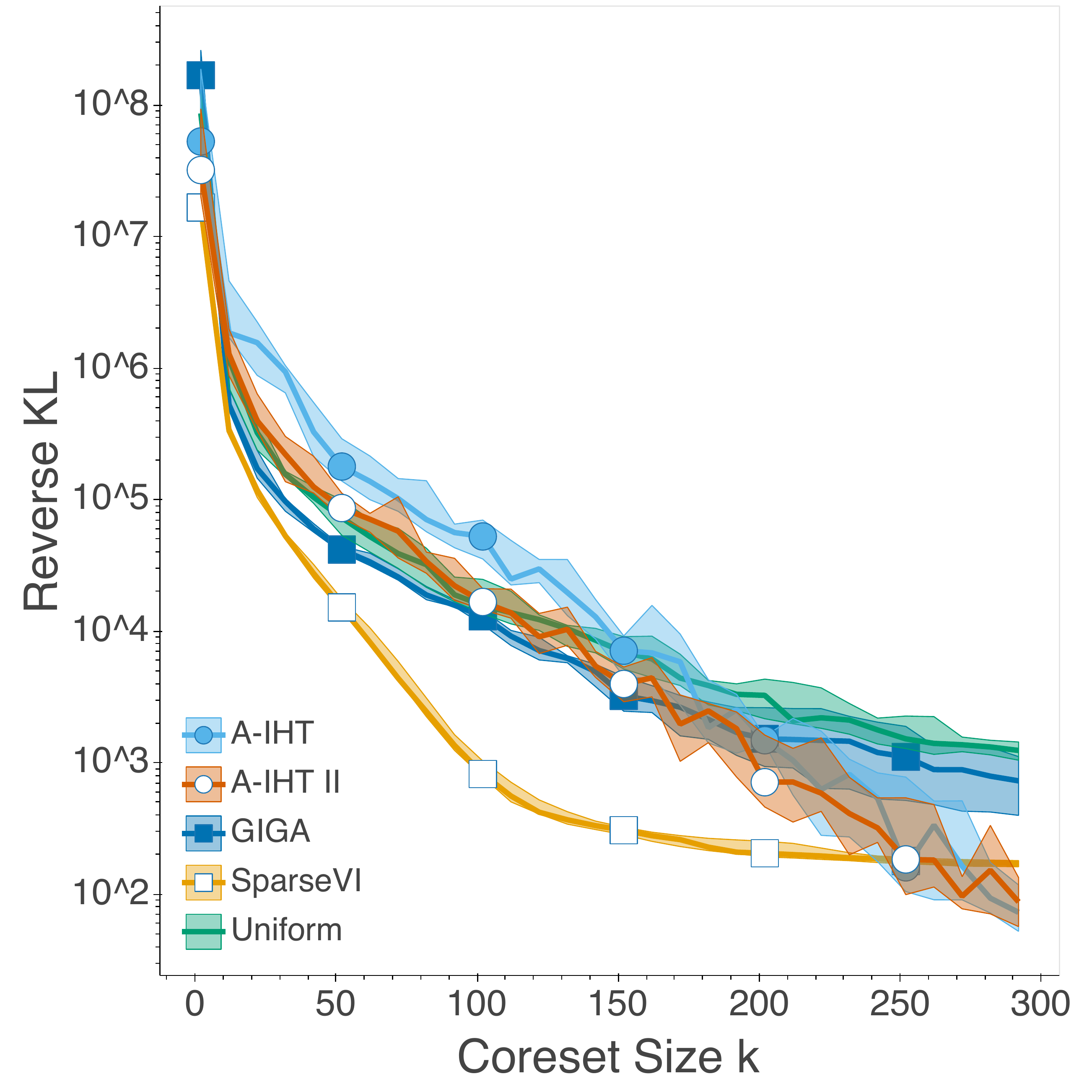}
		\small(b) %Bayesian radial basis function regression
		\vspace{-0.3em}
	\end{minipage}
	\caption{
	(a): Bayesian coresets for synthetic Gaussian posterior inference. (b): Experiments on Bayesian radial basis function regression, with the difference between true posterior and coreset posterior measured in both forward KL and reverse KL. 
	For both (a) and (b), $k$ is the sparsity setting, and the solid lines are the median KL divergence between the constructed coreset posterior and true posterior over 10 trials. The shaded area is the KL divergence between $25^{th}$ and $75^{th}$ percentiles.}\label{fig:exp-3-kl}\label{fig:exp-1}
	%\vspace{-0.6em}
\end{figure*}

%Approximate inference is a very well-studied problem. 
Other scalable approaches for Bayesian inference include subsampling and streaming methods for variational Bayes \citep{hoffman2013stochastic, broderick2013streaming}, subsampling methods for MCMC \citep{welling2011bayesian, ahn2012bayesian, korattikara2014austerity, maclaurin2015firefly}, and consensus methods for MCMC \citep{srivastava2015wasp, rabinovich2015variational, scott2016bayes}. 
These algorithms are motivated by empirical performance and come with few or no theoretical optimization-based guarantees on the inference quality, and often do not scale to larger datasets.
Bayesian coresets could be used as part of these approaches, thus resulting into a universal tool for approximate MCMC and variational inference. Recently, Bayesian coresets have been applied to complex models and data. For example, \cite{pinsler2019bayesian} apply Bayesian coresets to batch active learning on Bayesian neural networks with real-world image datasets. 

There have been few studies that study convergence properties of approximate inference algorithms. %albeit such studies are gaining in importance. 
%for obvious reasons.
\cite{campbell2019sparse}  presented a linear convergence rate, but the assumptions they make are non-standard as the rate of convergence depends on the how well individual samples correlate with the overall loss.
Approximation guarantees in terms of KL-divergence are provided~\citep{Koyejo2014approinference, khanna2017information} for structured sparse posterior inference using the greedy forward selection procedure.
\cite{Locatello2017BoostingVI, Locatello2018BBboosting} study convergence rates for a boosting based algorithm for iteratively refined variational inference.

Thresholding based optimization algorithms have been attractive alternatives to relaxing the constraint to a convex one or to greedy selection. % for sparsity constrained optimization.
\cite{Bahmani2013} provide a gradient thresholding algorithm that generalizes pursuit approaches for compressed sensing to more general losses.
\cite{Yuan2018HTP} study convergence of gradient thresholding algorithms for general losses.
\cite{Jain2014iht} consider several variants of thresholding based algorithms for high dimensional sparse estimation.
Additional related works are discussed in Section~\ref{sec:supp-related} in the appendix.

% !TEX root = main.tex

\section{Experiments}

We empirically examine the performance of our algorithms to construct coresets for Bayesian posterior approximation.  Three sets of experiments are presented: Gaussian posterior inference, Bayesian radial basis function regression, and Bayesian logistic and Poisson regression using real-world datasets. 

Besides the Automated Accelerated IHT (Algorithm~\ref{alg:a-iht}), we propose  Automated Accelerated IHT - \RomanNumeralCaps{2} (Algorithm~\ref{alg:a-iht-2} in section~\ref{sec:iht2} of appendix), that adds a de-bias step that further improves Algorithm~\ref{alg:a-iht} in practice. 
%The modified algorithm, named Automated Accelerated IHT - \RomanNumeralCaps{2}, exhibits even better actual performance due to faster convergence as a result of the de-bias step. 
We refer to the appendix for detailed explanation and discussion of Algorithm~\ref{alg:a-iht-2} due to space limitation.

The proposed algorithms, Automated Accelerated IHT (A-IHT) and Automated Accelerated IHT \RomanNumeralCaps{2} (A-IHT \RomanNumeralCaps{2}), are compared with three baseline algorithms, \emph{i.e.}, Random (Uniform), Greedy Iterative Geodesic Ascent (GIGA)~\citep{campbell2018bayesian} and Sparse Variational Inference (SparseVI)~\citep{campbell2019sparse}.  We use the public Github resources of GIGA and SparseVI for their implementation, where details are provided in our Github repository (link on page 2).  We note that the Frank-Wolfe (FW) method proposed in \citep{campbell2019automated} has been shown to be inferior to GIGA and SparseVI in the two corresponding articles, and thus we believe that comparing with GIGA and SparseVI is sufficient.

We calculate the Kullback–Leibler (KL) divergence between the constructed coresets posterior $\pi_w$ and the true posterior $\pi$. We measure both the forward KL divergence $D_{\text{KL}}(\pi \| \pi_w)$ and reverse KL divergence $D_{\text{KL}}(\pi_w \| \pi)$.
Both A-IHT and A-IHT \RomanNumeralCaps{2} require minimal tuning, \emph{i.e.}, only the stoping criterion is required: $\|w_t-w_{t-1}\| \leq 10^{-5} \|w_t\|$, or number of iterations $> 300$ for both A-IHT and A-IHT \RomanNumeralCaps{2} .

\subsection{Synthetic Gaussian posterior inference}\label{sec:exp-1}
We examine the algorithms in this experiment where we have closed-form exact expressions. Specifically, we compare each of these algorithms in terms of optimization accuracy without errors from sampling.  

\begin{figure*}[t!]\centering
	%\vspace{-2em}
	\begin{minipage}[c]{0.87\linewidth}\centering
		%\vspace{-2em}
		\includegraphics[width=0.32\linewidth]{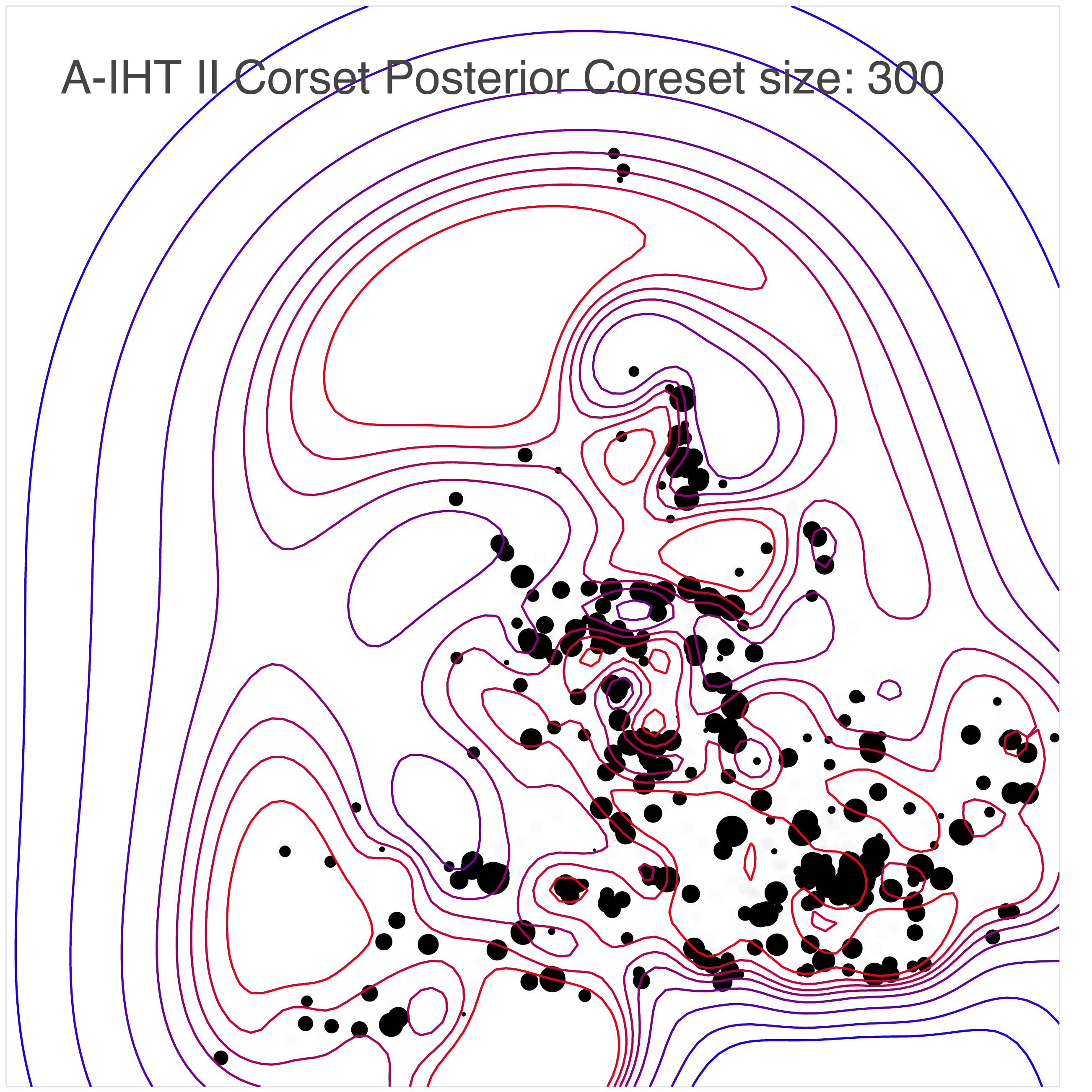}
		\includegraphics[width=0.32\linewidth]{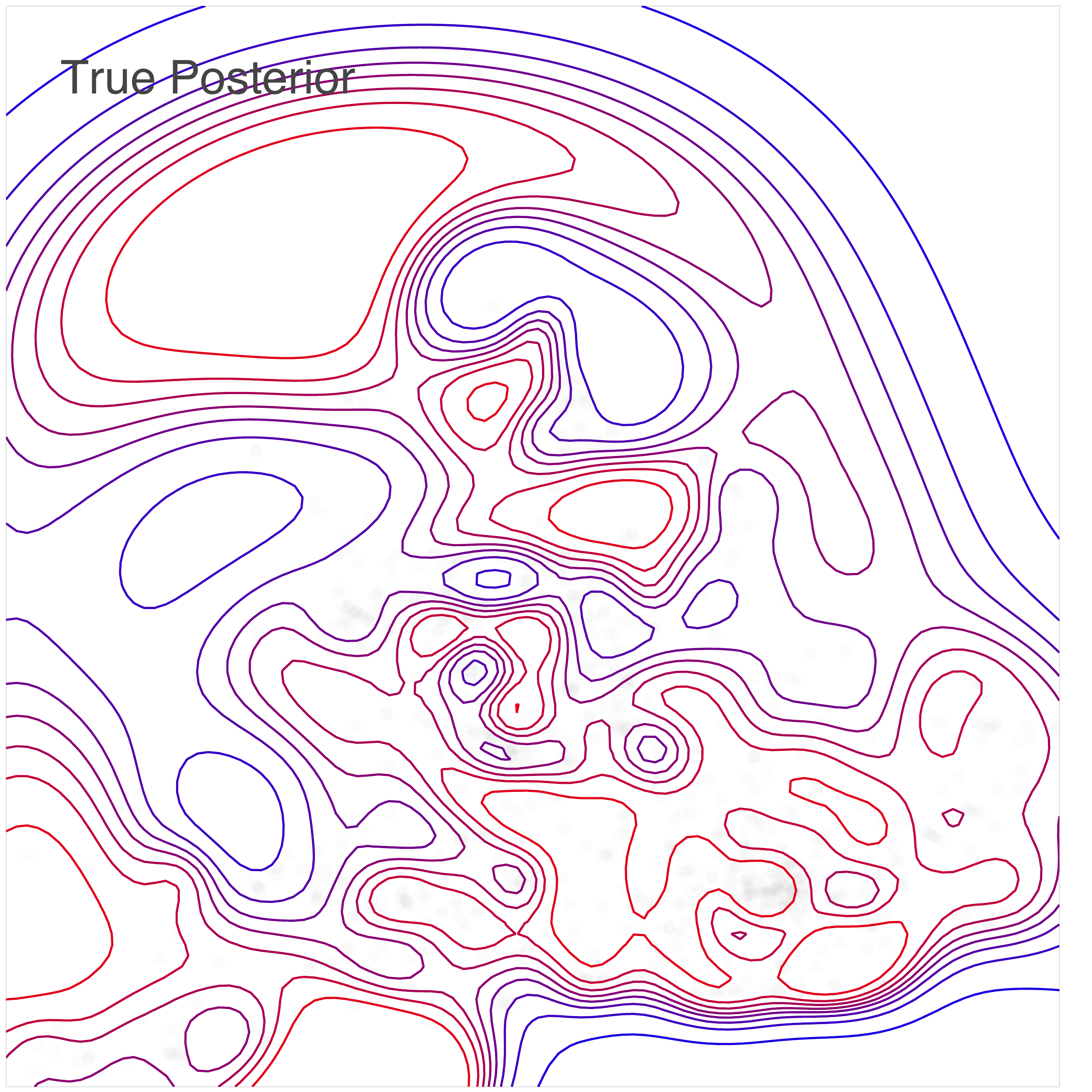}
		\includegraphics[width=0.32\linewidth]{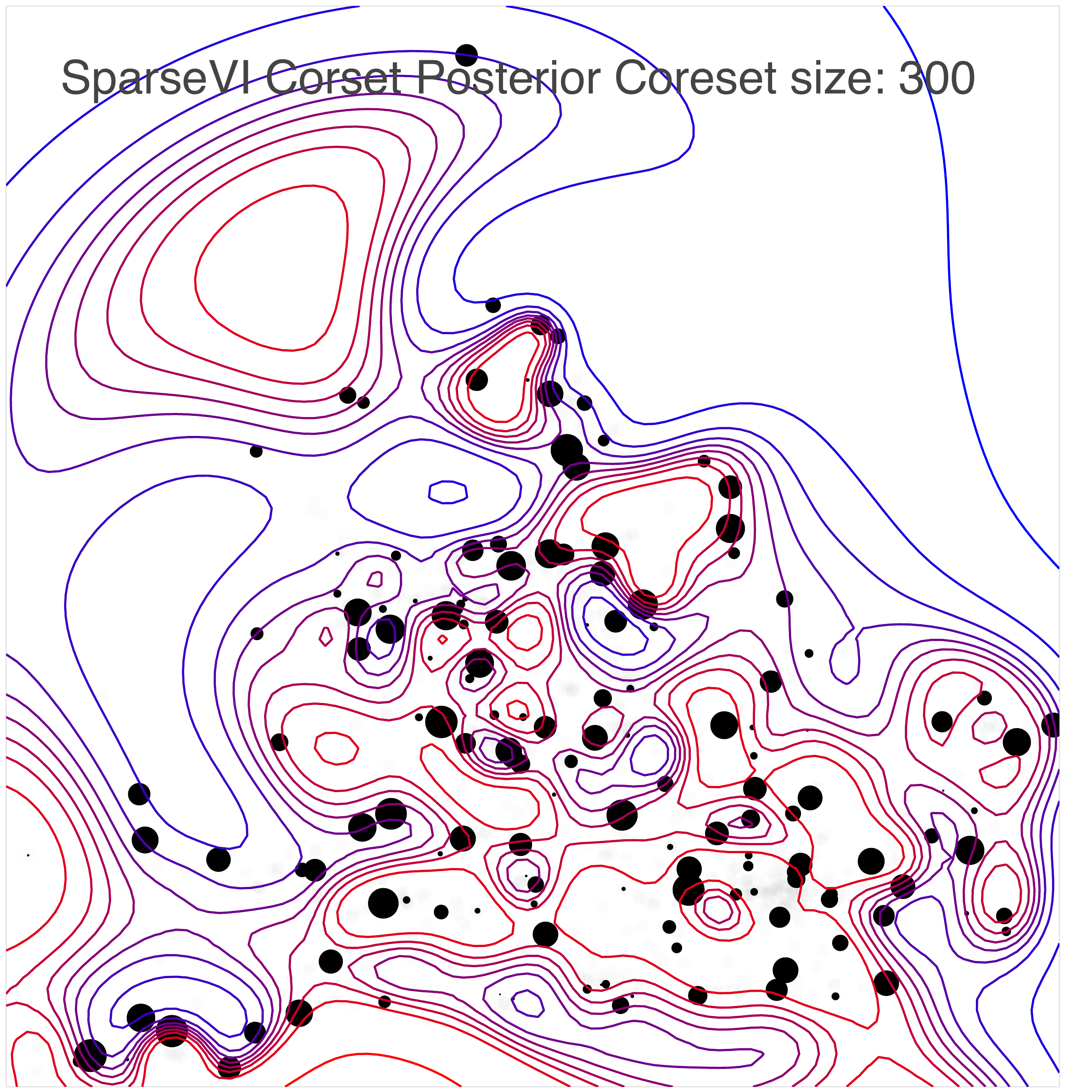}
		%\vspace{-4em}
	\end{minipage}
	\begin{minipage}[c]{0.87\linewidth}\centering
		%\vspace{-2em}
		\includegraphics[width=0.24\linewidth]{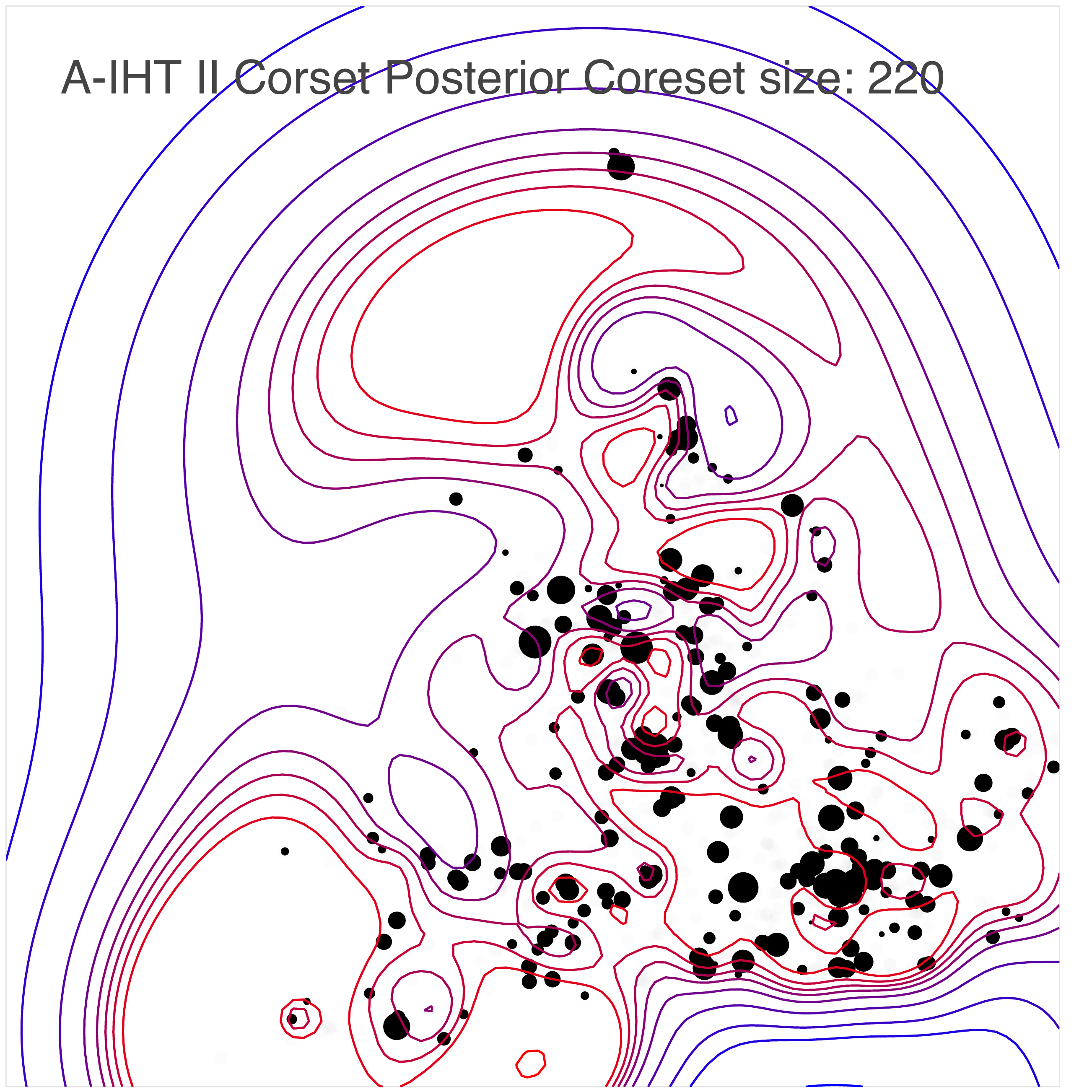}
		\includegraphics[width=0.24\linewidth]{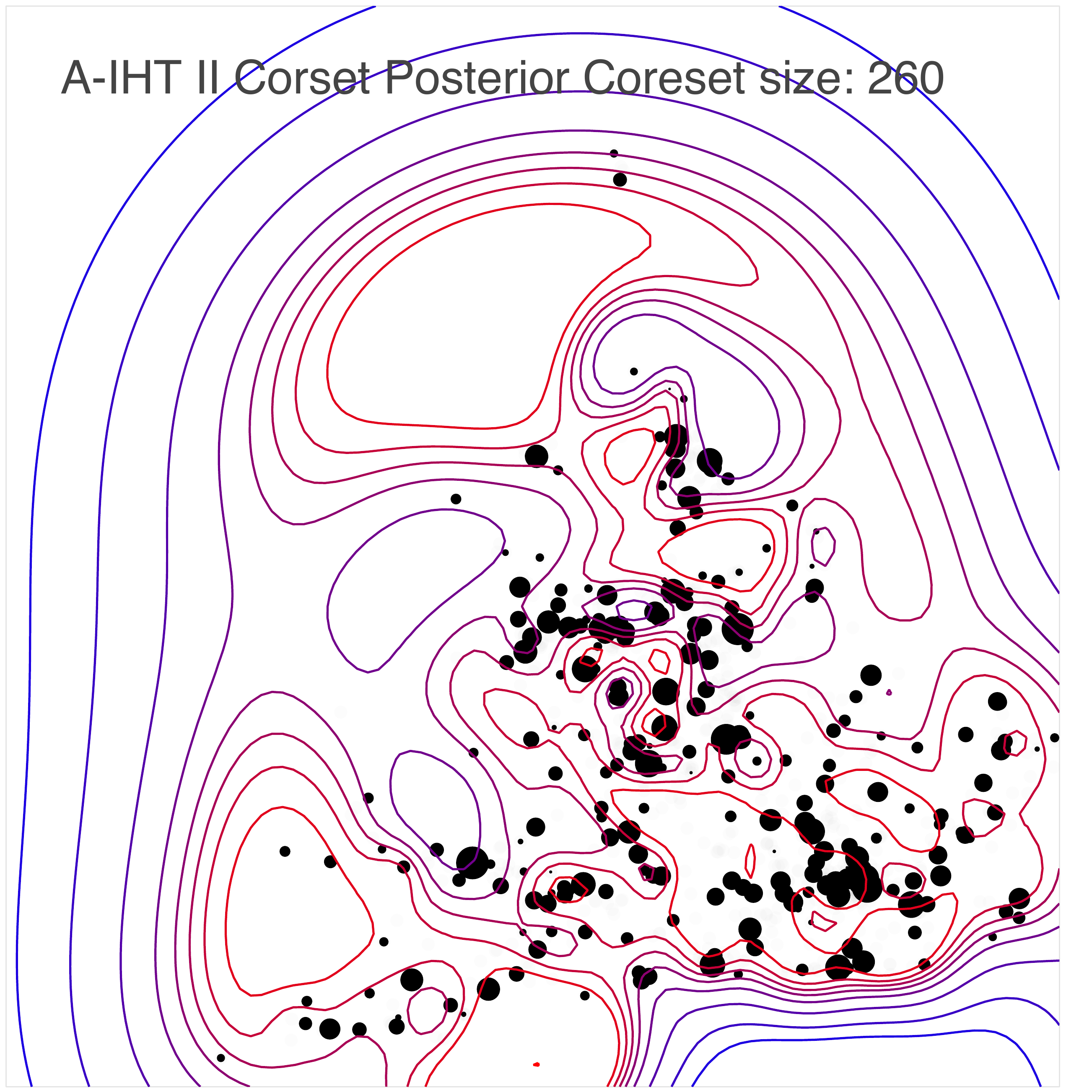}
		\includegraphics[width=0.24\linewidth]{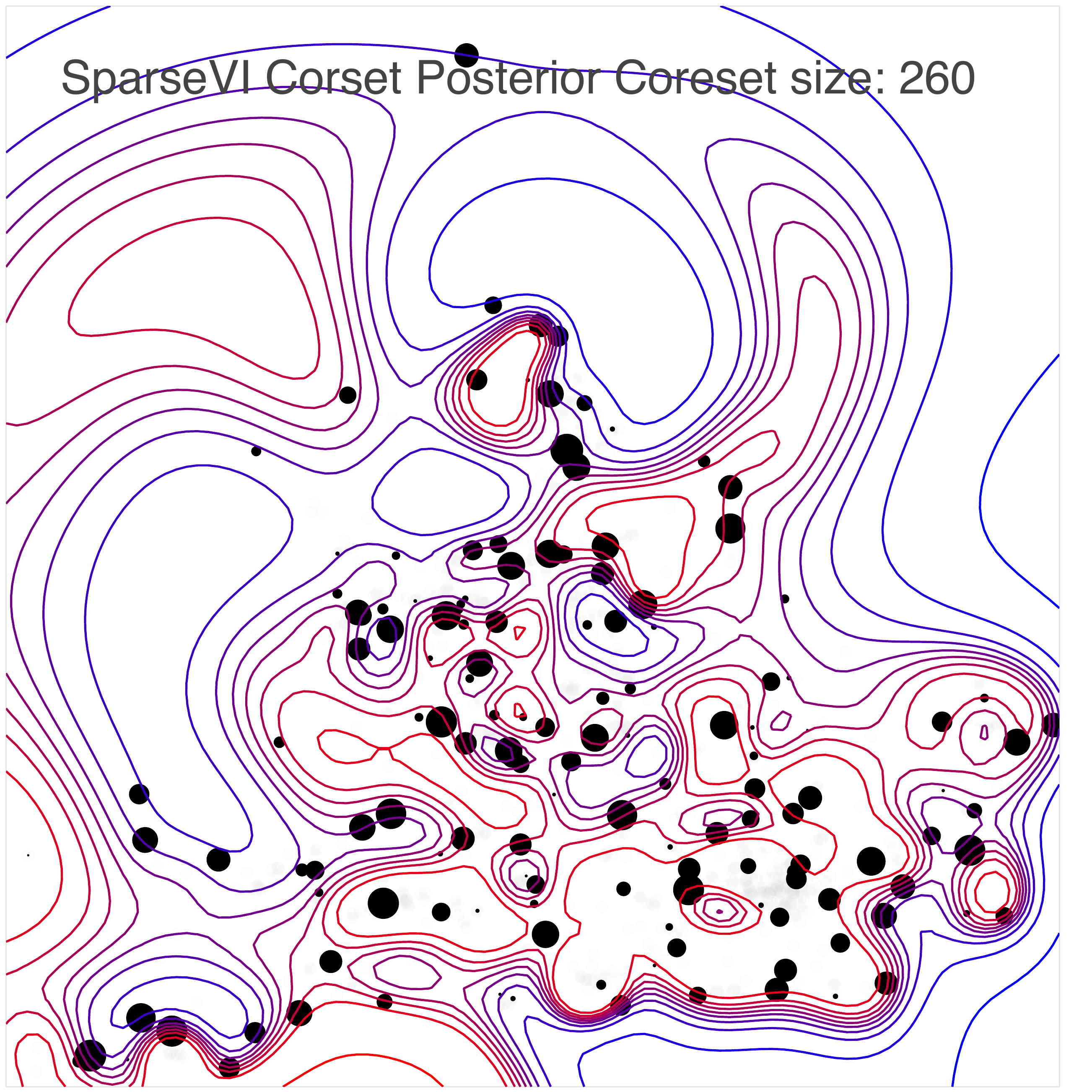}
		\includegraphics[width=0.24\linewidth]{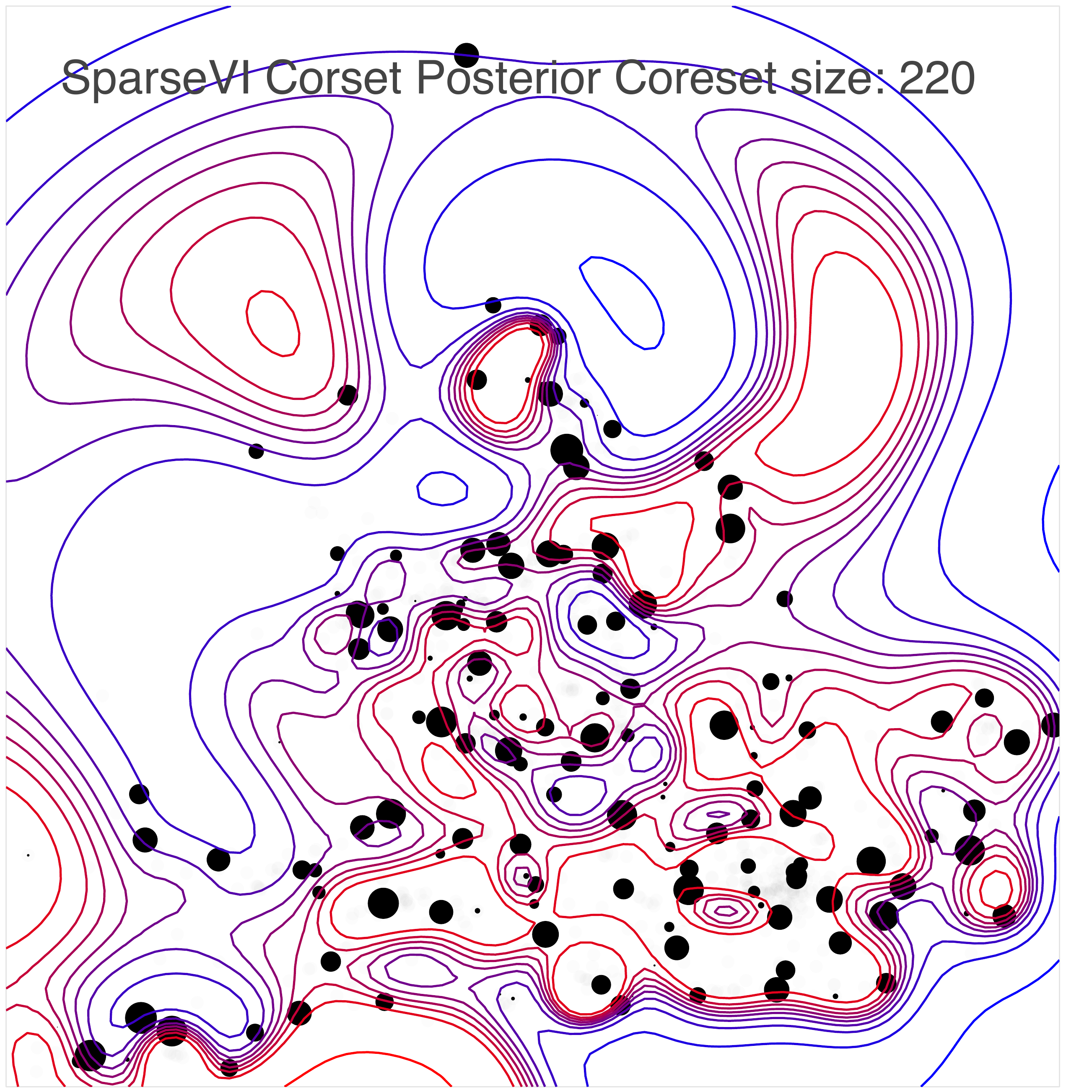}
		%\vspace{-4em}
	\end{minipage}
	\caption{Experiments on Bayesian radial basis function regression, where coreset sparsity setting $k=220, 260, 300$. Coreset points are presented as black dots, with their radius indicating assigned weights.  When $k=300$, posterior constructed by Accelerated IHT~\RomanNumeralCaps{2} (top left) shows almost exact contours as the true posterior (top middle), while posterior constructed by SparseVI (top right) shows deviated contours from the true posterior distribution.} \label{fig:exp-3-coresets}
	%\vspace{-0.6em}
\end{figure*}

For the $D$-dimensional Gaussian distribution, we set the parameter $\theta\sim \mathcal{N}(\mu_0, \Sigma_0)$ and draw $N$ i.i.d. samples $x_n {\sim} \mathcal{N}(\theta, \Sigma)$, which results in a Gaussian posterior distribution with closed-form parameters, as shown in ~\citep{campbell2019sparse}. 
We set the dimension $D=200$, number of samples $N=600$, and maximal sparsity $k$ is set to be $1,\dots, 300$. The initial mean $\mu_0=0$, and the initial covariance matrix is set to be $\Sigma_0 = \Sigma = I$. The learning rate for SparseVI is $\gamma_t=1/t$, and the number of weight update iterations for Sparse VI is $100$, as suggested by their paper. 

Comparison among all the 5 algorithms measuring the reverse KL divergence between the true posterior and the coreset posterior is presented in Figure~\ref{fig:exp-1}~(a), which shows that IHT outperforms SparseVI and GIGA, achieving nearly optimal results. We observe that SparseVI stops improving once it hits certain sparsity level, which we suspect is due to the limitations of its greedy nature. It can also be observed that A-IHT \RomanNumeralCaps{2} converges faster than A-IHT.   Additional results are put in the section~\ref{sec:add-exp-1} in appendix.

\subsection{Bayesian Radial Basis Function Regression}\label{sec:exp-radial}

In this subsection, we explore the performance of proposed methods versus the baselines in terms of the both forward KL and reverse KL divergence. The SparseVI algorithm optimizes reverse KL; we show this does not always imply reduction in the forward KL. Indeed selecting more points to greedily optimizing the reverse KL can cause an increase in the forward KL.

\begin{figure*}[t]
\centering
	%\vspace{-2em}
    \begin{minipage}[c]{0.87\linewidth}\centering
		%\vspace{-2em}
		\includegraphics[width=0.32\linewidth]{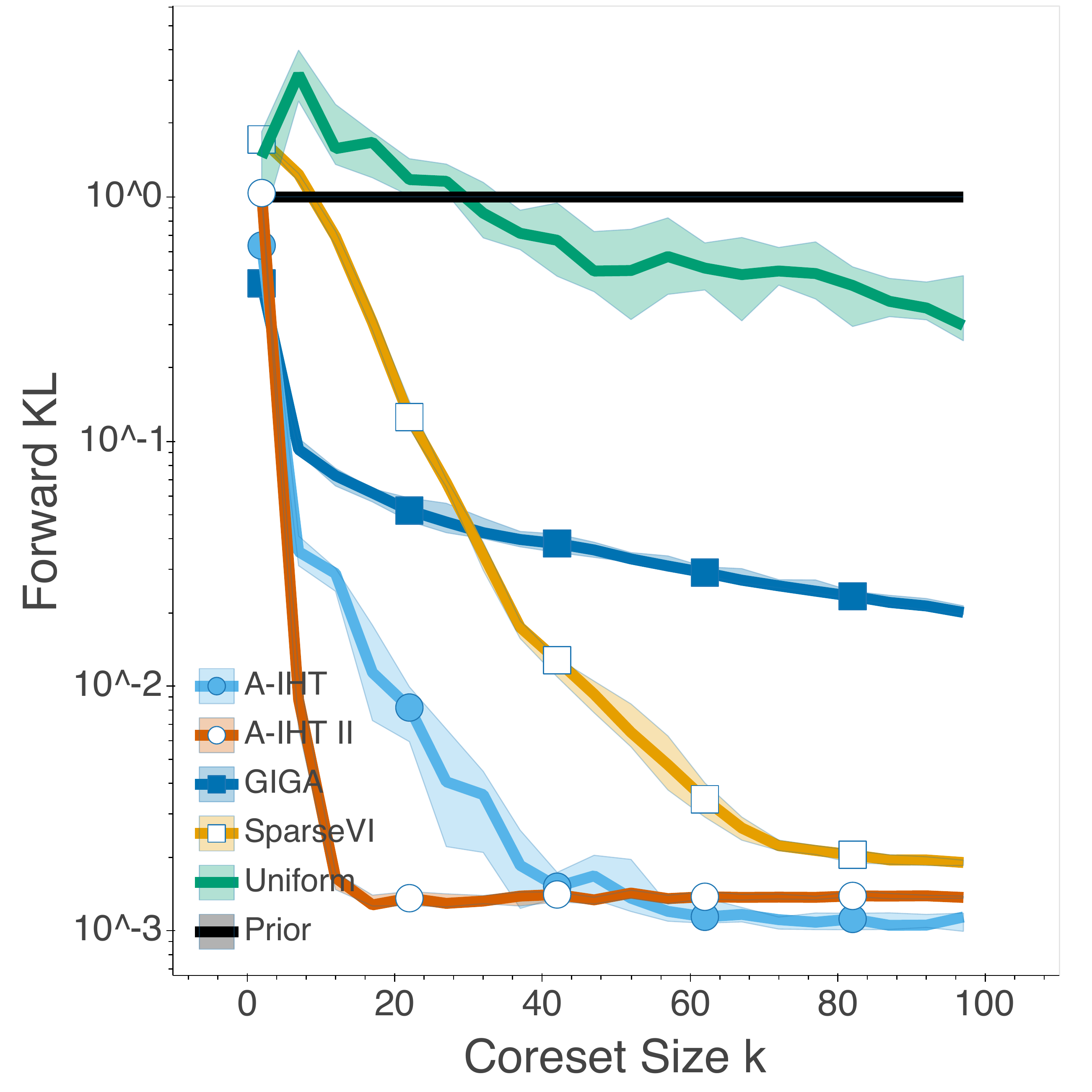}
		\includegraphics[width=0.32\linewidth]{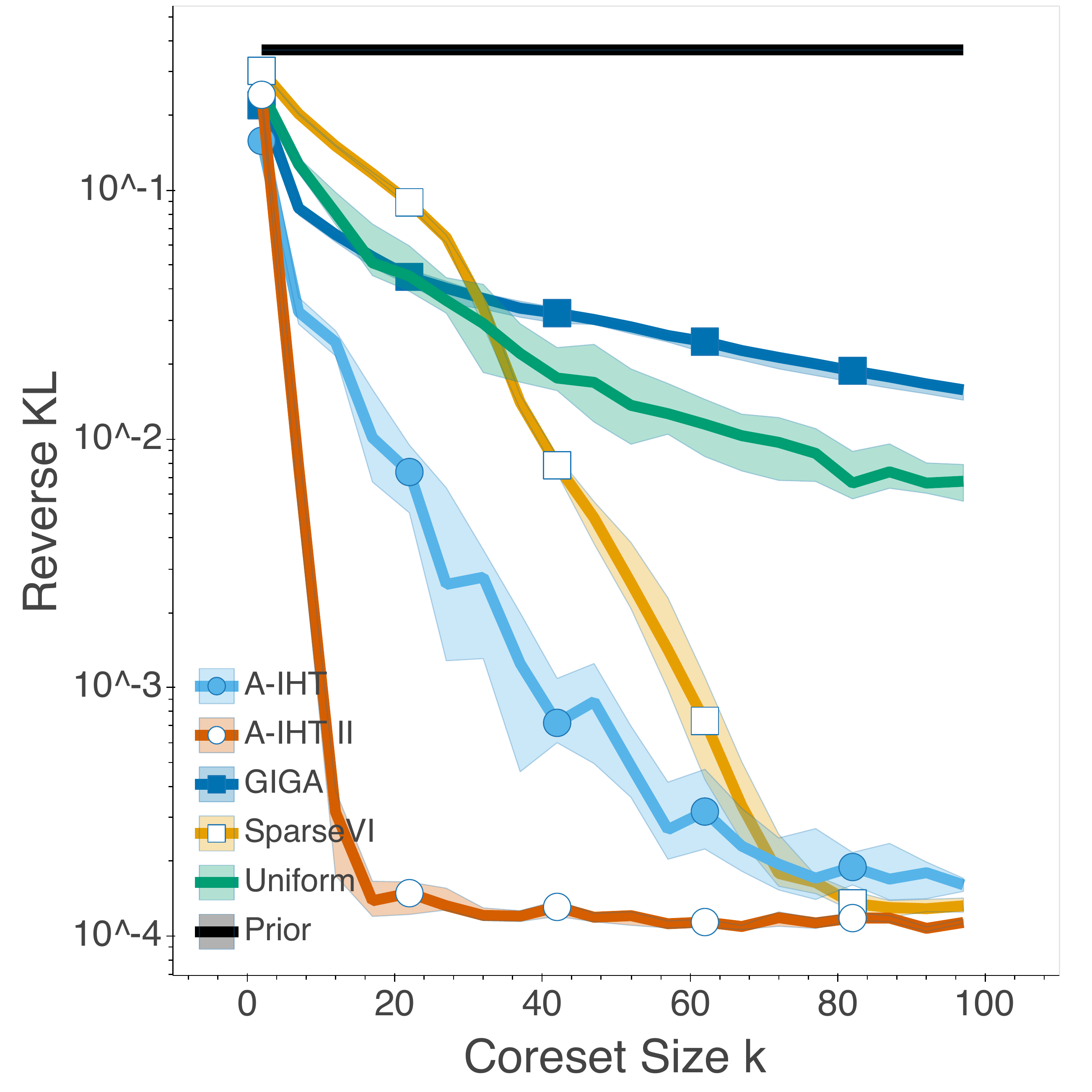}
		\includegraphics[width=0.32\linewidth]{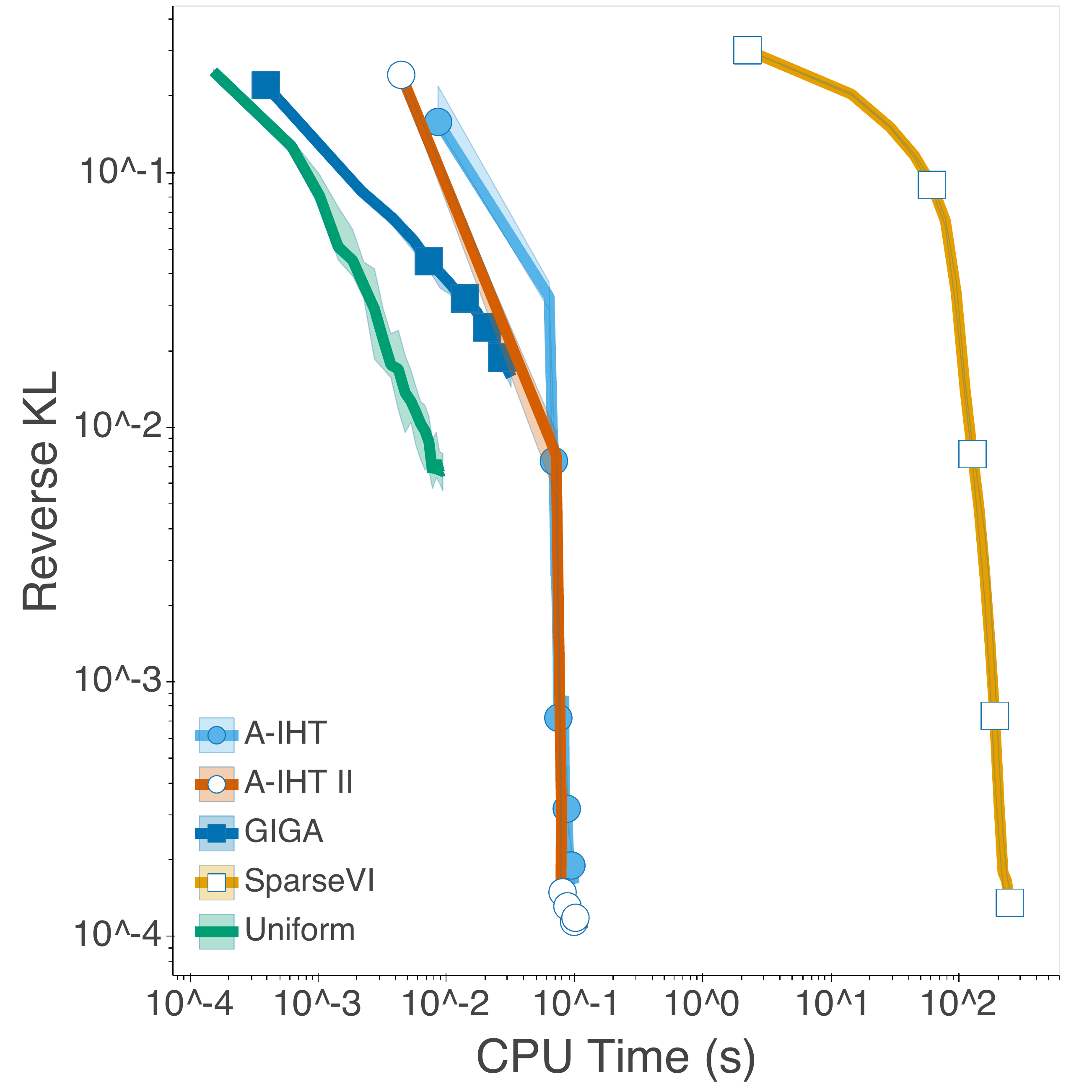}
		%\vspace{-4em}
		
		\small{Top Row: \texttt{synthetic dataset}}
	\end{minipage}
	\begin{minipage}[c]{0.87\linewidth}\centering
		\vspace{0.5em}
		\includegraphics[width=0.32\linewidth]{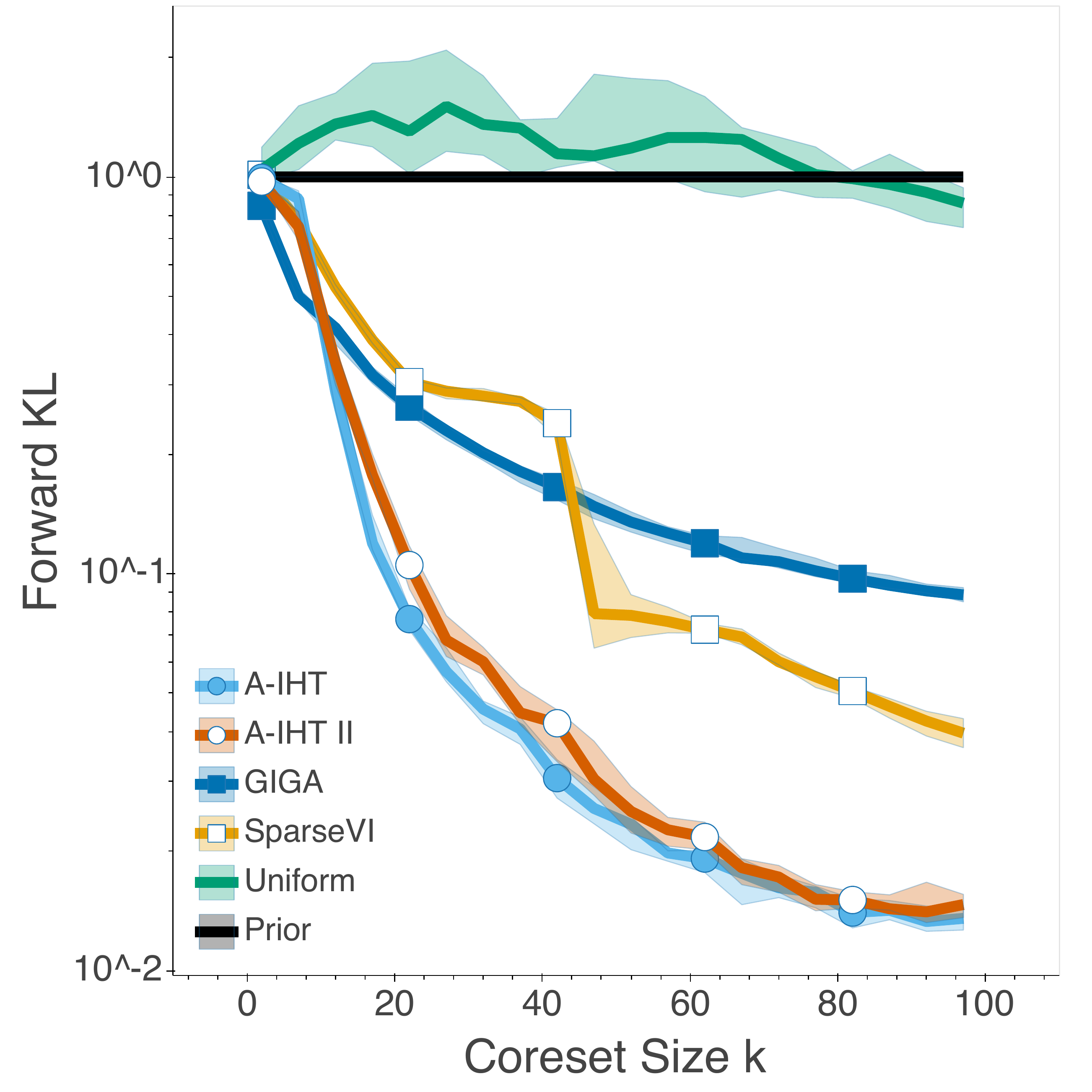}
		\includegraphics[width=0.32\linewidth]{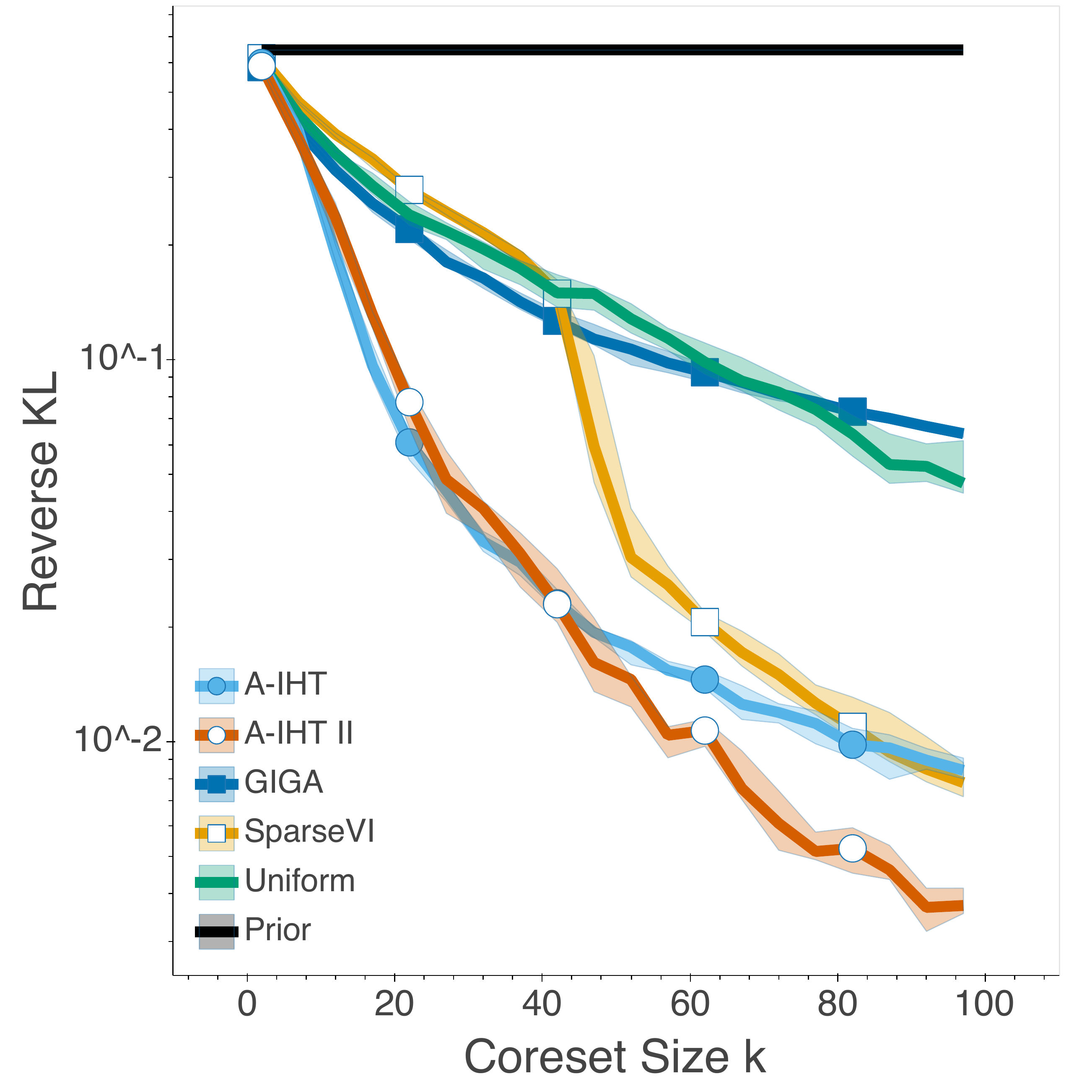}
		\includegraphics[width=0.32\linewidth]{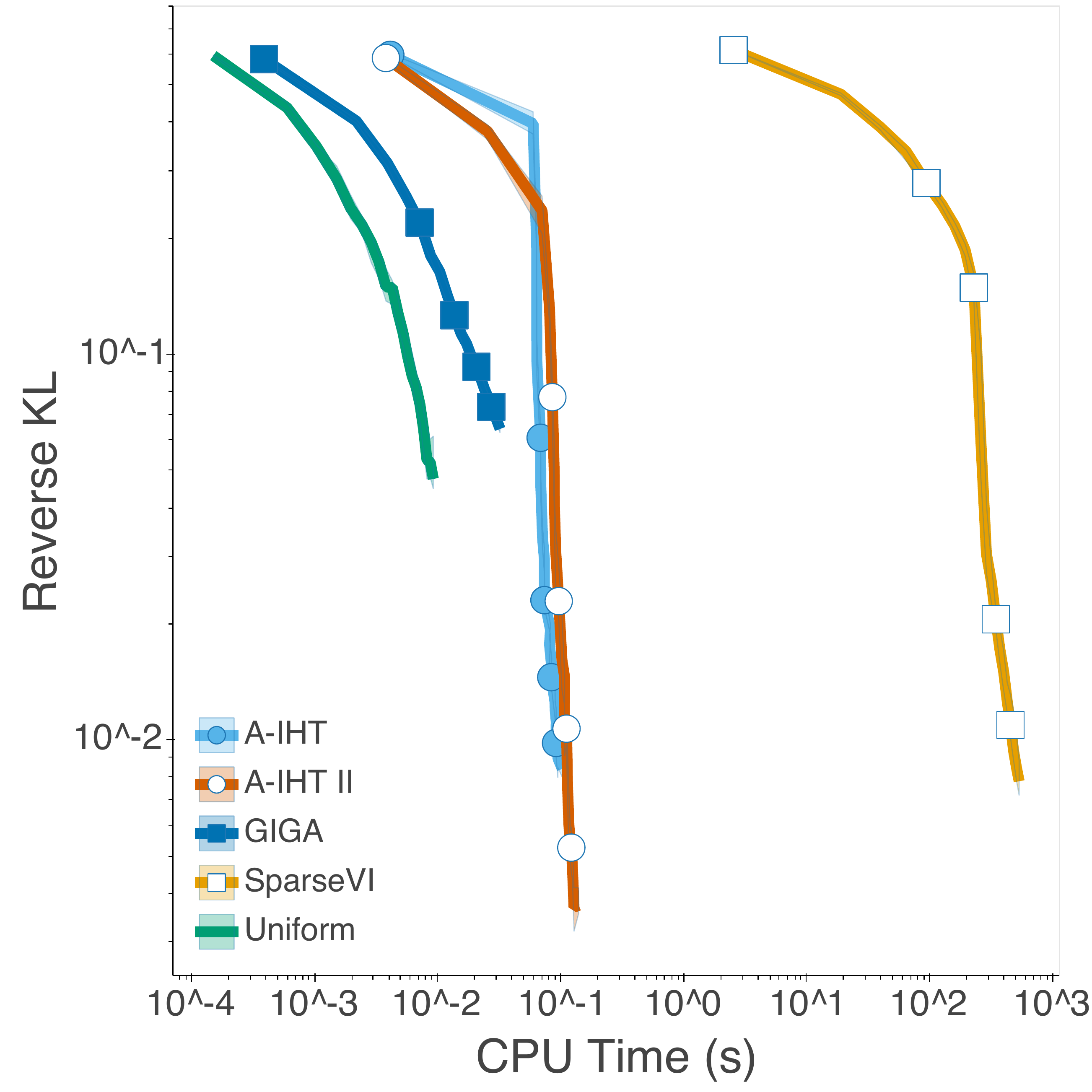}
		%\vspace{-4em}
		
		\small{Bottom Row: \texttt{phishing dataset}}
	\end{minipage}
	
	\caption{Bayesian coreset construction for logistic regression (LR) using the \texttt{synthetic dataset} (top row) and the \texttt{phishing dataset} (bottom row). All the algorithms are run $20$ times, and the median as well as the interval of $35^{th}$ and $65^{th}$ percentile, indicated as the shaded area, are reported. Different maximal coreset size $k$ is tested from $1$ to $100$. Forward KL (left column) and reverse KL (middle column) divergence between estimated true posterior and coreset posterior indicate the quality of the constructed coreset. The smaller the KL divergence, the better the coreset is. The running time for each algorithms is also recorded (right column).} \label{fig:exp-3-phishing}
	%\vspace{-0.6em}
\end{figure*}

We aim to infer the posterior for Bayesian radial basis function regression. Given the dataset\footnote{The task is to predict housing prices from the UK land registry data (\url{https://www.gov.uk/government/statistical-data-sets/price-paid-data-downloads}) using latitude/longitude coordinates
from the Geonames postal code data (\url{http://download.geonames.org/export/zip/}) as features.} $\{(x_n, y_n)\in \mathbb{R}^2 \times \mathbb{R}\}_{n=1}^N$, where $x_n$ is the latitude/longitude coordinates and $y_n$ is house-sale log-price in the United Kingdom, the goal is to infer coefficients $\alpha \in \mathbb{R}^D$ for $D$ radial basis functions $b_d(x)=\exp(-\frac{1}{2\sigma_d^2}(x-\mu_d)^2)$ for $d\in [D]$. The model is $y_n = b_n^\top \alpha + \epsilon_n$, where $\epsilon_n \sim \mathcal{N}(0,\sigma^2)$ with $\sigma^2$ be the variance of $\{y_n\}$,  and $b_n = [b_1(x_n), \dots, b_D(x_n)]^\top$. We set prior $\alpha \sim \mathcal{N}(\mu_0, \sigma_0^2 I)$, where $\mu_0, \sigma_0^2$ are empirical mean and second moment of the data. We subsampled the dataset uniformly at random to $N=1000$ records for the experiments, and generated 50 basis functions for each of the 6 scales $\sigma_d\in \{0.2, 0.4, 0.8, 1.2, 1.6, 2.0\}$ by generating means $\mu_d$ for each basis uniformly from data. Except for the $300$ basis functions, an additional near-constant basis of scale $100$, with mean corresponding to the mean latitude and longitude of the data, is added. Therefore, $D=301$ basis functions are considered. Each of the algorithms has access to the closed-form of posterior distribution and covariance (see \citep{campbell2019sparse} for detailed derivation).

Specific settings for the algorithms are as follows. 
For SparseVI, the exact covariance can be obtained, and the weight update step can be done without Monte Carlo estimation. 
For IHT and GIGA, we use true posterior for constructing the $\ell_2$ loss function. 
The learning rate for SparseVI is set to be $\gamma_t=1/t$, and iteration number $T=100$, which is the setting SparseVI uses for the experiment~\citep{campbell2019sparse}.

IHT's objective indicates both bounded forward KL and reverse KL. However, SparseVI, which optimizes the reverse KL, offers no guarantee for the forward KL. As shown in Figure~\ref{fig:exp-3-kl}~(b), SparseVI increasingly deviates from the true distribution in forward KL as the coreset growths. However, IHT methods offers consistently better coresets  in both the metrics.

The reverse KL divergence alone is not enough to indicate good approximation, as shown in Figure~\ref{fig:exp-3-coresets}. We plot the posterior contours for both the true posterior and coreset posterior at a random trial when sparsity level $k=220, 260, 300$.  The coreset posterior constructed by our Algorithm~\ref{alg:a-iht-2} recovers the true posterior almost exactly at $k=300$, unlike SparseVI. The results for other trials are provided in section~\ref{sec:rad-add} in the appendix.

\subsection{Bayesian logistic and Poisson regression}\label{sec:exp-2}

We consider how IHT performs when used in real applications where the closed-form expressions are unattainable. Moreover, large-scale datasets are considered to test running time of each algorithm. As the true posterior is unknown, a Laplace approximation is used for GIGA and IHT to derive the finite projection of the distribution, \emph{i.e.}, $\hat{g}_i$. Further, Monte Carlo sampling is used to derive gradients of $D_{\text{KL}}$ for SparseVI. We compare different algorithms estimating the posterior distribution for logistic regression and Poisson regression. The reverse KL and forward KL between the coreset posterior and true posterior are estimated using another Laplace approximation. The mode of the Laplace approximation is derived by maximizing the corresponding posterior density. The experiment was proposed by \cite{campbell2019automated}, and is used in \citep{campbell2018bayesian} and \citep{campbell2019sparse}.  Due to space limitations, we refer to section~\ref{sec:exp-lp} in the appendix for details of the experimental setup, and extensive additional results.

\begin{figure}[t]
\centering
	%\vspace{-2em}
    \begin{minipage}[c]{\linewidth}\centering
		%\vspace{-2em}
		\includegraphics[width=0.49\linewidth]{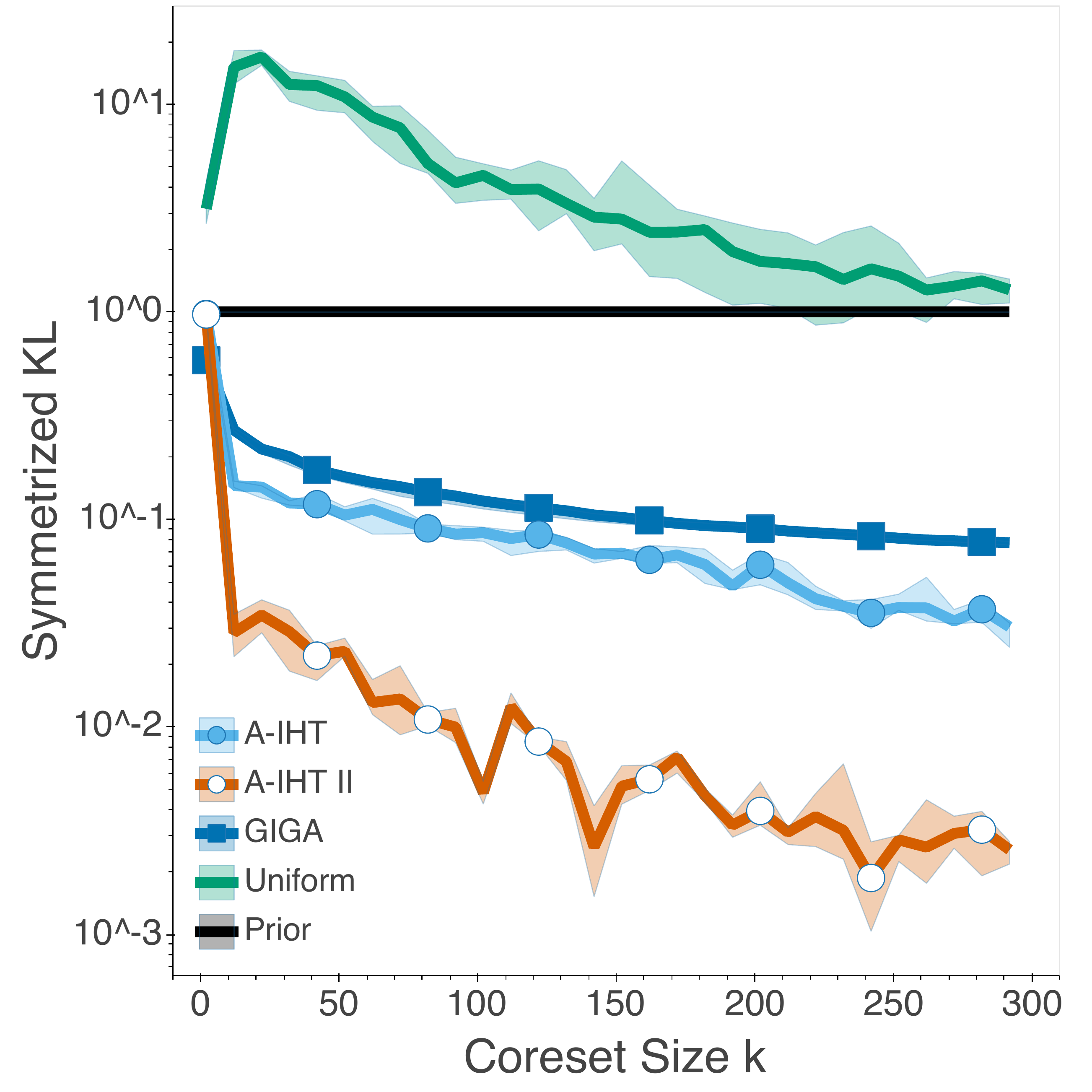}
		\includegraphics[width=0.49\linewidth]{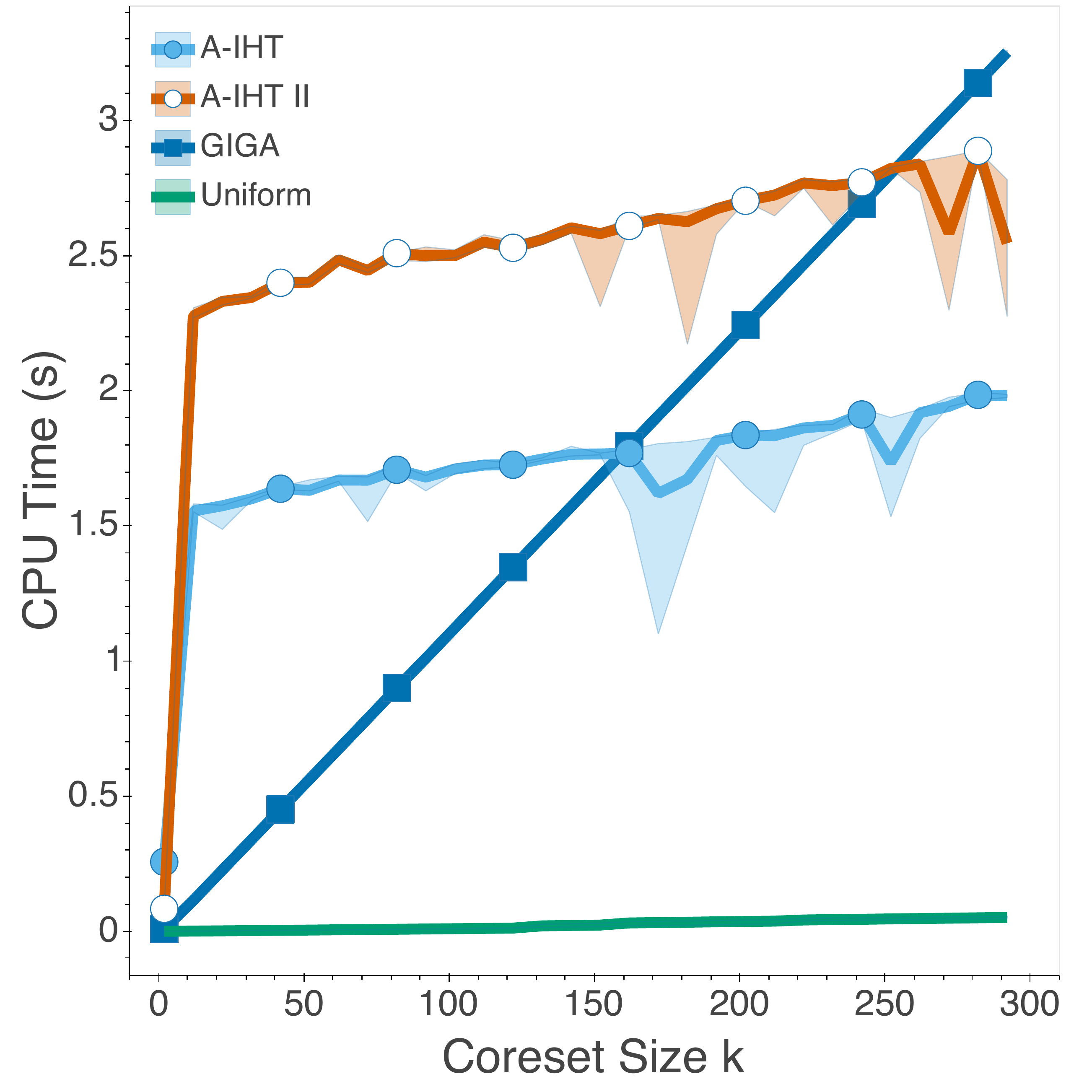}
		%\vspace{-4em}
		\small{Top Row: \texttt{large synthetic dataset}}
	\end{minipage}
	\begin{minipage}[c]{\linewidth}\centering
		\vspace{0.5em}
		\includegraphics[width=0.49\linewidth]{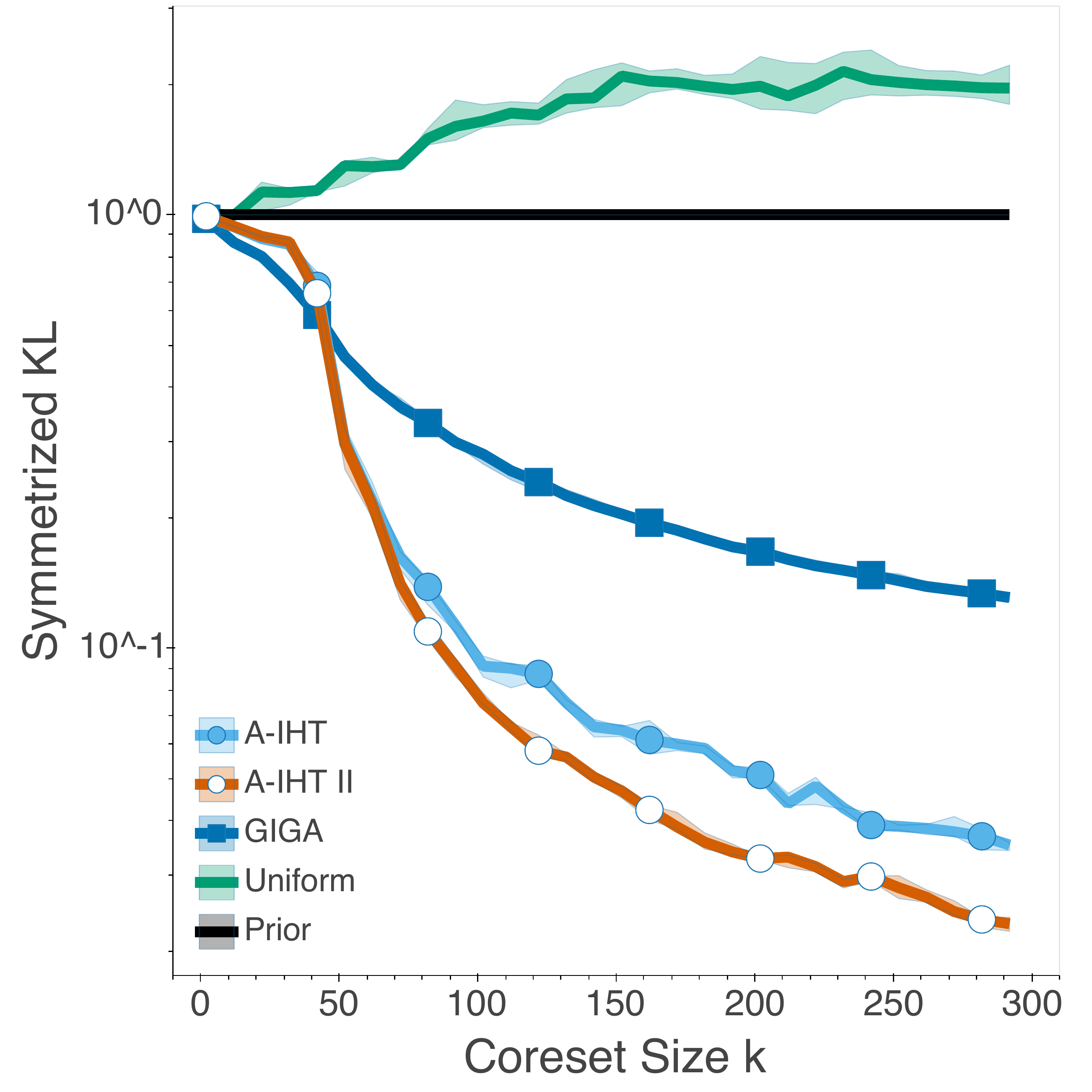}
		\includegraphics[width=0.49\linewidth]{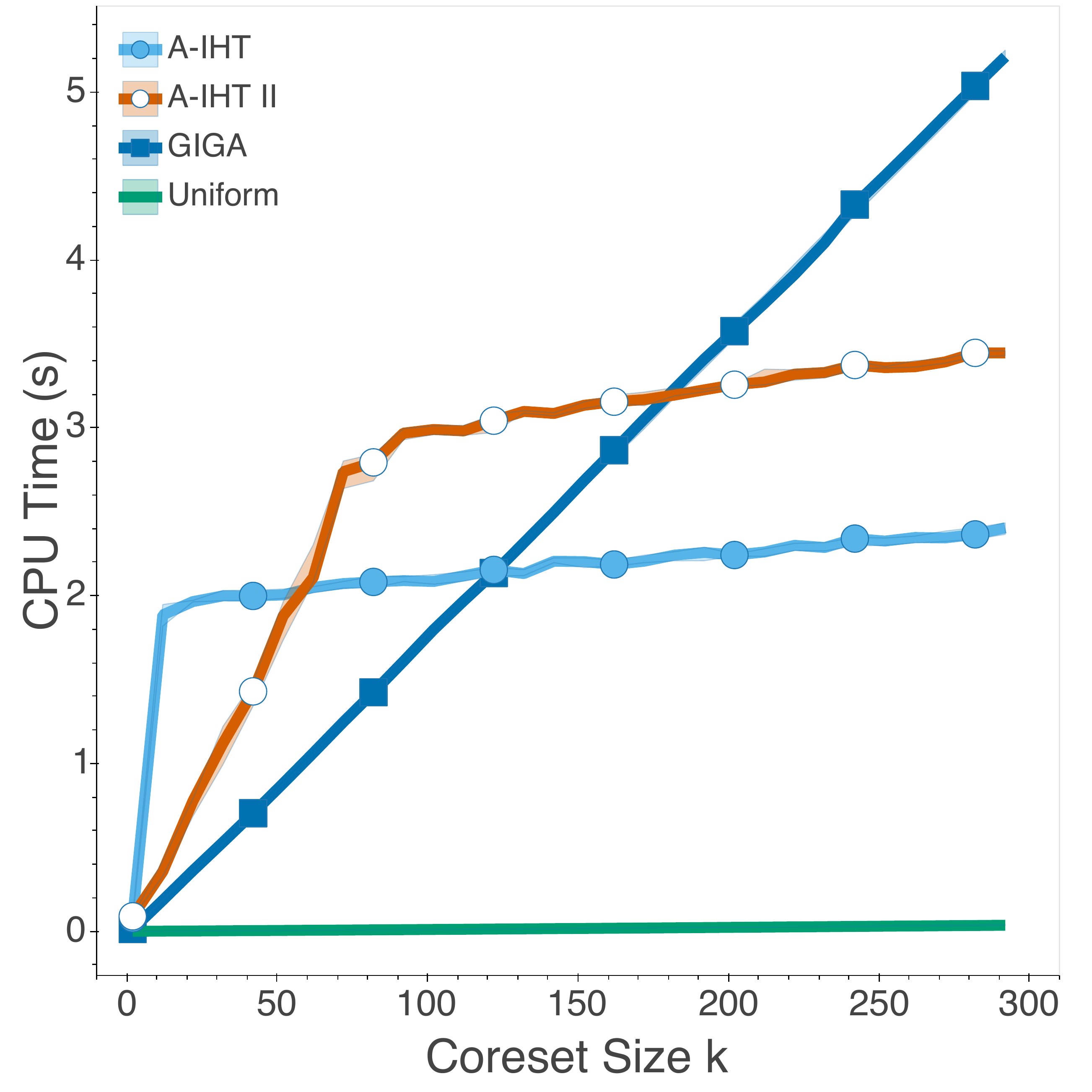}
		%\vspace{-4em}
		\small{Bottom Row: \texttt{original phishing dataset}}
	\end{minipage}
	
	\caption{Bayesian coreset construction for logistic regression (LR) using the \texttt{large synthetic dataset} (top row) and the \texttt{original phishing dataset} (bottom row). All the algorithms are run $10$ times, and the median as well as the interval of $35^{th}$ and $65^{th}$ percentile, indicated as the shaded area, are reported. Different maximal coreset size $k$ is tested. Symmetrized  KL divergence between estimated true posterior and coreset posterior indicate the quality of the constructed coreset (left column).  The running time for each algorithms is also recorded (right column).} \label{fig:exp-3-large}
	%\vspace{-0.4cm}
\end{figure}  

For logistic regression, given a dataset $\{(x_n, y_n)\in \mathbb{R}^D \times \{1, -1\} \mid n\in [N]\}$, we aim to infer $\theta \in \mathbb{R}^{D+1}$ based on the model:
\begin{align}
    y_n \mid x_n, \theta \sim \text{Bern}\left(\frac{1}{1+e^{-z_n^\top \theta}}\right),
\end{align}
where $z_n = [x_n^\top, 1]^\top$. We set $N=500$ by uniformly sub-sampling from datasets due to the high computation cost of SparseVI. Three datasets are used for logistic regression. Two of them are: the \texttt{synthetic dataset} consists of $x_n$ sampled i.i.d. from normal distribution $\gN(0,I)$, and label $y_n$ sampled from Bernoulli distribution conditioned on $x_n$ and $\theta=[3,3,0]^\top$. The \texttt{phishing dataset}\footnote{\url{https://www.csie.ntu.edu.tw/~cjlin/libsvmtools/datasets/binary.html}} is preprocessed~\citep{campbell2019sparse} via PCA to dimension of $D=10$ to mitigate high computation by SparseVI.

We present two sets of experiments, \emph{i.e.}, logistic regression using the \texttt{synthetic dataset} and the \texttt{phishing dataset}, in Figure~\ref{fig:exp-3-phishing}. One other set of experiments on logistic regression, and three sets of experiments on Poisson regression are deferred to section~\ref{sec:exp-lp} in appendix.

It is observed that A-IHT and A-IHT \RomanNumeralCaps{2} achieve state-of-the-art performance. The IHT algorithms often obtain coresets with smaller KL between the coreset posterior and true posterior than GIGA and SparseVI, with computing time comparable to GIGA and significantly less than SparseVI. We conjecture that GIGA and SparseVI perform worse than our methods due to their greedy nature: they can be "short-sighted" and do not rectify past decisions.
The experiments indicate that IHT outperforms the previous methods, improving the trade-off between accuracy and performance.

\textbf{Large-scale Datasets.} Two large datasets are considered: $i)$ the \texttt{large synthetic dataset} for logistic regression is generated following the same procedure as before, but with dataset size $N=9000$; $ii)$ the \texttt{original phishing} dataset has size $N=11055$ and dimension $D=68$. The maximal iteration number of the two IHT algorithms is $500$. Symmetrized KL, \emph{i.e.}, the sum of forward and reverse KL,  is reported.

Results are shown in Figure~\ref{fig:exp-3-large}.  We have to omit SparseVI due to its prohibitively high cost (\emph{e.g.}, as shown in Figure~\ref{fig:exp-3-phishing}, SparseVI needs $\times10^4$ more time than IHT and GIGA). As our complexity analysis of the algorithms in subsection~\ref{subsec:acc-iht}, the running time of GIGA grows linearly with respect to the coreset size $k$, while that is almost free for IHT. GIGA begins to cost more time than IHT at $k\approx200$, \emph{i.e.}, about only $2\%$ of the dataset. 

\textbf{Additional evaluation.} 
For large-scale datasets, it is often necessary to "batch" the algorithms. We test the performance of IHT using a stochastic gradient estimator. The gradient estimator is calculated with random batches in each iteration, where we use a batch size of $20\%$ of the full dataset size. Results on six datasets are defer to section~\ref{sec:exp-lp} in appendix.

Moreover, as an alternative evaluation of the quality of constructed coresets, we test the $\ell_2$-distance between the maximum-a-posteriori  (MAP) estimation of the full-dataset posterior and coreset posterior. Results on six datasets are deferred to section~\ref{sec:exp-lp} in appendix.

\section{Conclusion}
In this paper, we consider the Bayesian coreset construction problem from a sparse optimization perspective, through which we propose a new algorithm that incorporates the paradigms of sparse as well as accelerated optimization. We provide theoretical analysis for our method, showing linear convergence under standard assumptions. Finally, numerical results demonstrate the improvement in both accuracy and efficiency when compared to the state of the art methods. Our viewpoint of using sparse optimization for Bayesian coresets can potentially help to consider more complex structured sparsity, which is left as future work.

\clearpage
\section*{Acknowledgements}
AK acknowledges funding by the NSF (CCF-1907936, CNS-2003137). AK thanks TOOL's Danny Carey for his percussion performance in ``Pneuma''. We would like to thank the reviewers for their valuable and constructive comments. Their feedback enables us to further improve the paper.

\bibliographystyle{iclr2021_conference.bst}
\bibliography{ref}
\clearpage
%\newpage
\onecolumn
\appendix

\vspace*{\baselineskip}
{\centering{\Large\textbf{Bayesian Coresets: \\Revisiting the Nonconvex Optimization Perspective\\ Appendix}}
\vspace*{2.5\baselineskip}

\textbf{Appendix Contents}
\begin{itemize}[leftmargin=5em,rightmargin=5em]
    \item Section~\ref{sec:iht2}: Automated Accelerated IHT with de-bias step (Algorithm~\ref{alg:a-iht-2}).
    \item Section~\ref{sec:theory}: Theoretical Analysis 
    \begin{itemize}
        \item Detailed theoretical analysis on the convergence of our main algorithm, \emph{i.e.}, Automated Accelerated IHT in Algorithm~\ref{alg:a-iht}.
    \end{itemize}
    \item Section~\ref{sec:proof}: Proofs 
    \begin{itemize}
        \item Proofs of the theories presented in section~\ref{sec:theory}.
    \end{itemize} 
    \item Section~\ref{sec:supp-related}: Additional related work.
    \item Section~\ref{sec:add-exp-1}: Additional Results for Synthetic Gaussian Posterior Inference (experiments introduced in section~\ref{sec:exp-1}).
    \begin{itemize}
        \item Convergence speed of the two proposed IHT algorithms.
        \item Illustration of the coresets constructed by A-IHT II.
    \end{itemize}
    \item Section~\ref{sec:rad-add}: Additional Results for Radial Basis Regression (experiments introduced in section~\ref{sec:exp-radial}).
    \begin{itemize}
        \item Additional experimental results of posterior contours for the radial basis regression experiment.
    \end{itemize}
    \item Section~\ref{sec:exp-lp}: Details and Extensive Results of the Bayesian logistic and Poisson regression Experiments (experiments introduced in section~\ref{sec:exp-2}).
    \begin{itemize}
        \item Details of the Bayesian logistic and Poisson regression Experiments.
        \item Results on all of the six datasets.
        \item Results with a stochastic gradient estimator using batches of data.
        \item Results with alternative evaluation on coresets quality---$\ell_2$-distance between the maximum-a-posteriori (MAP) estimation of the full-dataset posterior and coreset posterior.
    \end{itemize}
\end{itemize}}

\section{Automated Accelerated IHT with De-bias Step}\label{sec:iht2}

In the main text, we mention that Algorithm~\ref{alg:a-iht} can be boosted better in practice using de-bias steps. Here we present the algorithm with de-bias step, as shown in Algorithm~\ref{alg:a-iht-2}. 

Like Automated Accelerated IHT, Algorithm~\ref{alg:a-iht-2} also starts with active subspace expansion, \emph{i.e.}, line 3 \& 4. As $\gZ = \supp(z_t)= \supp(w_{t-1}) \cup \supp(w_t)$ is a $2k$-sparse index set,  the expanded index set $\gS$ is a $3k$-sparse index set that is the union of the support of three elements, \emph{i.e.}, 
\begin{align}
    \gS = \supp(w_{t-1}) \cup \supp(w_t) \cup \supp(\Pi_{\mathcal{C}_k\backslash \gZ} \left( \nabla f(z_t) \right)).
\end{align}
We note that, with a little abuse of notation, we use $\gZ$ to denote both the support set $\mathcal{Z}\subset [n]$, and the subspace restricted by the support, \emph{i.e.}, $\{x\in \sR^n \mid \supp(x)\subseteq\mathcal{Z}\}$, depending on the context. 

The subspace corresponding to this index set $\gS$ is a subspace that the algorithm considers as potential to achieve low loss within. Therefore, in the next step, we perform projected gradient descent in this expanded subspace. Note that we use $\nabla f(\cdot)\big|_{\gS}$ to denote a sparse subset $\gS$ of the gradient, \emph{i.e.}, setting the $i^{th}$ entry of $\nabla f(\cdot)$ to $0$ if $i\notin \gS$.

The projected gradient descent step consists of three sub-steps, \emph{i.e.}, step size selection (line 6), gradient descent (line 7), and projection to non-negative $k$-sparse restricted domain (line 7). The step size selection is performed by an exact line search to obtain a good step size automatically. The projection step (line 7) is where we do ``hard thresholding'' to obtain a $k$-sparse solution $x_t$. As mentioned before, this projection step can be done optimally in the sense of $\ell_2$-norm by choosing the $k$-largest non-negative elements.

Then, we come to the key difference between Algorithm~\ref{alg:a-iht} and Algorithm~\ref{alg:a-iht-2}, \emph{i.e.}, the de-bias step at line 8, 9 \& 10. With additional de-bias steps, we adjust the solution $k$-sparse solution $x_t$ inside its own sparse space, \emph{i.e.}, the space corresponding to $\supp(x_t)$, such that a better $k$-sparse solution is found. After computing the gradient (line 8), another exact line search is performed (line 9). By gradient descent and imposing the non-negativity constraint (line 10), we have the  solution $w_{t+1}$ for this iteration.

Lastly, the momentum step (line 11 \& 12) is the same as Algorithm~\ref{alg:a-iht}. We select the momentum term as the minimizer of the objective: $\tau_{t+1} = \argmin_\tau f(w_{t+1}+\tau (w_{t+1}-w_t))$, and then apply the momentum to our solutions $w_{t+1}$ and $w_t$ as $z_{t+1}=w_{t+1}+\tau_{t+1} (w_{t+1}-w_t)$ to capture memory in the algorithm. Momentum can offer faster convergence rate for convex optimization~\citep{nesterov1983method}.

\begin{algorithm}[t]
	\caption{Automated Accelerated IHT - II (A-IHT II)}
	\label{alg:a-iht-2}
	\begin{algorithmic}[1]
		\INPUT{Objective $f(w)=\|y-\Phi w\|_2^2$; sparsity $k$}
		\STATE $t=0$, $z_0=0$, $w_0=0$
		\REPEAT
		\STATE $\gZ=\supp(z_t)$
		\STATE $\mathcal{S}=\supp(\Pi_{\mathcal{C}_k\backslash \gZ} \left( \nabla f(z_t) \right)) \cup \gZ$ where $|\gS|\leq 3k$  \hfill\COMMENT{active subspace expansion}
		\STATE $\widetilde{\nabla}^{(1)}= \nabla f(z_t)\big|_\mathcal{S}$
		\STATE $\mu_t^{(1)}=\argmin_\mu f(z_t-\mu \widetilde{\nabla}^{(1)} )= \frac{\|\widetilde{\nabla}^{(1)}\|_2^2}{2\|\Phi \widetilde{\nabla}^{(1)}\|^2_2}$ \hfill\COMMENT{step size selection}
		\STATE $x_t = \Pi_{\mathcal{C}_k \cap \sR^n_+}\left(z_t-\mu_t^{(1)} {\nabla}f(z_t)\right)$  \hfill\COMMENT{projected gradient descent}
		\STATE $\widetilde{\nabla}^{(2)} =   \nabla f(x_t)\big|_{\supp(x)}  $
		\STATE $\mu_t^{(2)}=\argmin_\mu f(x_t-\mu \widetilde{\nabla}^{(2)} )= \frac{\|\widetilde{\nabla}^{(2)}\|_2^2}{2\|\Phi \widetilde{\nabla}^{(2)}\|^2_2}$    \hfill\COMMENT{step size selection}
		\STATE $w_{t+1}=\Pi_{\mathbb{R}^n_+}(x_t-\mu_t^{(2)} \widetilde{\nabla}^{(2)})$  \hfill\COMMENT{de-bias step}
		\STATE $\tau_{t+1} = \argmin_\tau f(w_{t+1}+\tau (w_{t+1}-w_t)) = \frac{\langle y-\Phi w_{t+1}, \Phi (w_{t+1}-w_t)\rangle}{2\|\Phi (w_{t+1}-w_t)\|^2_2}$
		\STATE $z_{t+1}=w_{t+1}+\tau_{t+1} (w_{t+1}-w_t)$  \hfill\COMMENT{momentum step}
		\STATE $t=t+1$
		\UNTIL Stop criteria met
		\STATE {\bfseries return} $w_t$
	\end{algorithmic}
\end{algorithm}

\section{Theoretical Analysis}\label{sec:theory}

In this section, we provide a detailed theoretical analysis that is abstracted in the main paper due to space limitation. All of the proofs are defer to section~\ref{sec:proof} for clarity. To begin with, let us show that all of the projection operators used in our algorithms can be done optimally and efficiently.  

Given an index  set $\mathcal{S}\subseteq [n]$, the projection of $w$ to the subspace with support  $\mathcal{S}$ is $\Pi_{\mathcal{S}} (w)$, which can be done optimally by setting $w_{\mathcal{S}^c}=0$, where $\mathcal{S}^c$  denotes the complement of $\mathcal{S}$. We note that, with a little abuse of notation, we use $\gS$ to denote both the support set $\mathcal{S}\subset [n]$, and the subspace restricted by the support, \emph{i.e.}, $\{x\in \sR^n \mid \supp(x)\subseteq\mathcal{S}\}$. The projection to non-negative space, \emph{i.e.}, $\Pi_{\mathbb{R}_+^n} (w)$,  can also be done optimally and efficiently by setting the negative entries to zero.  Moreover, $\Pi_{\mathcal{C}_k} $ is shown to be optimal by simply picking the top $k$ largest (in absolute value) entries. It is also the case for $\Pi_{\mathcal{C}_k \cap \mathbb{R}_+^n} (w)$, where it can be done by picking the top $k$ largest non-negative entries. The optimality for the above projections is in terms of Euclidean distance.

Let us show the optimality for $\Pi_{\mathcal{C}_k \cap \mathbb{R}_+^n} (w)$. Given a $k$-sparse support $\mathcal{S}$, the optimal projection of $w\in \sR^n$ to its restricted sparsity space intersecting the non-negative orthant is $w'=\Pi_{\gS \cap \mathbb{R}_+^n} (w)$. We can see that for entry $i\in [n]$, $w'_i=w_i$ if $i\in\gS$ and $w_i\geq 0$, and $w'_i=0$ otherwise. Therefore, the distance between $w$ and its projection to $\gS \cap \mathbb{R}_+^n$ is $\|w'-w\|^2_2=\|w\|_2^2-\sum_{i\in S, w_i > 0} w_i^2$. As $\Pi_{\mathcal{C}_k \cap \mathbb{R}_+^n} (w) = \min_{\gS:|\gS|\leq k} \Pi_{\gS \cap \mathbb{R}_+^n} (w)$, we can see that it is the support with $k$ largest $w_i$ that has the least distance. Therefore, simply picking top $k$ largest non-negative entries gives the optimal projection.

We give the convergence analysis for our main algorithm Automated Accelerated IHT in Algorithm~\ref{alg:a-iht}. One standard assumption about the objective is required for the theory to begin, \emph{i.e.,} RIP property, which is a normal assumption in IHT context, reflecting convexity and smoothness of the objective in some sense \citep{khanna2018iht, kyrillidis2014matrix}. We note that the assumption is not necessary but is sufficient. 
For example, if the number of samples required to exactly construct $\hat g$ is less than the coreset size ($a_k=0$ in RIP), so that the system becomes underdetermined, then local minima can be global one achieving zero-error without the RIP.
On the other hand, when the number of samples goes to infinity, RIP ensures the eigenvalues of the covariance matrix, $cov[\gL_i(\theta), \gL_j(\theta)]$ where $\theta\sim\hat \pi$, are lower and upper bounded. It is an active area of research in random matrix theory to quantify RIP constants \emph{e.g.} see~\citep{Baraniuk2008ASP}.
\setcounter{assumption}{0}
\begin{assumption}[Restricted Isometry Property]
    Matrix $\Phi$ in the objective function satisfies the RIP property, \emph{i.e.}, for $\forall w\in \mathcal{C}_k$
    \begin{align}
        \alpha_k \|w\|_2^2\leq \|\Phi w\|_2^2 \leq \beta_k \|w\|^2_2.
    \end{align}
\end{assumption}
It is known that there are connections between RIP and restricted strong convexity and smoothness assumptions \citep{chen2010general}; thus our results could potentially generalized for different convex $f(\cdot)$ functions.

Leading to our main theorem, some useful technical  properties are presented. An useful observation is that,  for any set $\mathcal{S} \subseteq [n]$, the projection operator $\Pi_\mathcal{S}:\mathbb{R}^n\to \mathbb{R}^n$ is in fact a linear operator in the form of a diagonal matrix
\begin{align}
    \Pi_\mathcal{S} = \{\text{diag}(\delta_i)\}_{i=1}^n,
\end{align}
where $\delta_i$ is an indicator function: $\delta_i=1$ if $i\in \mathcal{S}$, and $\delta_i = 0$  otherwise. This leads to our first lemma.

\begin{lemma}\label{lemma:1}
    Supposing $\Phi$ satisfies the RIP assumption, given a sparse set $\mathcal{S} \subseteq [n]$ and $|\mathcal{S}|\leq k$, for $\forall w\in \mathbb{R}^n$ it holds that
    \begin{align}
        \alpha_k \|\Pi_\mathcal{S}w\|_2  \leq \|\Pi_\mathcal{S} \Phi^\top \Phi \Pi_\mathcal{S} w\|_2 \leq  \beta_k \|\Pi_\mathcal{S}w\|_2 .
    \end{align}
\end{lemma}
Lemma~\ref{lemma:1} reveals a property of the eigenvalues of $\Pi_{\gS}\Phi^\top\Phi\Pi_{\gS}$, which leads to the following lemma that bounds an iterated projection using the RIP property.
%For notation simplicity, we denote the projection to the complement set of $\mathcal{S}$ as $\Pi_{\mathcal{S}^c}:=\Pi_{[n]\backslash \mathcal{S}}$. We now present a result that bounds iterated projections using the RIP property.
\begin{lemma}\label{lemma:2}
     Supposing $\Phi$ satisfies the RIP assumption, given two sets $\mathcal{S}_1, \mathcal{S}_2 \subseteq [n]$ and $|\mathcal{S}_1\cup \mathcal{S}_2|\leq k$, for $\forall w\in \mathbb{R}^n$ it holds that
    \begin{align}
        \|\Pi_{\mathcal{S}_1} \Phi^\top \Phi \Pi_{\mathcal{S}_1^c} \Pi_{\mathcal{S}_2}w\|_2 \leq \tfrac{\beta_k -\alpha_k}{2} \cdot \|\Pi_{\mathcal{S}_2}w\|_2.
    \end{align}
\end{lemma}

Armed with above two lemmas, we are ready to prove convergence for Automated Accelerated IHT (Algorithm~\ref{alg:a-iht}). A key observation is that solution $w_{t+1}$ found by Algorithm~\ref{alg:a-iht} is derived by the following two steps:
\begin{align}
   \{w_t, w_{t-1}\}\xRightarrow[\text{line 9}]{\text{\circled{1}}} z_t %\xRightarrow[\text{line 6}]{\text{\circled{2}}} x_t
   \xRightarrow[\text{line 7}]{\text{\circled{2}}} w_{t+1}.
\end{align} 

Procedure \circled{1} is a momentum step, with momentum size chosen automatically; procedure \circled{2} aims for exploration in an expanded subspace spanned by a $3k$-sparse subset $\mathcal{S}$, and projecting to $k$-sparse non-negative subspace. 

We break down the proof into two parts. Denoting the optimal solution as \[w^\star=\argmin_{w\in \mathcal{C}_k\cap \mathbb{R}^n_+} \|y-\Phi w\|_2^2 ,\]
we propose the following two lemmas for the two steps respectively.

\begin{lemma}\label{prop:1}
    For procedure \circled{1}, the following iterative invariant holds.
    \begin{align}
        \|z_t-w^\star\|_2\leq |1+\tau_t|\cdot\|w_t - w^\star\|_2+|\tau_t|\cdot\|w_{t-1}-w^\star\|_2.
    \end{align}
\end{lemma}
For the second procedure, we consider the actual step size $\mu_t$ automatically chosen by the algorithm. Noting that $|\supp(\widetilde{\nabla}_t)|\leq 3k$, according to RIP we can see that the step size $\mu_t=\frac{\|\widetilde{\nabla}_t\|_2^2}{2\|\Phi\widetilde{\nabla}_t \|_2^2}$ is bounded as
\begin{align}
    \frac {1}{ 2\beta_{3k}}\leq \mu_t \leq \frac {1}{ 2\alpha_{3k}}.
    %, \qquad \frac {1}{ 2\beta_{k}}\leq \mu_t^{(2)} \leq \frac {1}{ 2\alpha_{k}}
\end{align}
Therefore, using the Lemma~\ref{lemma:1} and Lemma~\ref{lemma:2}, one can prove the following lemma.
\begin{lemma}\label{prop:2}
    For procedure \circled{2}, the following iterative invariant holds.
    \begin{align}
        \|w_{t+1}-w^\star\|_2 &\leq \rho\|z_t-w^\star\|_2 + 2\beta_{3k} \sqrt{\beta_{2k}} \| \epsilon\|_2 ,
    \end{align}
    where $\rho = \left( 2 \max\{ \frac{\beta_{2k}}{\alpha_{3k}} - 1, 1-\frac{\alpha_{2k}}{\beta_{3k}}\} + \frac{\beta_{4k}-\alpha_{4k}}{\alpha_{3k}} \right)$, and $\|\epsilon\|_2 = \|y-\Phi w^\star\|_2$ is the optimal objective value.
\end{lemma}

Combining the above two lemmas leads to our main convergence analysis theorem.
\setcounter{theorem}{0}
\begin{theorem}[Restated]
In the worst case scenario, with Assumption~\ref{assum:rip}, the solutions path find by Automated Accelerated IHT (Algorithm~\ref{alg:a-iht}) satisfy the following iterative invariant. 
\begin{align}
        \|w_{t+1}-w^\star\|_2 &\leq \rho|1+\tau_t|\cdot \|w_t - w^\star\|_2 + \rho |\tau_t| \cdot {\|w_{t-1}-w^\star\|_2}  + 2\beta_{3k} \sqrt{\beta_{2k}} \| \epsilon\|_2,
\end{align}
where $\rho = \left( 2 \max\{ \frac{\beta_{2k}}{\alpha_{3k}} - 1, 1-\frac{\alpha_{2k}}{\beta_{3k}}\} + \frac{\beta_{4k}-\alpha_{4k}}{\alpha_{3k}} \right)$, and $\|\epsilon\|_2 = \|y-\Phi w^\star\|_2$ is the optimal objective value.
\end{theorem}

The theorem provides an upper bound invariant among consecutive iterates of the algorithm. To have better sense of convergence rate, we assume the optimal solution achieves $\|\epsilon\|_2=0$. Theorem~\ref{the:1} then implies
\begin{align}
\|w_{t+1}-w^\star\|_2 &\leq \rho(1+|\tau_t|) \|w_t - w^\star\|_2 + \rho |\tau_t| \cdot {\|w_{t-1}-w^\star\|_2}.
\end{align}

Given the above homogeneous recurrence, we can solve for the following corollary that shows linear convergence of the proposed algorithm under given conditions.

\setcounter{corollary}{0}
\begin{corollary}[Restated]%\label{coro:1}
	 Given the iterative invariant as stated in Theorem~\ref{the:1}, and assuming the optimal solution achieves $\|\epsilon\|_2=0$, the solution found by Algorithm~\ref{alg:a-iht} satisfies:
	 \begin{align}
	     f(w_{t+1})-f(w^\star)\leq \phi^t\left( \frac{\beta_{2k}}{\alpha_{2k}} f(w_1) + \frac{\rho\tau\beta_{2k}}{\phi\alpha_{k}} f(w_0) \right),
	 \end{align}
	 where $\phi = (\rho(1+\tau) +\sqrt{\rho^2(1+\tau)^2+4\rho \tau})/2$ and $\tau = \max_{i\in [t]} |\tau_i|$. It is sufficient to show linear convergence to the global optimum, when $\phi<1$, or equivalently $\rho<1/(1+2\tau)$.
\end{corollary}

\section{Proofs}\label{sec:proof}
This section provides proofs for the theoretical results presented in the previous section. For the sake of good readability, the lemma/theorem to be proven is also restated preceding its proof.
\setcounter{lemma}{0}
\setcounter{theorem}{0}

\subsection{Proof of Lemma~\ref{lemma:1}}
\begin{lemma}[Restated]
    Supposing $\Phi$ satisfies the RIP assumption, given a sparse set $\mathcal{S} \subseteq [n]$ and $|\mathcal{S}|\leq k$, for $\forall w\in \mathbb{R}^n$ it holds that
    \begin{align}
        \alpha_k \|\Pi_\mathcal{S}w\|_2  \leq \|\Pi_\mathcal{S} \Phi^\top \Phi \Pi_\mathcal{S} w\|_2 \leq  \beta_k \|\Pi_\mathcal{S}w\|_2 .
    \end{align}
\end{lemma}
\begin{proof}
    Recall that $\Pi_\gS$ is a linear operator that projects a vector $w\in \sR^n$ to sparse restricted set with support $\gS$ by simply setting $w_i=0$ for each $i\notin \gS$. As a result, for a $k$-sparse set $\gS$, $\Pi_\gS w$ is a $k$-sparse vector. Given that $\Phi \in \mathbb{R}^{m\times n}$ satisfies RIP property, for $\forall w\in \mathbb{R}^n$, it holds that
    \begin{align}
        \alpha_k \|\Pi_\mathcal{S}w\|^2_2 \leq \|\Phi \Pi_\mathcal{S} w\|^2_2 \leq \beta_k \|\Pi_\mathcal{S}w\|^2_2 .  \label{eq:1-0}
    \end{align}

    Let us denote $b=\Phi \Pi_\mathcal{S} w$, and $\langle\cdot,\cdot\rangle$ as standard Euclidean inner product. With regular linear algebra manipulation, the following stands:
    \begin{align}
        \|\Pi_\mathcal{S} \Phi^\top b\|^2_2&=\max_{x\in \mathbb{R}^n:\|x\|_2=1} \left( \langle \Pi_\mathcal{S} \Phi^\top b, x \rangle \right)^2\\
        &= \max_{x\in \mathbb{R}^n:\|x\|_2=1} \left( b^\top \Phi \Pi_\mathcal{S}x \right)^2\\
        &= \max_{x\in \mathbb{R}^n:\|x\|_2=1} \left( \langle b, \Phi \Pi_\mathcal{S}x \rangle \right)^2\\
        &= \max_{x\in \mathbb{R}^n:\|x\|_2=1} \left( \langle \Phi \Pi_\mathcal{S} w, \Phi \Pi_\mathcal{S}x \rangle \right)^2 \label{eq:1-1},
    \end{align}
    where the second equality is due to the fact that $\Pi_\gS$ is symmetric, \emph{i.e.}, $(\Pi_\mathcal{S} \Phi^\top b)^\top = b^\top \Phi \Pi_\mathcal{S}$.
    
    Letting $x^\star$ be the solution of (\ref{eq:1-1}), we have the upper bound of (\ref{eq:1-1}):
    \begin{align}
        (\ref{eq:1-1}) = \left( \langle \Phi \Pi_\mathcal{S} w, \Phi \Pi_\mathcal{S}x^\star \rangle \right)^2\leq \|\Phi \Pi_\mathcal{S} w\|_2^2\cdot \|\Phi \Pi_\mathcal{S}x^\star\|_2^2,
    \end{align}
    where the inequality is by Cauchy-Schwarz inequality applying on inner product. 
    
    On the other hand, the lower bound can be obtained by removing the maximizing operator and setting $x=\Pi_\gS w/\|\Pi_\gS w\|_2$, as follows. Denoting $x'=\Pi_\gS w/\|\Pi_\gS w\|_2$, we have,
    \begin{align}
        (\ref{eq:1-1})\geq  \left( \langle \Phi \Pi_\mathcal{S} w, \Phi \Pi_\mathcal{S} x' \rangle \right)^2=
        \|\Phi \Pi_\mathcal{S} w\|_2^2\cdot \|\Phi \Pi_\mathcal{S}x'\|_2^2,
    \end{align}
    where the last equality is due to that $\Pi_\gS w$ and $x'$ are parallel.
    
    Applying (\ref{eq:1-0}) to the above upper bound and lower bound, it follows that 
     \begin{align}
        (\ref{eq:1-1}) \leq \|\Phi \Pi_\mathcal{S} w\|_2^2\cdot \|\Phi \Pi_\mathcal{S}x^\star\|_2^2 \leq \beta_k \|\Pi_\mathcal{S}w\|^2_2 \cdot \beta_k\|\Pi_\mathcal{S}x^\star\|^2_2, \\
        (\ref{eq:1-1}) \geq \|\Phi \Pi_\mathcal{S} w\|_2^2\cdot \|\Phi \Pi_\mathcal{S}x'\|_2^2 \geq \alpha_k \|\Pi_\mathcal{S}w\|^2_2 \cdot \alpha_k\|\Pi_\mathcal{S}x'\|^2_2.\label{eq:1-a}
    \end{align}
    Noting that $x^\star$ is an unit-length vector, and the projection $\Pi_\gS$ is done by setting elements to zero, we can see that $\|\Pi x^\star\|_2\leq 1$. As $x'=\Pi_\gS w/\|\Pi_\gS w\|_2$ has already been a sparse vector in the restricted space by $\gS$, we can see that $\|\Pi_\gS x'\|_2 = \|x'\|_2 = 1$. Plugging them in (\ref{eq:1-a}), it holds that
    \begin{align}
        \alpha_k^2 \|\Pi_\mathcal{S}w\|^2_2 =\alpha_k \|\Pi_\mathcal{S}w\|^2_2 \cdot \alpha_k\|\Pi_\mathcal{S}x'\|^2_2 \leq (\ref{eq:1-1}) \leq \beta_k \|\Pi_\mathcal{S}w\|^2_2 \cdot \beta_k\|\Pi_\mathcal{S}x^\star\|^2_2 \leq \beta_k^2 \|\Pi_\mathcal{S}w\|^2_2.
    \end{align}
    Plugging that $(\ref{eq:1-1}) = \|\Pi_\mathcal{S} \Phi^\top b\|^2_2=\|\Pi_\mathcal{S} \Phi^\top \Phi \Pi_\mathcal{S} w\|_2^2$, and taking the square root, we finally have
    \begin{align}
        \alpha_k \|\Pi_\mathcal{S}w\|_2  \leq \|\Pi_\mathcal{S} \Phi^\top \Phi \Pi_\mathcal{S} w\|_2 \leq  \beta_k \|\Pi_\mathcal{S}w\|_2 .
    \end{align}
    
\end{proof}

\subsection{Proof of Lemma \ref{lemma:2}}
\begin{lemma}[Restated]
     Supposing $\Phi$ satisfies the RIP assumption, given two sets $\mathcal{S}_1, \mathcal{S}_2 \subseteq [n]$ and $|\mathcal{S}_1\cup \mathcal{S}_2|\leq k$, for $\forall w\in \mathbb{R}^n$ it holds that
    \begin{align}
        \|\Pi_{\mathcal{S}_1} \Phi^\top \Phi \Pi_{\mathcal{S}_1^c} \Pi_{\mathcal{S}_2}w\|_2 \leq \tfrac{\beta_k -\alpha_k}{2} \cdot \|\Pi_{\mathcal{S}_2}w\|_2.
    \end{align}
\end{lemma}
\begin{proof}
    Similar to the proof of Lemma~\ref{lemma:1}, we first write the norm in the form of an inner product. Given two sets $\mathcal{S}_1, \mathcal{S}_2 \subseteq [n]$ and $|\mathcal{S}_1\cup \mathcal{S}_2|\leq k$, for $\forall w\in \mathbb{R}^n$, with regular linear algebra manipulation, it holds that
    \begin{align}
        \|\Pi_{\mathcal{S}_1}& \Phi^\top \Phi \Pi_{\mathcal{S}_1^c} \Pi_{\mathcal{S}_2}w\|_2\\
        &= \max_{b\in \mathbb{R}^n: \|b\|_2=1} |\langle b, \Pi_{\mathcal{S}_1} \Phi^\top \Phi \Pi_{\mathcal{S}_1^c} \Pi_{\mathcal{S}_2}w \rangle|\\
        &= \max_{b\in \mathbb{R}^n: \|b\|_2=1} |\langle \Phi \Pi_{\mathcal{S}_1} b, \Phi \Pi_{\mathcal{S}_1^c} \Pi_{\mathcal{S}_2}w \rangle|\label{c-1} ,
    \end{align}
    where the second equality is due to the fact that $\Pi_{\gS_1}$ is symmetric.
    
    Define two unit-length vectors
    \begin{align}
        X=\frac{\Pi_{\mathcal{S}_1^c} \Pi_{\mathcal{S}_2}w }{\|\Pi_{\mathcal{S}_1^c} \Pi_{\mathcal{S}_2}w \|_2}   ,\qquad  Y=\frac{ \Pi_{\mathcal{S}_1} b}{\|  \Pi_{\mathcal{S}_1} b\|},
    \end{align}
     and we can see that $\langle X, Y\rangle=0$, as $\mathcal{S}_1^c$ and $\mathcal{S}_1$ are disjoint. As a result, $\|X+Y\|_2^2=\|X\|_2^2+\|Y\|_2^2=2$.
     Moreover, given that $|\gS_1\cup \gS_2|\leq k$, we can see that $X+Y$ is $k$-sparse. Applying the RIP property, the following holds:
     \begin{align}
        2\alpha_k=\alpha_{k}\|X+Y\|_2^2 \leq \|\Phi X+\Phi Y\|_2^2 \leq \beta_{k}\|X+Y\|_2^2=2\beta_k.
     \end{align}
    Similarly, $\|X-Y\|^2_2=2$ and $X-Y$ is also $k$-sparse:
    \begin{align}
        2\alpha_{k}\leq \|\Phi X-\Phi Y\|_2^2 \leq 2\beta_k .
    \end{align}
    Noting that 
    \begin{align}
        \langle \Phi X, \Phi Y\rangle = \frac{ \|\Phi X+\Phi Y\|_2^2 - \|\Phi X-\Phi Y\|_2^2}{4},
    \end{align}
    we can see the following,
    \begin{align}
        -\frac{\beta_k-\alpha_k}{2}\leq \langle \Phi X, \Phi Y\rangle \leq \frac{\beta_k-\alpha_k}{2} .\label{eq:l2-1}
    \end{align}
    Recall that 
    \[(\ref{c-1})=\max_{\|b\|_2=1} |\langle \Phi X, \Phi Y\rangle|\cdot \|\Pi_{\mathcal{S}_1} b\|_2 \cdot \|\Pi_{\mathcal{S}_1^c} \Pi_{\mathcal{S}_2}w\|_2,\]
    and apply (\ref{eq:l2-1}) to the above, we conclude that
    \begin{align}
        (\ref{c-1})&\leq \max_{\|b\|_2=1} \frac{\beta_k-\alpha_k}{2} \cdot \|\Pi_{\mathcal{S}_1} b\|_2 \cdot \|\Pi_{\mathcal{S}_1^c} \Pi_{\mathcal{S}_2}w\|_2\\
        &\leq \frac{\beta_k-\alpha_k}{2}\| \Pi_{\mathcal{S}_2}w\|_2.
    \end{align}
\end{proof}

\subsection{Proof of Lemma \ref{prop:1}}
\begin{lemma}[Restated]
    For procedure \circled{1}, the following iterative invariant holds.
    \begin{align}
        \|z_t-w^\star\|_2\leq |1+\tau_t|\cdot\|w_t - w^\star\|_2+|\tau_t|\cdot\|w_{t-1}-w^\star\|_2.
    \end{align}
\end{lemma}
\begin{proof}
According to line 9 in Algorithm~\ref{alg:a-iht}, with some regular linear algebra manipulation, we can derive
\begin{align}
    \|z_t-w^\star\|_2 &= \| w_t + \tau_t(w_t-w_{t-1}) -w^\star \|_2\\
    &=\| (1+\tau_t) (w_t-w^\star) + \tau_t(w^\star-w_{t-1}) \|_2\\
    &\leq |1+\tau_t|  \| w_t-w^\star\|_2 + |\tau_t|\|w_{t-1} -w^\star\|_2,
\end{align}
where the last inequality is done by triangle inequality.
    
\end{proof}

\subsection{Proof of Lemma \ref{prop:2}}
\begin{lemma}[Restated]
    For procedure \circled{2}, the following iterative invariant holds.
    \begin{align}
        \|w_{t+1}-w^\star\|_2 &\leq \rho\|z_t-w^\star\|_2 + 2\beta_{3k} \sqrt{\beta_{2k}} \| \epsilon\|_2 ,
    \end{align}
    where $\rho = \left( 2 \max\{ \frac{\beta_{2k}}{\alpha_{3k}} - 1, 1-\frac{\alpha_{2k}}{\beta_{3k}}\} + \frac{\beta_{4k}-\alpha_{4k}}{\alpha_{3k}} \right)$, and $\|\epsilon\|_2 = \|y-\Phi w^\star\|_2$ is the optimal objective value.
\end{lemma}
\begin{proof}
    Denoting $v=z_t-\mu_t\nabla f(z_t)$, and
    set $\mathcal{S}_{\star}=\supp(w_{t+1})\cup \supp(w^\star)$, we begin by the projection at line 7 in Algorithm~\ref{alg:a-iht}. Applying the triangle inequality,
    \begin{align}
        \|w_{t+1}-w^\star\|_2 \leq   \|w_{t+1} - \Pi_{\mathcal{S_\star}}v\|_2 +\|\Pi_{\mathcal{S_\star}}v  -w^\star\|_2 . \label{eq:2-0}
    \end{align}
    As $\gS_\star=\supp(w_{t+1})\cup \supp(w^\star)$, we can observe that $\langle w_{t+1}, \Pi_{\mathcal{S_\star}^c}v \rangle = 0$ and $\langle w^\star, \Pi_{\mathcal{S_\star}^c}v \rangle = 0$. As a result, 
    \begin{align}
        \|w_{t+1} - \Pi_{\mathcal{S_\star}}v\|_2^2 &= \|w_{t+1} - v+ \Pi_{\mathcal{S_\star}^c}v\|_2^2\\
        &= \|w_{t+1} - v\|^2_2+ \|\Pi_{\mathcal{S_\star}^c}v\|_2^2 + 2\langle w_{t+1} - v, \Pi_{\mathcal{S_\star}^c}v\rangle\\
        &= \|w_{t+1} - v\|^2_2+ \|\Pi_{\mathcal{S_\star}^c}v\|_2^2 + 2\langle - v, \Pi_{\mathcal{S_\star}^c}v\rangle\\
        &\leq \|w^\star - v\|^2_2+ \|\Pi_{\mathcal{S_\star}^c}v\|_2^2 +2\langle - v, \Pi_{\mathcal{S_\star}^c}v\rangle\\
        &= \|w^\star - v\|^2_2+ \|\Pi_{\mathcal{S_\star}^c}v\|_2^2 +2\langle w^\star - v, \Pi_{\mathcal{S_\star}^c}v\rangle\\
        &=\|w^\star - v+\Pi_{\mathcal{S_\star}^c}v\|_2^2\\
        &=\|w^\star - \Pi_{\mathcal{S_\star}}v\|_2^2,
    \end{align}
    where the inequality is due to the projection step $w_{t+1}=\Pi_{\mathcal{C}_k\cap \sR^n_+} v$ is done optimally, and $w^\star \in \mathcal{C}_k\cap \sR^n_+$. Plugging the above inequality into (\ref{eq:2-0}), it holds that
    \begin{align}
        \|w_{t+1}-w^\star\|_2 \leq  2\|\Pi_{\mathcal{S_\star}}v  -w^\star\|_2 . \label{eq:2-1}
    \end{align}
 
    Expanding $v$ and denoting $\epsilon=\Phi w^\star-y$, we have
    \begin{align}
        v&=z_t-\mu_t  \left( \nabla f(z_t) \right)\\
        &=z_t-\mu_t  \left( 2\Phi^\top(\Phi z_t-y) \right)\\
        &=z_t-\mu_t  \left( 2\Phi^\top \Phi (z_t-w^\star)+2\Phi^\top (\Phi w^\star-y) \right)\\
        &=z_t- 2\mu_t \Phi^\top \Phi (z_t-w^\star) -2\mu_t  \Phi^\top \epsilon .
        %&=z_t- 2\mu_t\Pi_\mathcal{S} \Phi^\top \Phi (z_t-w^\star)-\mu_t\Pi_\mathcal{S} \left(\nabla f(w^\star) \right) 
    \end{align}
    Plugging the above into inequality (\ref{eq:2-1}), we can further expand
    \begin{align}
        \|w_{t+1}-w^\star\|_2 &\leq  2\|\Pi_{\mathcal{S_\star}}(z_t- 2\mu_t \Phi^\top \Phi (z_t-w^\star) -2\mu_t  \Phi^\top \epsilon ) -w^\star\|_2 \\
        &= 2\|\Pi_{\mathcal{S_\star}} (z_t-w^\star) - 2\mu_t \Pi_{\mathcal{S_\star}}\Phi^\top \Phi (z_t-w^\star) -2\mu_t  \Pi_{\mathcal{S_\star}}\Phi^\top \epsilon  \|_2\\
        &\leq 2\|\Pi_{\mathcal{S_\star}} (z_t-w^\star) - 2\mu_t \Pi_{\mathcal{S_\star}}\Phi^\top \Phi (z_t-w^\star)\|_2 +4\mu_t \| \Pi_{\mathcal{S_\star}}\Phi^\top \epsilon  \|_2\\
        &= 2\|\Pi_{\mathcal{S_\star}} (z_t-w^\star) - 2\mu_t \Pi_{\mathcal{S_\star}}\Phi^\top \Phi I(z_t-w^\star)\|_2 +4\mu_t \| \Pi_{\mathcal{S_\star}}\Phi^\top \epsilon  \|_2\label{eq:2-a} .
    \end{align}
    Expanding the identity matrix by $I=\Pi_{\mathcal{S_\star}}+\Pi_{\mathcal{S_\star}^c}$, we have
    \begin{align}
        (\ref{eq:2-a}) \leq \underbrace{ 2\|(I- 2\mu_t \Pi_{\mathcal{S_\star}}\Phi^\top \Phi \Pi_{\mathcal{S_\star}})\Pi_{\mathcal{S_\star}}(z_t-w^\star)\|_2}_A \\
        + \underbrace{4\mu_t\|\Pi_{\mathcal{S_\star}}\Phi^\top \Phi \Pi_{\mathcal{S_\star}^c}(z_t-w^\star)\|_2}_B\\
        +\underbrace{ 4\mu_t \| \Pi_{\mathcal{S_\star}}\Phi^\top \epsilon  }_C\|_2   .
    \end{align}
    Now we bound the three terms respectively.
    
    Noting that $|\mathcal{S}_{\star}|\leq 2k$, according to Lemma~\ref{lemma:1}, in the subspace with support $\mathcal{S}_{\star}$, \emph{i.e.}, $\{w\mid \supp(w)=\mathcal{S}_{\star}\}$, the eigenvalues $\alpha_{2k}\leq \lambda_{\mathcal{S}_{\star}}( \Pi_{\mathcal{S}_{\star}}\Phi^\top \Phi \Pi_{\mathcal{S}_{\star}})\leq \beta_{2k}$. Therefore, eigenvalues 
    \begin{align}
        \lambda_{\mathcal{S}_{\star}}(I- 2\mu_t\Pi_{\mathcal{S}_{\star}}\Phi^\top& \Phi \Pi_{\mathcal{S}_{\star}}) \in [ 1-2\mu_t\beta_{2k}, 1-2\mu_t\alpha_{2k}],
    \end{align}
    which means
    \begin{align}
        A &\leq 2 \max\{ 2\mu_t\beta_{2k} - 1, 1-2\mu_t\alpha_{2k}\} \|\Pi_{\mathcal{S}_{\star}}(z_t-w^\star)\|_2\\
        &\leq 2 \max\{ \beta_{2k}/\alpha_{3k} - 1, 1-\alpha_{2k}/\beta_{3k}\} \|z_t-w^\star\|_2. \label{eq:2-4}
    \end{align}
    For term B, demoting $\mathcal{S}'=\supp(z_t)\cup \supp(w^\star)$, it can be observed that
    \begin{align}
        B=4\mu_t\|\Pi_{\mathcal{S_\star}}\Phi^\top \Phi \Pi_{\mathcal{S_\star}^c}\Pi_{\mathcal{S}'}(z_t-w^\star)\|_2 .
    \end{align}
    Noting that $|\mathcal{S}'\cup \mathcal{S}_{\star}|\leq 4k$, by directly applying Lemma~\ref{lemma:2} we have
    \begin{align}
        B&\leq 4\mu_t\frac{\beta_{4k}-\alpha_{4k}}{2}\|\Pi_{\mathcal{S}'}(z_t-w^\star)\|_2\\
        &\leq \frac{\beta_{4k}-\alpha_{4k}}{\alpha_{3k}}\|z_t-w^\star\|_2 .
    \end{align}
    To complete the proof, let us deal with the last piece.  Similar to the techniques used in the proof on Lemma~\ref{lemma:1}, 
    \begin{align}
         \|\Pi_{\mathcal{S}_{\star}} \Phi^\top \epsilon\|_2&=\max_{x\in \mathbb{R}^n:\|x\|_2=1} \langle \Pi_{\mathcal{S}_{\star}} \Phi^\top \epsilon, x \rangle \\
        &= \max_{x\in \mathbb{R}^n:\|x\|_2=1}  \epsilon^\top \Phi \Pi_{\mathcal{S}_{\star}}x \\
        &= \max_{x\in \mathbb{R}^n:\|x\|_2=1}  \langle \epsilon, \Phi \Pi_{\mathcal{S}_{\star}}x \rangle \\
        &\leq \max_{x\in \mathbb{R}^n:\|x\|_2=1} \| \epsilon\|_2\cdot\| \Phi \Pi_{\mathcal{S}_{\star}}x \|_2\\
        &\leq \sqrt{\beta_{2k}} \| \epsilon\|_2 ,
    \end{align}
    where the last inequality is done by directly applying the definition of RIP.
    Therefore, 
    \begin{align}
        C \leq 4\mu_t \sqrt{\beta_{2k}} \| \epsilon\|_2 \leq 2\beta_{3k} \sqrt{\beta_{2k}} \| \epsilon\|_2  .
    \end{align}
    Combining the 3 pieces together, we finally derive
    \begin{align}
         \|w_{t+1}-w^\star\|_2 &\leq 2 \max\{ \frac{\beta_{2k}}{\alpha_{3k}} - 1, 1-\frac{\alpha_{2k}}{\beta_{3k}}\} \|z_t-w^\star\|_2 \\
         &+ \frac{\beta_{4k}-\alpha_{4k}}{\alpha_{3k}}\|z_t-w^\star\|_2 + 2\beta_{3k} \sqrt{\beta_{2k}} \| \epsilon\|_2  .
    \end{align}

    Rearranging the inequality completes the proof.

\end{proof}

\subsection{Proof of Theorem~\ref{the:1}}
\begin{theorem}[Restated]
In the worst case scenario, with Assumption~\ref{assum:rip}, the solutions path find by Automated Accelerated IHT (Algorithm~\ref{alg:a-iht}) satisfy the following iterative invariant. 
\begin{align}
        \|w_{t+1}-w^\star\|_2 &\leq \rho|1+\tau_t|\cdot \|w_t - w^\star\|_2 +\rho |\tau_t| \cdot {\|w_{t-1}-w^\star\|_2}  + 2\beta_{3k} \sqrt{\beta_{2k}} \| \epsilon\|_2,
\end{align}
where $\rho = \left( 2 \max\{ \frac{\beta_{2k}}{\alpha_{3k}} - 1, 1-\frac{\alpha_{2k}}{\beta_{3k}}\} + \frac{\beta_{4k}-\alpha_{4k}}{\alpha_{3k}} \right)$, and $\|\epsilon\|_2 = \|y-\Phi w^\star\|_2$ is the optimal objective value.
\end{theorem}
\begin{proof}
	Lemma~\ref{prop:1} suggests
	\begin{align}
	\|z_t-w^\star\|_2\leq |1+\tau_t|\|w_t - w^\star\|_2+|\tau_t|{\|w_{t-1}-w^\star\|_2}.
	\end{align}
	Combining with lemma~\ref{prop:2}, \emph{i.e.},
	\begin{align}
	\|w_{t+1}-w^\star\|_2 &\leq \rho\|z_t-w^\star\|_2 + 2\beta_{3k} \sqrt{\beta_{2k}} \| \epsilon\|_2 ,
	\end{align}
	where $\rho = \left( 2 \max\{ \frac{\beta_{2k}}{\alpha_{3k}} - 1, 1-\frac{\alpha_{2k}}{\beta_{3k}}\} + \frac{\beta_{4k}-\alpha_{4k}}{\alpha_{3k}} \right)$, and $\|\epsilon\|_2 = \|y-\Phi w^\star\|_2$, we have
	\begin{align}
	\|x_t-w^\star\|_2 \leq& \rho|1+\tau_t|\|w_t - w^\star\|_2 + \rho |\tau_t|{\|w_{t-1}-w^\star\|_2} + 2\beta_{3k} \sqrt{\beta_{2k}} \| \epsilon\|_2 ,
	\end{align}
	which completes the proof.
\end{proof}

\subsection{Proof of Corollary~\ref{coro:1}}
\setcounter{corollary}{0}
\begin{corollary}[Restated]\label{coro:1}
	 Given the iterative invariant as stated in Theorem~\ref{the:1}, and assuming the optimal solution achieves $\|\epsilon\|_2=0$, the solution found by Algorithm~\ref{alg:a-iht} satisfies:
	 \begin{align}
	     f(w_{t+1})-f(w^\star)\leq \phi^t\left( \frac{\beta_{2k}}{\alpha_{2k}} f(w_1) + \frac{\rho\tau\beta_{2k}}{\phi\alpha_{k}} f(w_0) \right),
	 \end{align}
	 where $\phi = (\rho(1+\tau) +\sqrt{\rho^2(1+\tau)^2+4\rho \tau})/2$ and $\tau = \max_{i\in [t]} |\tau_i|$. It is sufficient to show linear convergence to the global optimum, when $\phi<1$, or equivalently $\rho<1/(1+2\tau)$.
\end{corollary}
\begin{proof}
    Theorem~\ref{the:1} provides an upper bound invariant among consecutive iterates of the algorithm. To have better sense of convergence rate, we assume the optimal solution achieves $\|\epsilon\|_2=0$.  Theorem~\ref{the:1} then implies
\begin{align}
\|w_{t+1}-w^\star\|_2 &\leq \rho(1+|\tau_t|) \|w_t - w^\star\|_2 + \rho |\tau_t| \cdot {\|w_{t-1}-w^\star\|_2}\\
&\leq \rho(1+\tau) \|w_t - w^\star\|_2 + \rho \tau \cdot {\|w_{t-1}-w^\star\|_2}.
\end{align}

Rearranging the inequality with some regular algebraic manipulations, we have 
\begin{align}
    \|w_{t+1}-w^\star\|_2 +\frac{\rho\tau}{\phi}  \|w_t - w^\star\|_2 &\leq  \phi \left( \|w_{t}-w^\star\|_2 + \frac{\rho\tau}{\phi}  \|w_{t-1} - w^\star\|_2 \right)\\
    &\leq \phi^{t} \left( \|w_{1}-w^\star\|_2 + \frac{\rho\tau}{\phi}  \|w_{0} - w^\star\|_2 \right),
\end{align}
where $\phi=\frac{\sqrt{\rho^2(1+\tau)^2+4\rho\tau}+\rho(1+\tau)}{2}$. 

Noting that all $\rho, \tau, \phi$ are non-negative, we can relax the inequality a bit to be
\begin{align}
    \|w_{t+1}-w^\star\|_2 
    &\leq \phi^{t} \left( \|w_{1}-w^\star\|_2 + \frac{\rho\tau}{\phi}  \|w_{0} - w^\star\|_2 \right). \label{eq:linear-conv}
\end{align}
It is sufficient for linear convergence when $\phi<1$, \emph{i.e.},
\begin{align}
    &\frac{\sqrt{\rho^2(1+\tau)^2+4\rho\tau}+\rho(1+\tau)}{2}<1\\
    \iff \quad& \sqrt{\rho^2(1+\tau)^2+4\rho\tau}<2-\rho(1+\tau)\\
    \iff \quad &
    \begin{cases}
        \rho^2(1+\tau)^2+4\rho\tau &< (2-\rho(1+\tau))^2\\
        0 &< 2-\rho(1+\tau)
    \end{cases}\\
    \iff \quad &
    \begin{cases}
        \rho(1+2\tau) &< 1\\
        \rho(1+\tau) &< 2
    \end{cases}\\
    \iff \quad & \rho < 1/(1+2\tau)
\end{align}

In our case, this also indicates the linear convergence of function values. Noting that $(w_{t+1}-w^\star)$ and $(w_{1}-w^\star)$ are at most $2k$-sparse, and $(w_{0}-w^\star)=-w^\star$ is $k$-sparse, we have the following statements according to RIP property:
\begin{align}
    \|\Phi(w_{t+1}-w^\star)\|^2_2 &\leq \beta_{2k} \|w_{t+1}-w^\star\|^2_2\\
    \|\Phi(w_{1}-w^\star)\|^2_2 &\geq \alpha_{2k} \|w_{1}-w^\star\|^2_2\\
    \|\Phi(w_{0}-w^\star)\|^2_2 &\geq \alpha_{k} \|w_{0}-w^\star\|^2_2
\end{align}
As we assume $\|\epsilon\|_2 = \|y-\Phi w^\star\|_2 = 0$, \emph{i.e.}, $y=\Phi w^\star$ and $f(w^\star)$, we can see that
\begin{align}
    f(w_{t+1}) = \|\Phi w_{t+1}-y\|^2_2 &\leq \beta_{2k} \|w_{t+1}-w^\star\|^2_2\\
    f(w_1)=\|\Phi w_{1}-y\|^2_2 &\geq \alpha_{2k} \|w_{1}-w^\star\|^2_2\\
    f(w_0)=\|\Phi w_{0}-y\|^2_2 &\geq \alpha_{k} \|w_{0}-w^\star\|^2_2
\end{align}
Plugging these into (\ref{eq:linear-conv}) completes the proof.
\end{proof}

\section{Additional Related Work}\label{sec:supp-related}

Thresholding-based optimization algorithms have been attractive alternatives to relaxing the constraint to a convex one or to greedy selection. % for sparsity constrained optimization.
\cite{Bahmani2013} provide a gradient thresholding algorithm that generalizes pursuit approaches for compressed sensing to more general losses.
\cite{Yuan2018HTP} study convergence of gradient thresholding algorithms for general losses.
\cite{Jain2014iht} consider several variants of thresholding-based algorithms for high dimensional sparse estimation.
\cite{Nguyen2014LinearCO,Li2016stochasticIHT} discuss convergence properties of thresholding algorithms for stochastic settings, while in~\citep{jain2016structured} the algorithm is extended to structured sparsity. 
Greedy algorithms~\citep{Shalev2010Greedy} for cardinality constrained problems have similar convergence guarantees and smaller per iteration cost but tend to underperform when compared to thresholding-based algorithms~\citep{khanna2018iht}. 

Acceleration using momentum term~\citep{Beck2009AFI,Ghadimi2015heavyball} allows for faster convergence of first-order methods without increasing the per iteration cost. 
In the context of accelerating sparsity constrained first-order optimization, \cite{khanna2018iht, Blumensath2012AccIHT} use momentum terms in conjunction with thresholding and prove linear convergence of their method. 
We extend their work by also including additional constraints of non-negativity. 
More recently, there have also been works~\citep{Ma2019acceleration} that study acceleration in sampling methods such as MCMC that are relevant to Bayesian coresets.

\section{Additional Results for Synthetic Gaussian Posterior Inference}\label{sec:add-exp-1}

\begin{figure}[t!]
	
	\begin{minipage}[c]{\linewidth}\centering
		%\vspace{-2em}
		\includegraphics[width=0.4\linewidth]{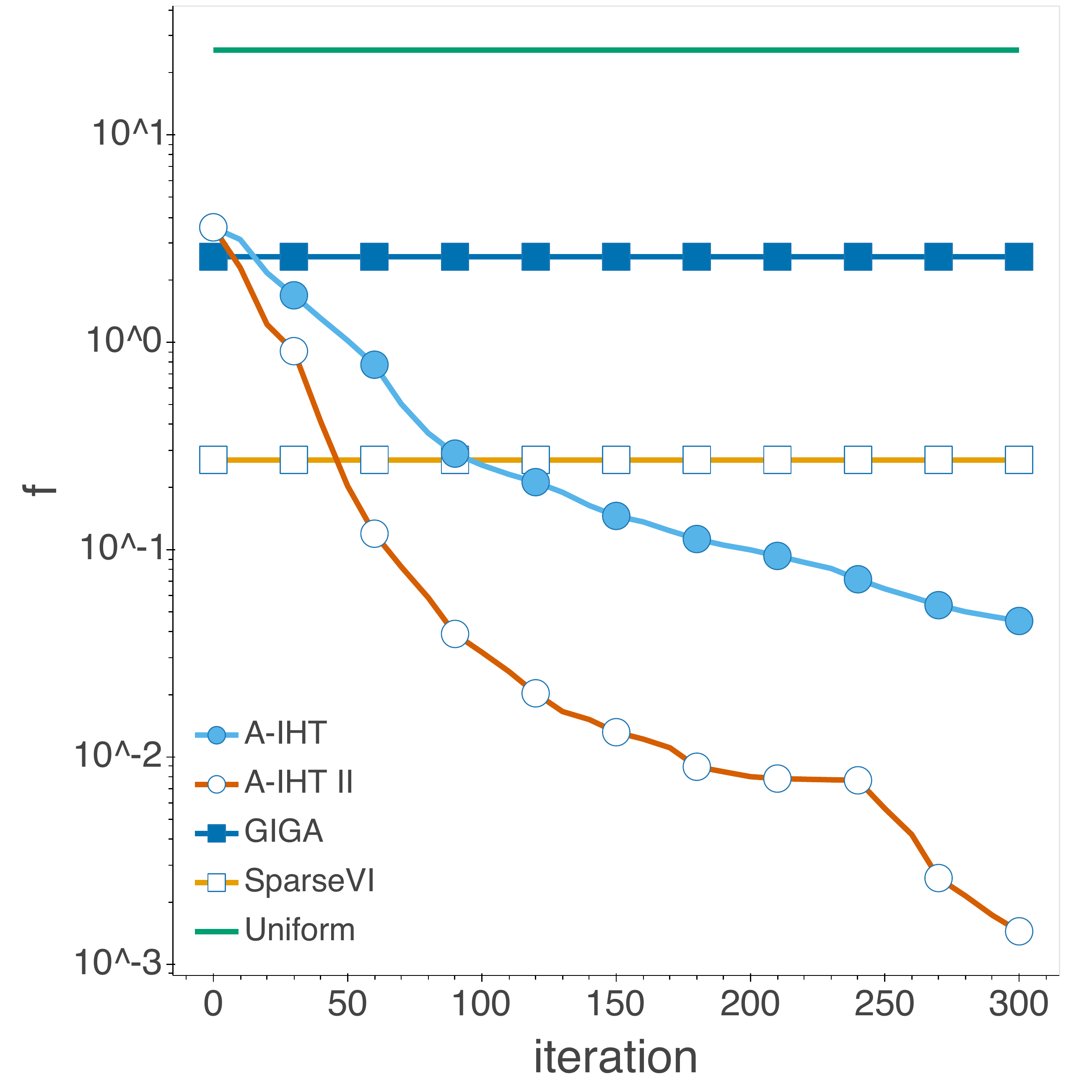}
		%\vspace{-4em}
	\end{minipage}
	\caption{Convergence results for synthetic Gaussian posterior inference (subsection~\ref{sec:exp-1}) when sparsity setting $k=200$ in the first trial. For GIGA, SparseVI and Uniform, each of the objective function values $f$ is calculated by the final output of each algorithms.  } \label{fig:exp-1-conv}
\end{figure}

\begin{figure}[t!]\centering
	\begin{minipage}[c]{\linewidth}\centering
		%\vspace{-2em}
		\includegraphics[width=0.32\linewidth]{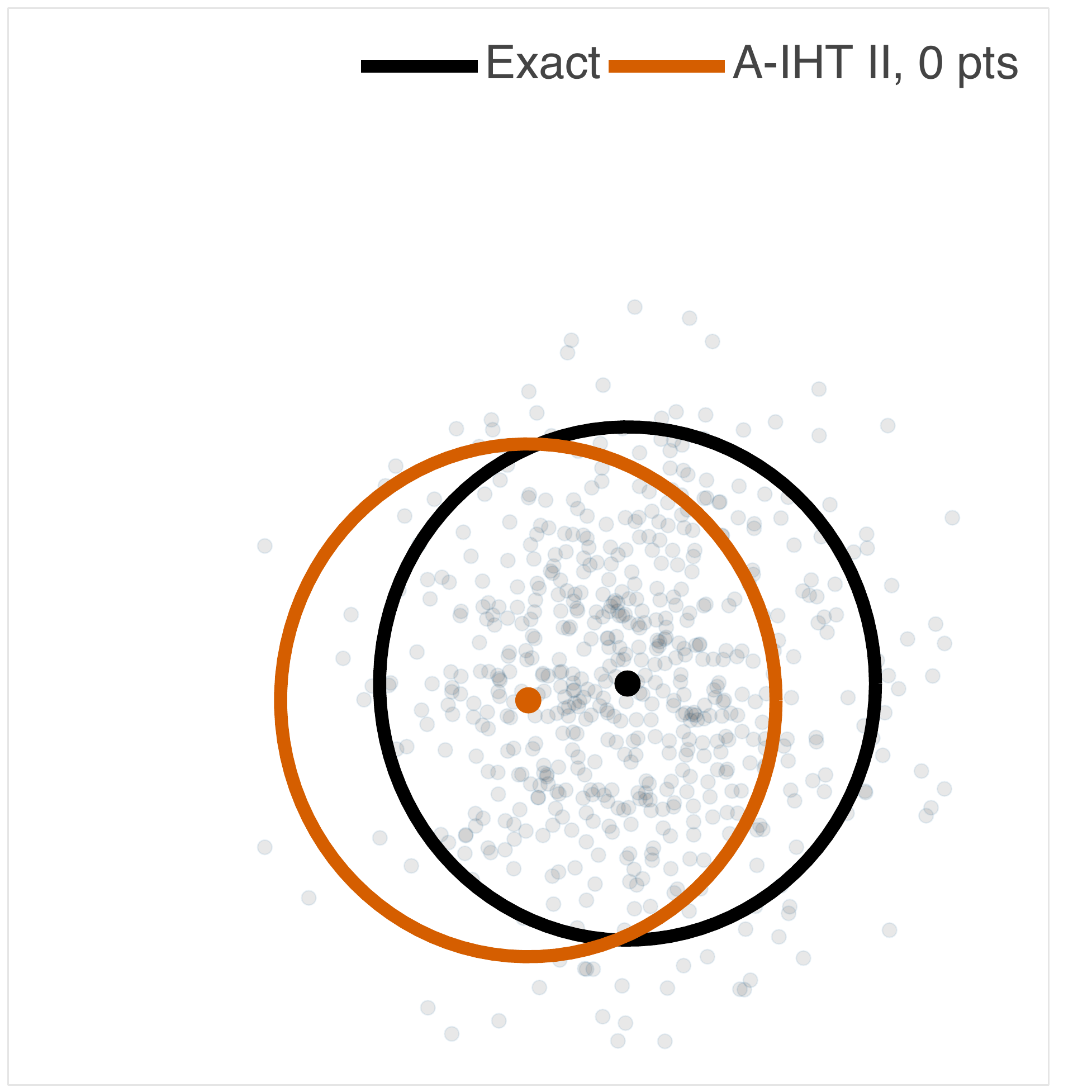}
		\includegraphics[width=0.32\linewidth]{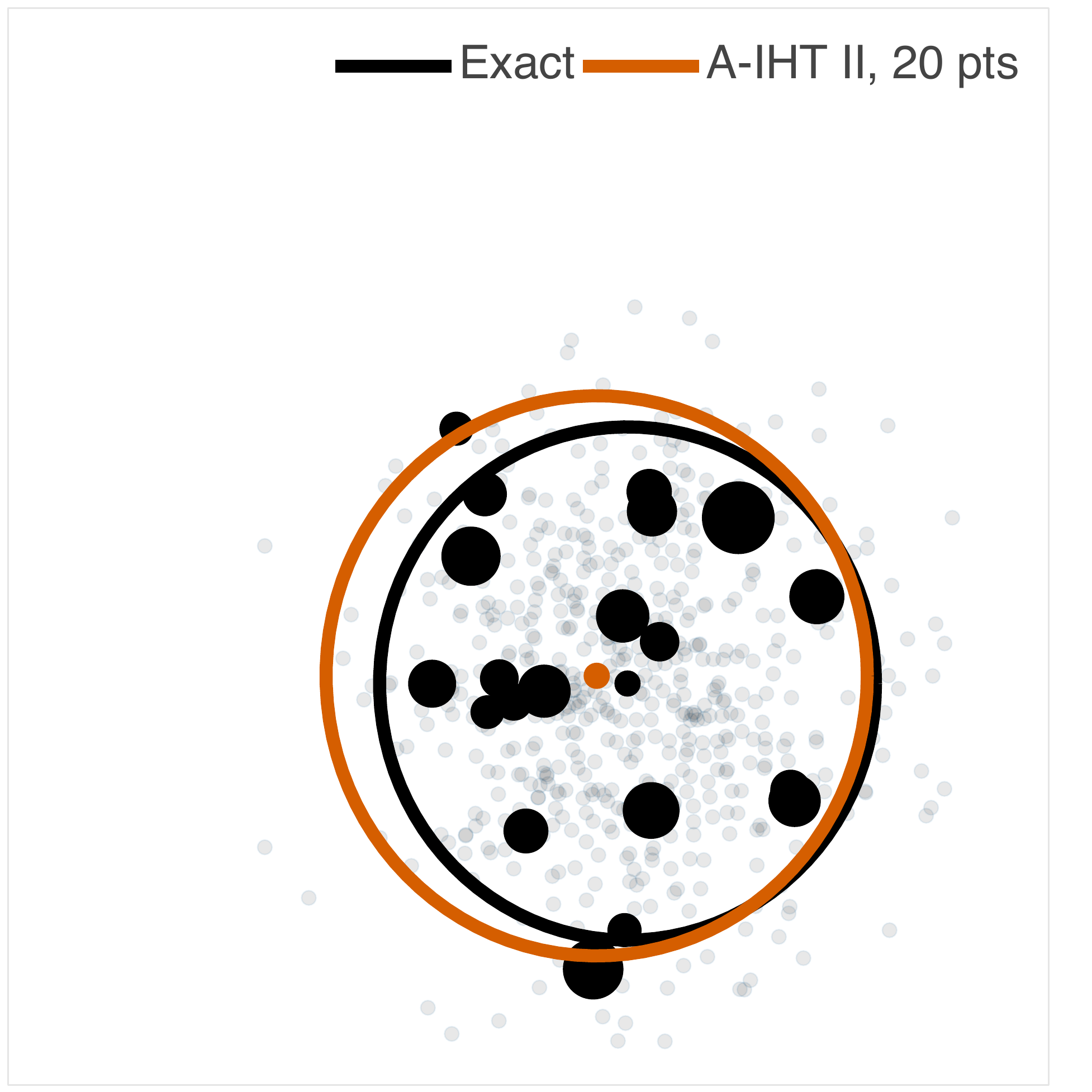}
	\end{minipage}
	\begin{minipage}[c]{\linewidth}\centering
		%\vspace{-2em}
		\includegraphics[width=0.32\linewidth]{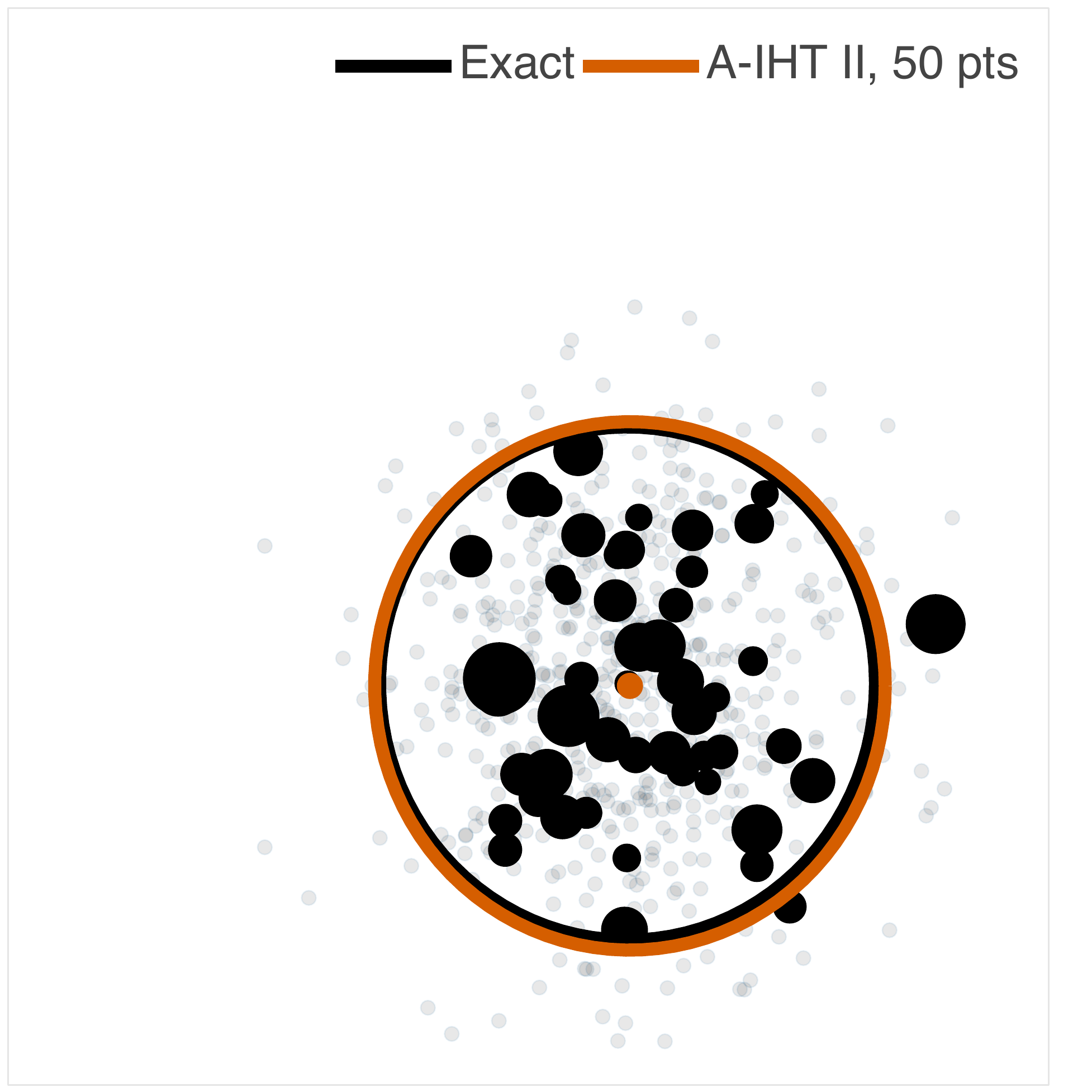}
		\includegraphics[width=0.32\linewidth]{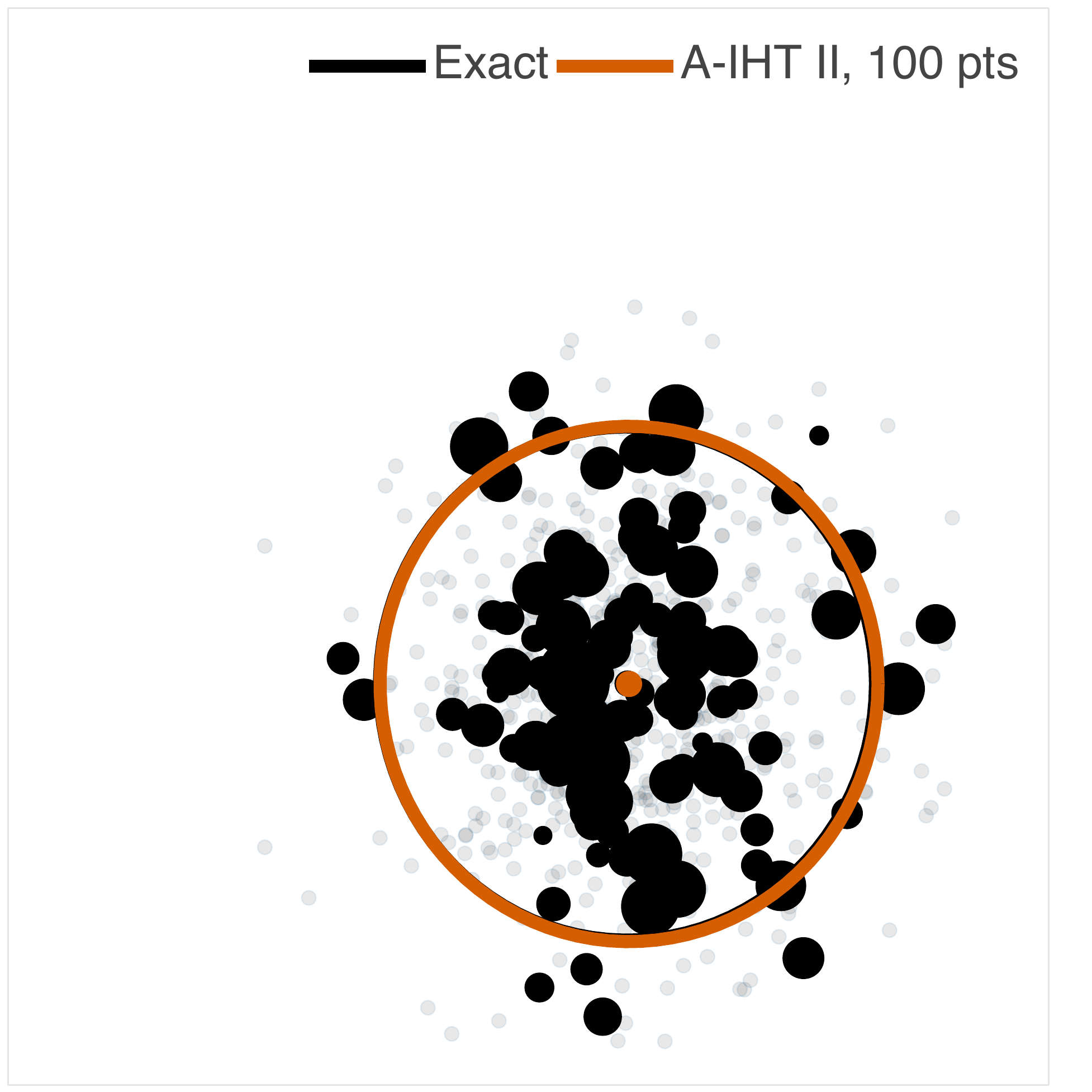}
	\end{minipage}
	\caption{Illustration of true posterior and posterior constructed by A-IHT~\RomanNumeralCaps{2} after projecting to 2-dimensional plane for synthetic Gaussian posterior inference (Section~\ref{sec:exp-1}). Results at different sparsity level are shown.  The ellipses indicate $2\sigma$-prediction of the posterior distribution, and the black dots represent coreset points selected with their radius denoting the respective weights.} \label{fig:exp-1-coreset}
\end{figure}

Additional results for experiments in section~\ref{sec:exp-1} are provided in this section. 

From an optimization perspective, one may be curious about the convergence speed of the two proposed algorithms, \emph{i.e.}, A-IHT and Accelerated A-IHT~\RomanNumeralCaps{2} (Algorithm~\ref{alg:a-iht}~\&~\ref{alg:a-iht-2}). The convergence for the two algorithms compared to the solutions by baselines are presented in Figure~\ref{fig:exp-1-conv}. The x-axis is iteration number for A-IHT and A-IHT~\RomanNumeralCaps{2}, and the y-axis is the objective function to be minimized, \emph{i.e.},
\begin{align}
f(w) = \|y-\Phi w\|_2^2, 
\end{align}
where $y=\sum_{i=1}^n \hat{g}_{i}$ and $\Phi = [\hat{g}_{1},\dots, \hat{g}_{n}]$.

The two IHT algorithms' fast convergence speed reflects what our theory suggests. They surpass GIGA within about 30 iterations, and surpass SparseVI within 50 iterations (A-IHT~\RomanNumeralCaps{2}) and within 100 iterations (A-IHT), respectively. Although we should note that the objective function which SparseVI minimizes is reverse KL divergence instead of $l_2$ distance, the two IHT algorithms can achieve much better solutions when considering KL divergence as well, as shown in Figure~\ref{fig:exp-1}. Moreover, the tendency of a further decrease in objective value is still observed for the two IHT algorithms at $300^{th}$ iteration.

Illustration of the coresets constructed by A-IHT~\RomanNumeralCaps{2} in the first trial after projecting to 2D is presented in Figure~\ref{fig:exp-1-coreset}.

\section{Additional Results for Radial Basis Regression}\label{sec:rad-add}
\begin{figure}[t!]\centering
	
	\begin{minipage}[c]{0.9\linewidth}\centering
		%\vspace{-2em}
		\includegraphics[width=0.3\linewidth]{figures/exp2-contour_IHT-2_m300_id_4.pdf}
		\includegraphics[width=0.3\linewidth]{figures/exp2-contour_true_id_4.pdf}
		\includegraphics[width=0.3\linewidth]{figures/exp2-contour_SVI_m300_id_4.pdf}
		\vspace{1em}
	\end{minipage}
	\begin{minipage}[c]{0.9\linewidth}\centering
		%\vspace{-2em}
		\includegraphics[width=0.3\linewidth]{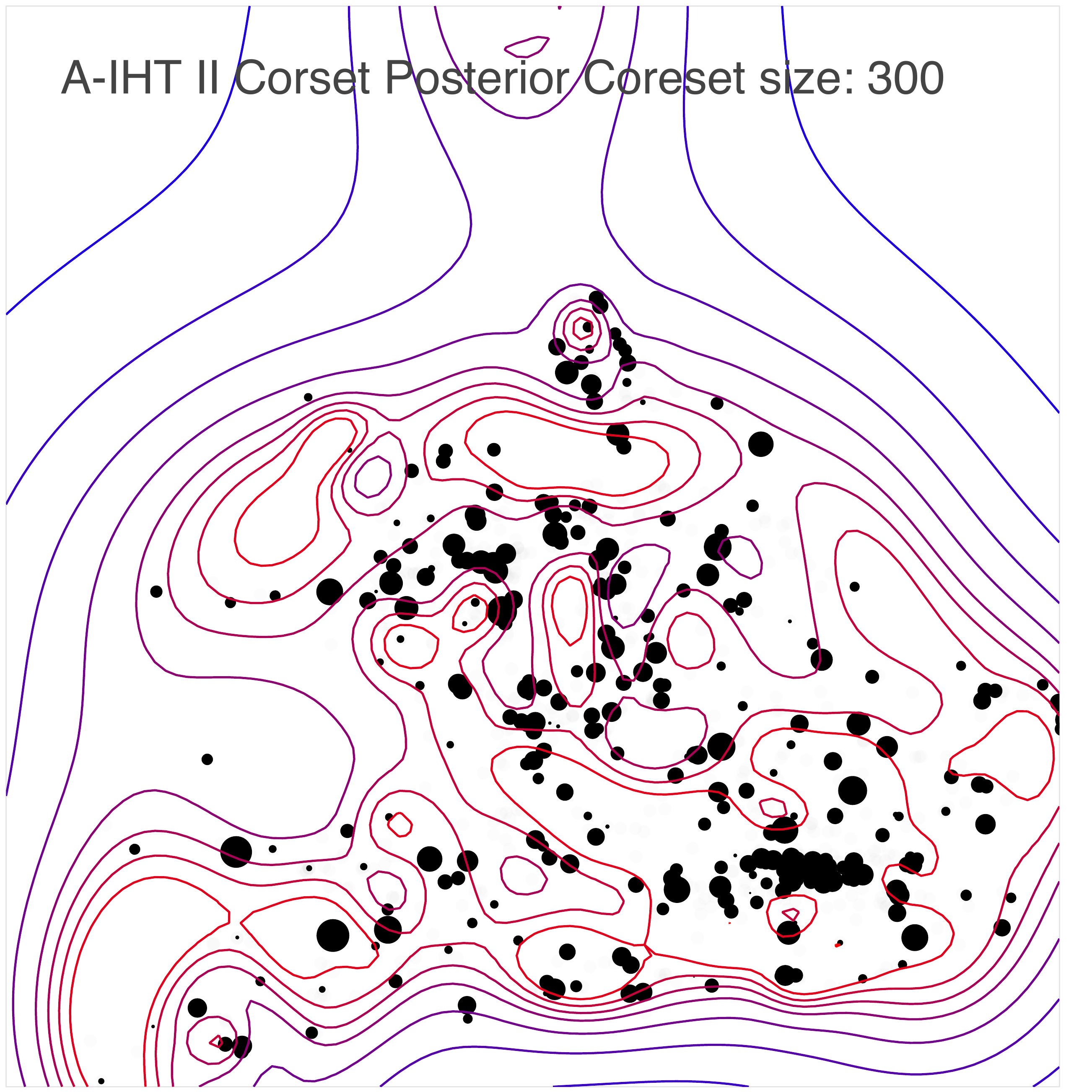}
		\includegraphics[width=0.3\linewidth]{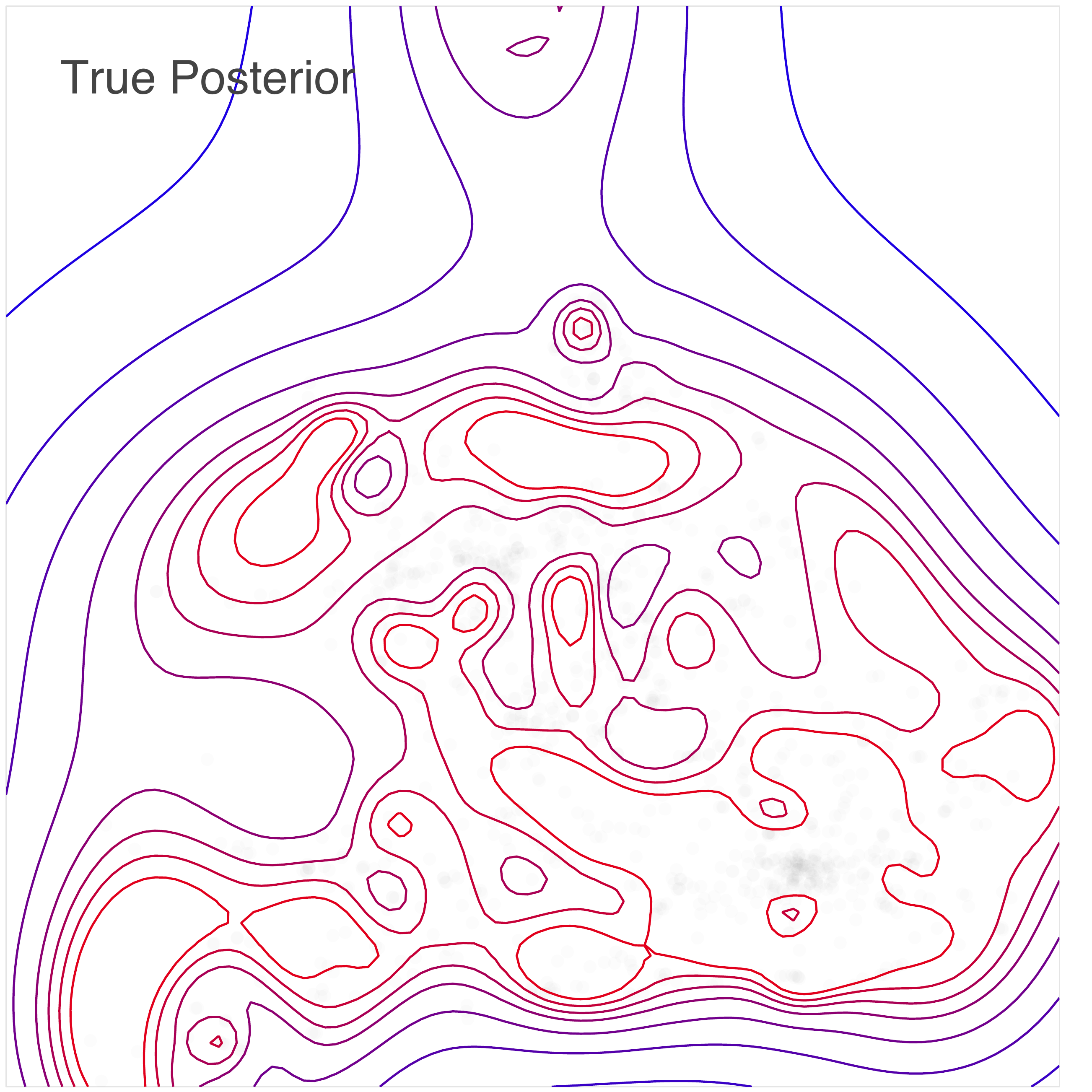}
		\includegraphics[width=0.3\linewidth]{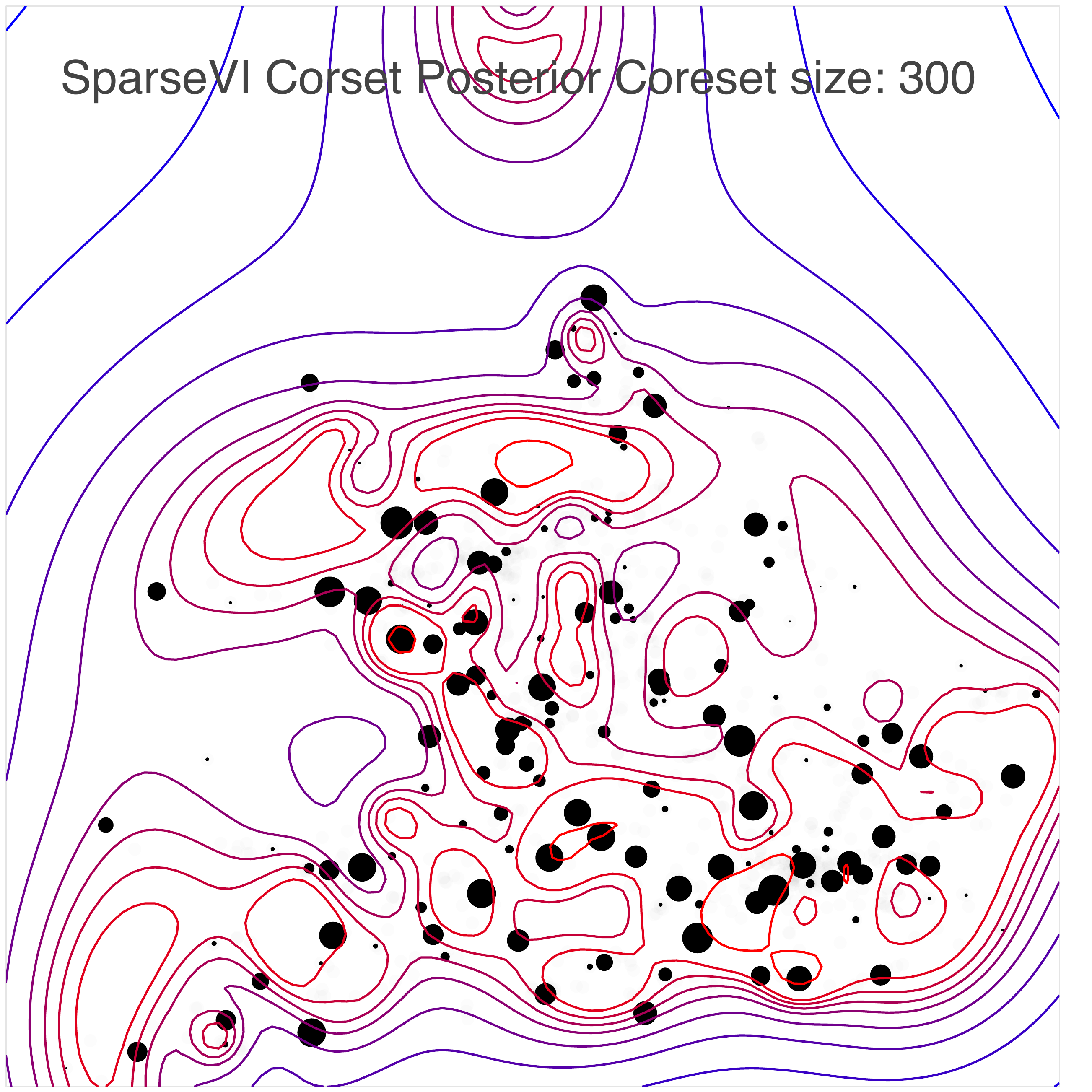}
		\vspace{1em}
	\end{minipage}
	\begin{minipage}[c]{0.9\linewidth}\centering
		%\vspace{-2em}
		\includegraphics[width=0.3\linewidth ]{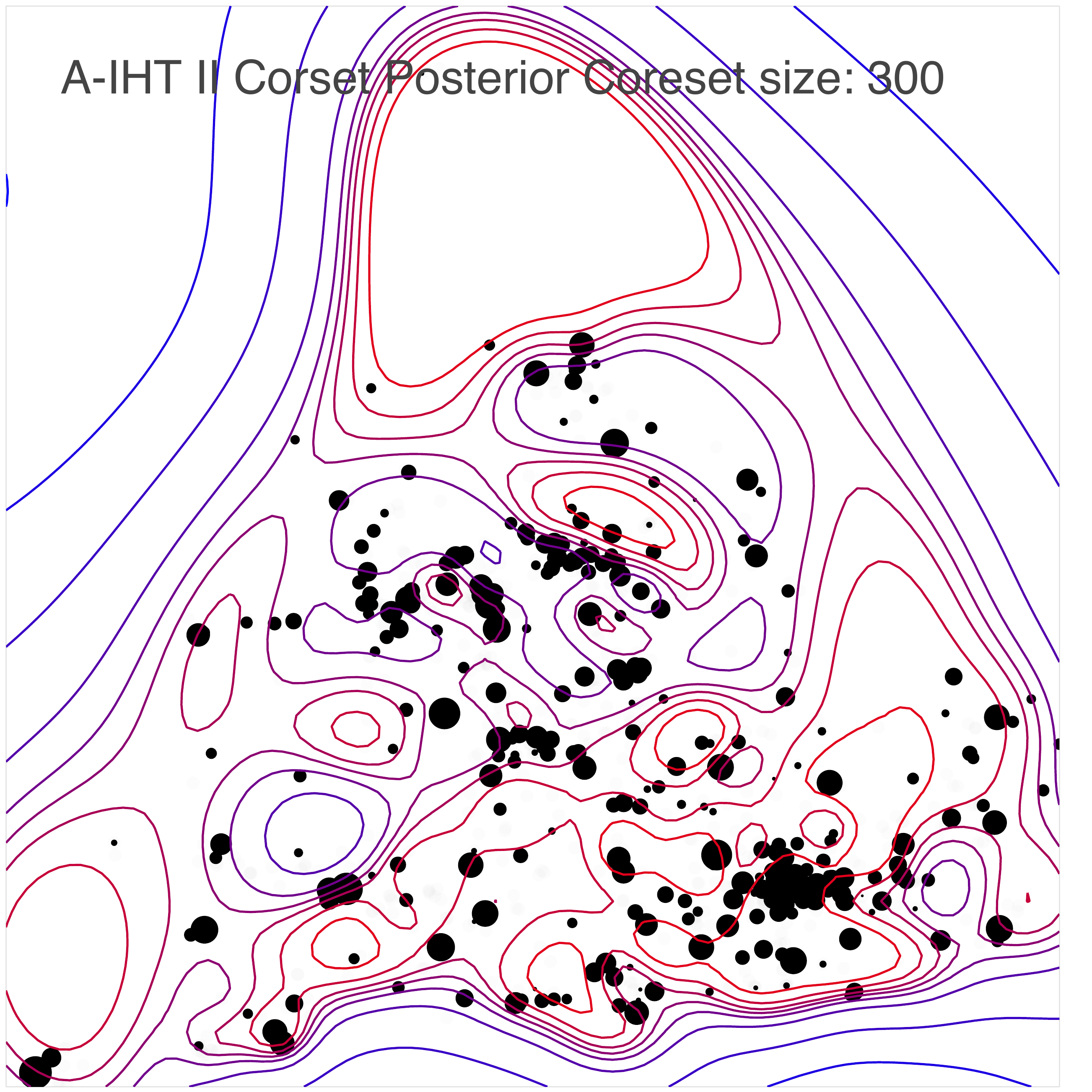}
		\includegraphics[width=0.3\linewidth ]{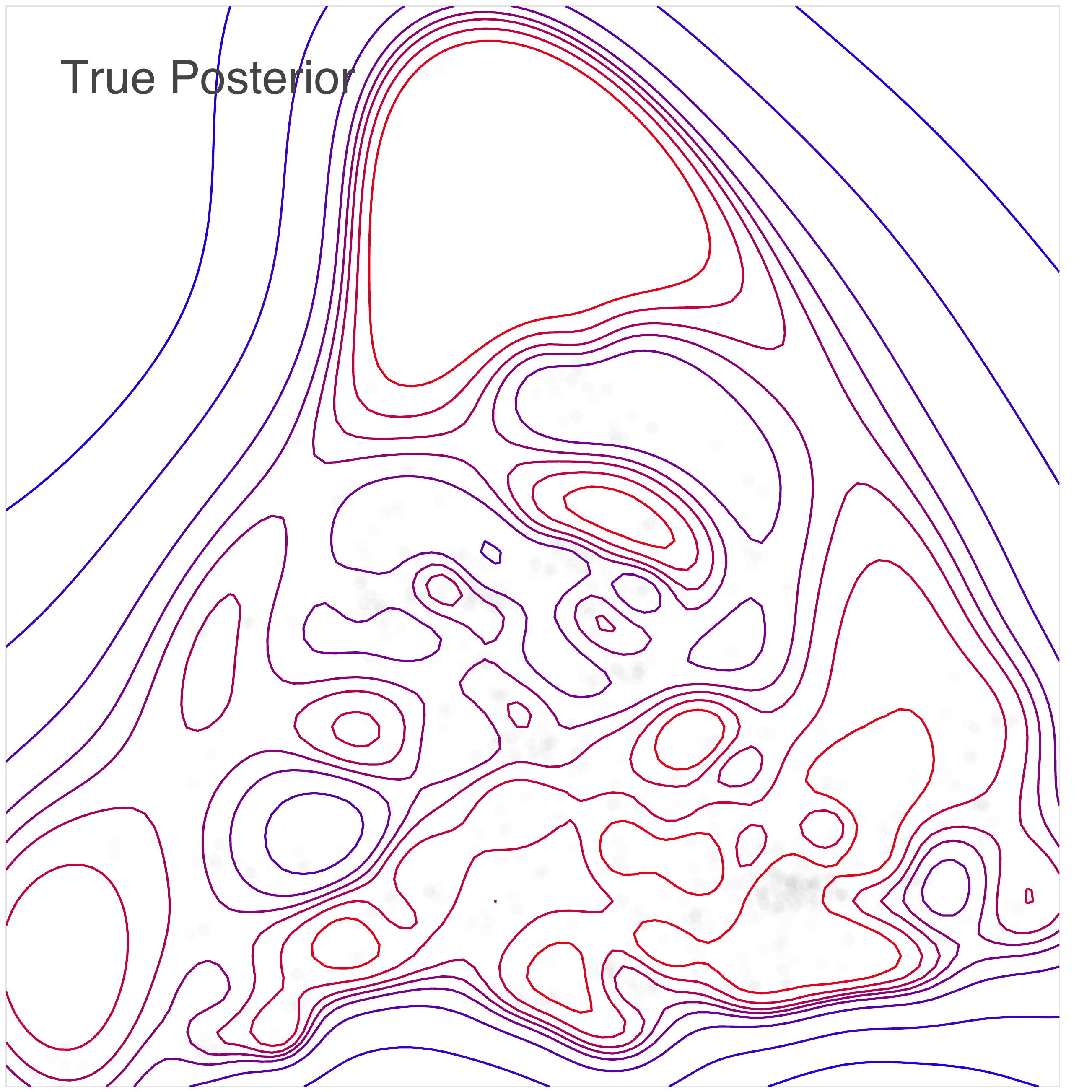}
		\includegraphics[width=0.3\linewidth ]{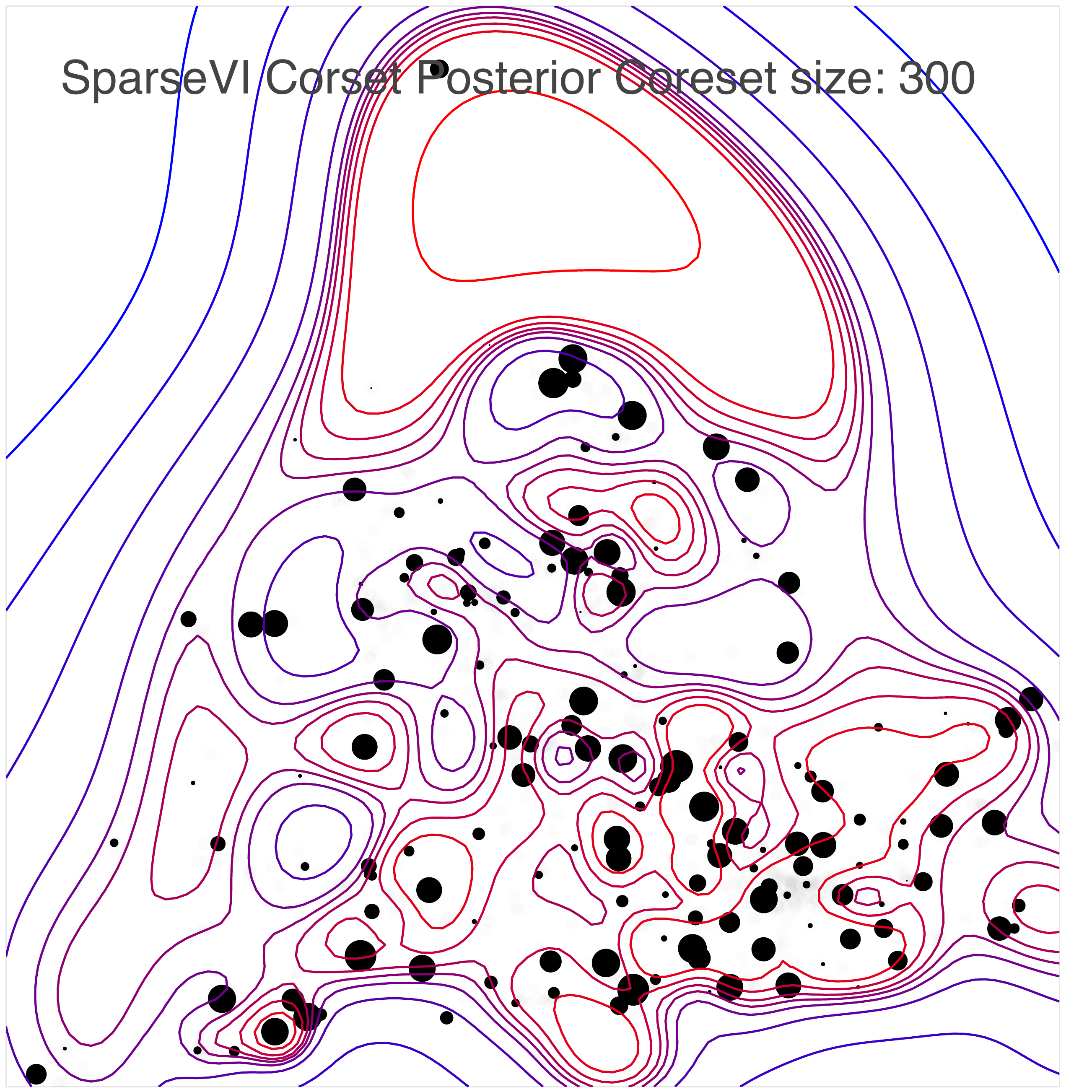}
		\vspace{1em}
	\end{minipage}
	\begin{minipage}[c]{0.9\linewidth}\centering
		%\vspace{-2em}
		\includegraphics[width=0.3\linewidth]{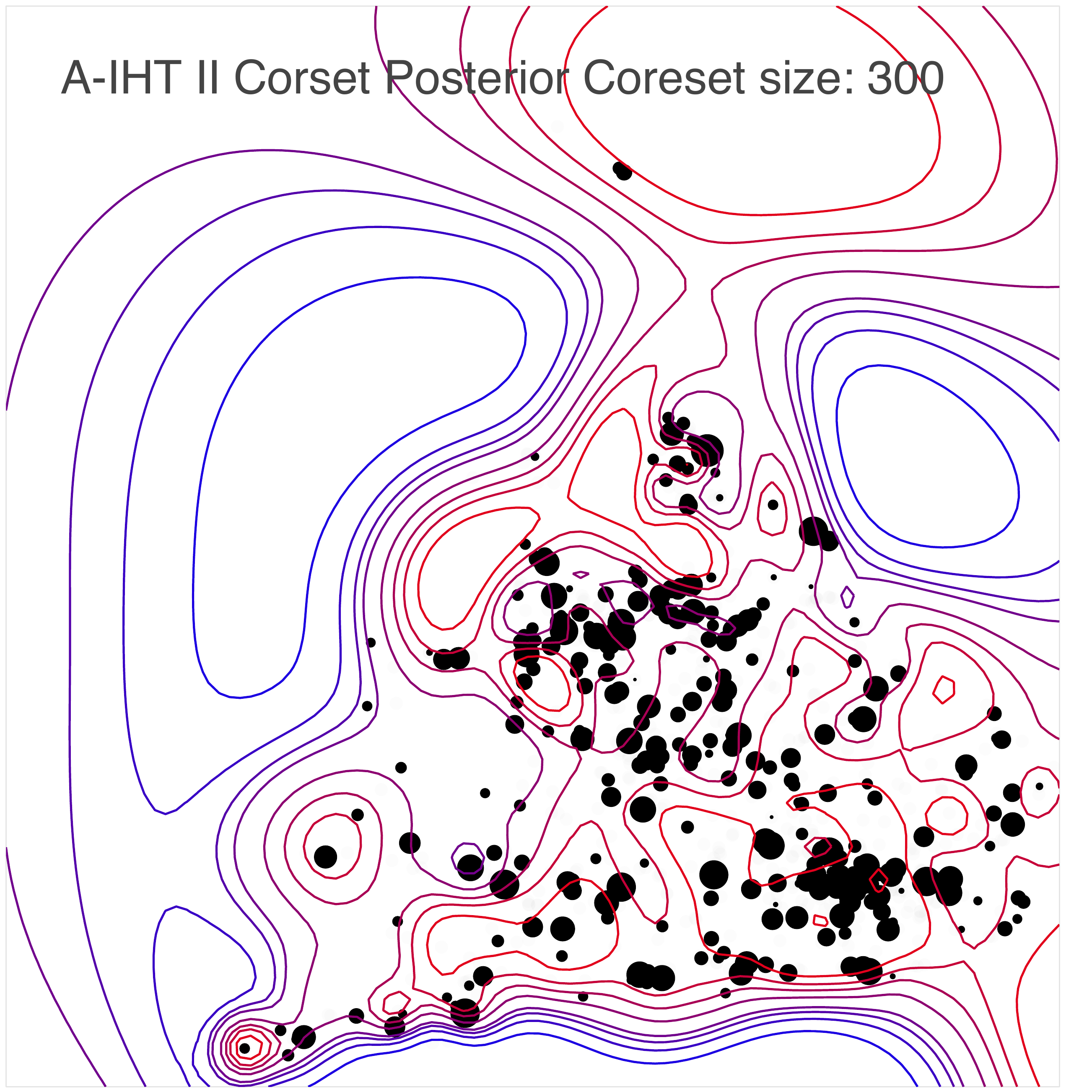}
		\includegraphics[width=0.3\linewidth]{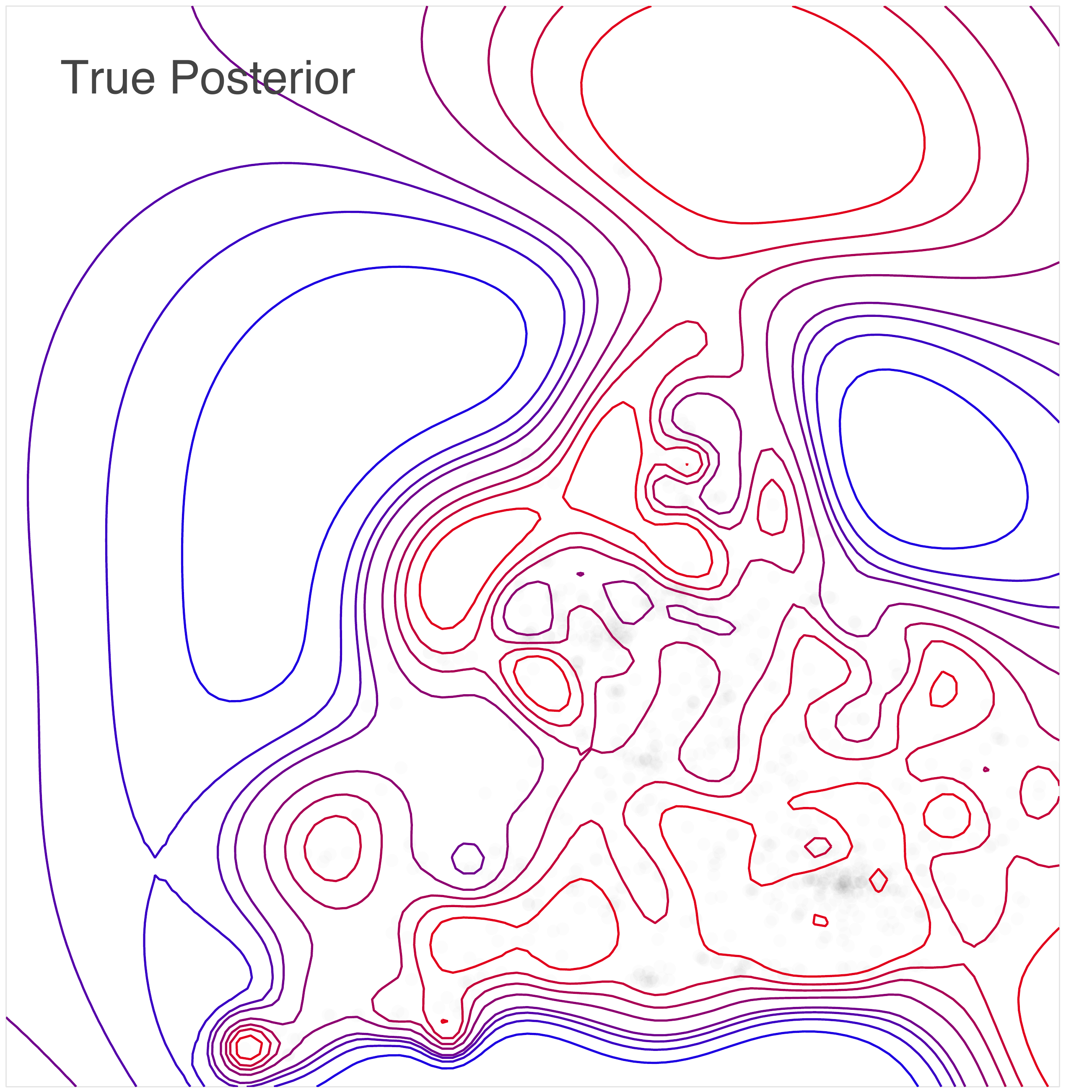}
		\includegraphics[width=0.3\linewidth]{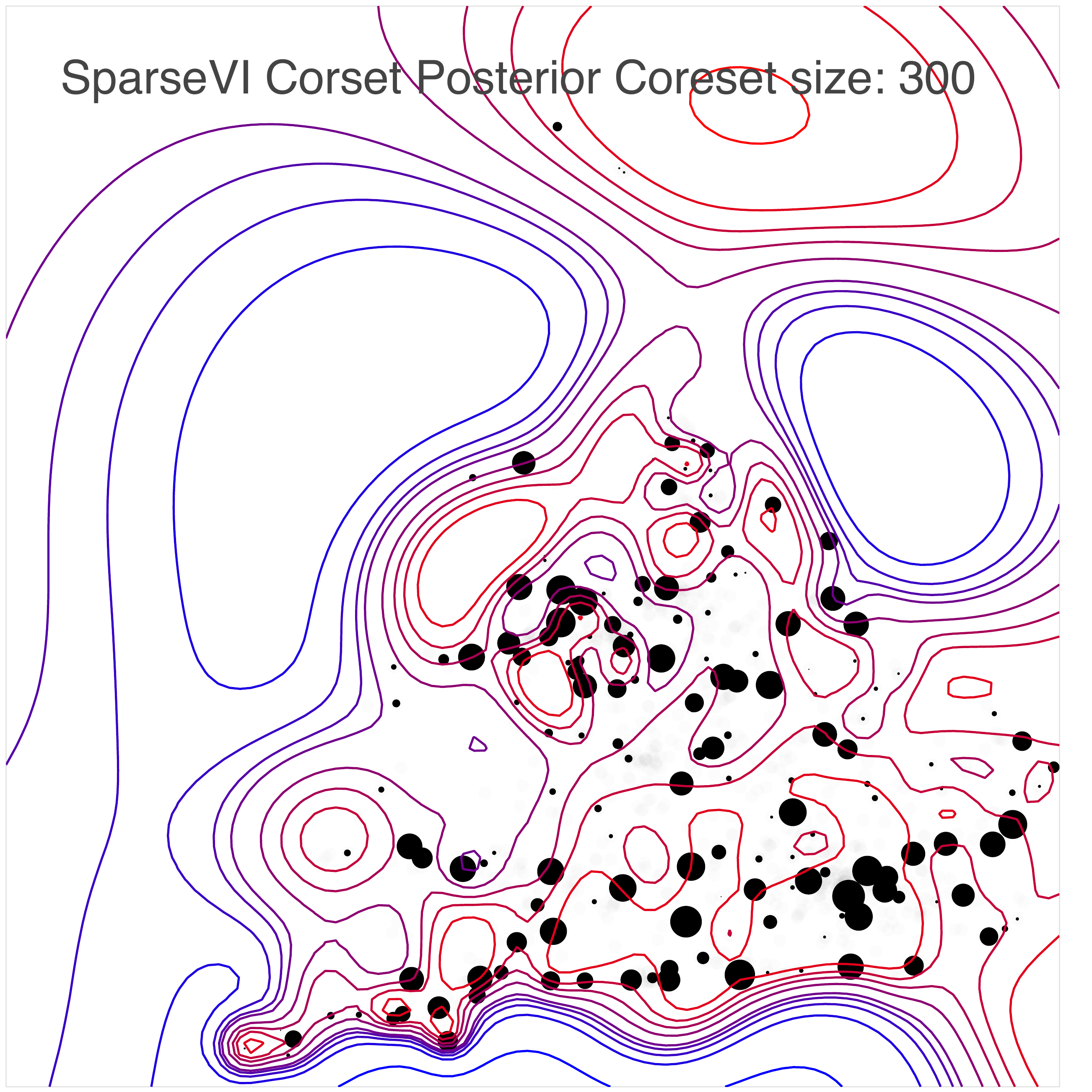}
		\vspace{1em}
	\end{minipage}
	
% 	\textit{True posterior (middle) compared to posterior constructed by A-IHT~\RomanNumeralCaps{2} (left) and posterior constructed by SparseVI (right) at the second trial. }
	\caption{Experiments on Bayesian radial basis function regression in the first four random trials out of ten trails, where coreset sparsity setting $k=300$. Coreset points are presented as black dots, with their radius indicating assigned weights.  Posterior constructed by Accelerated IHT~\RomanNumeralCaps{2} (left) shows almost exact contours as the true posterior distribution (middle), while posterior constructed by SparseVI (right) shows deviated contours from the true posterior distribution.} \label{fig:exp-3-coresets-add}
\end{figure}

In this section, we provide additional experimental results of posterior contours for the radial basis regression experiment (section~\ref{sec:exp-radial}).

We plot the posterior contours for both the true posterior and coreset posterior when sparsity level $k=300$ in the first four random trials out of ten trials.  The coreset posterior constructed by our Algorithm~\ref{alg:a-iht-2} recovers the true posterior almost exactly, unlike SparseVI. Results are shown in Figure~\ref{fig:exp-3-coresets-add}.

\section{Details and Extensive Results of the Bayesian Logistic and Poisson Regression Experiments}\label{sec:exp-lp}

\begin{figure}[t!]\centering
	%\vspace{-2em}
	\begin{minipage}[c]{0.9\linewidth}\centering
		%\vspace{-2em}
		\includegraphics[width=0.32\linewidth]{figures/synth_lr_FKL.pdf}
		\includegraphics[width=0.32\linewidth]{figures/synth_lr_RKL.pdf}
		\includegraphics[width=0.32\linewidth]{figures/synth_lr_log_time.pdf}
		\vspace{-0.5em}
		{\small (a) \texttt{synthetic dataset for logistic regression}}
	\end{minipage}
		\begin{minipage}[c]{0.9\linewidth}\centering
		\vspace{1em}
		\includegraphics[width=0.32\linewidth]{figures/phishing_FKL.pdf}
		\includegraphics[width=0.32\linewidth]{figures/phishing_RKL.pdf}
		\includegraphics[width=0.32\linewidth]{figures/phishing_log_time.pdf}
		%\vspace{-4em}
		{\small (b) \texttt{phishing dataset} for logistic regression}
	\end{minipage}
	\begin{minipage}[c]{0.9\linewidth}\centering
		\vspace{1em}
		\includegraphics[width=0.32\linewidth]{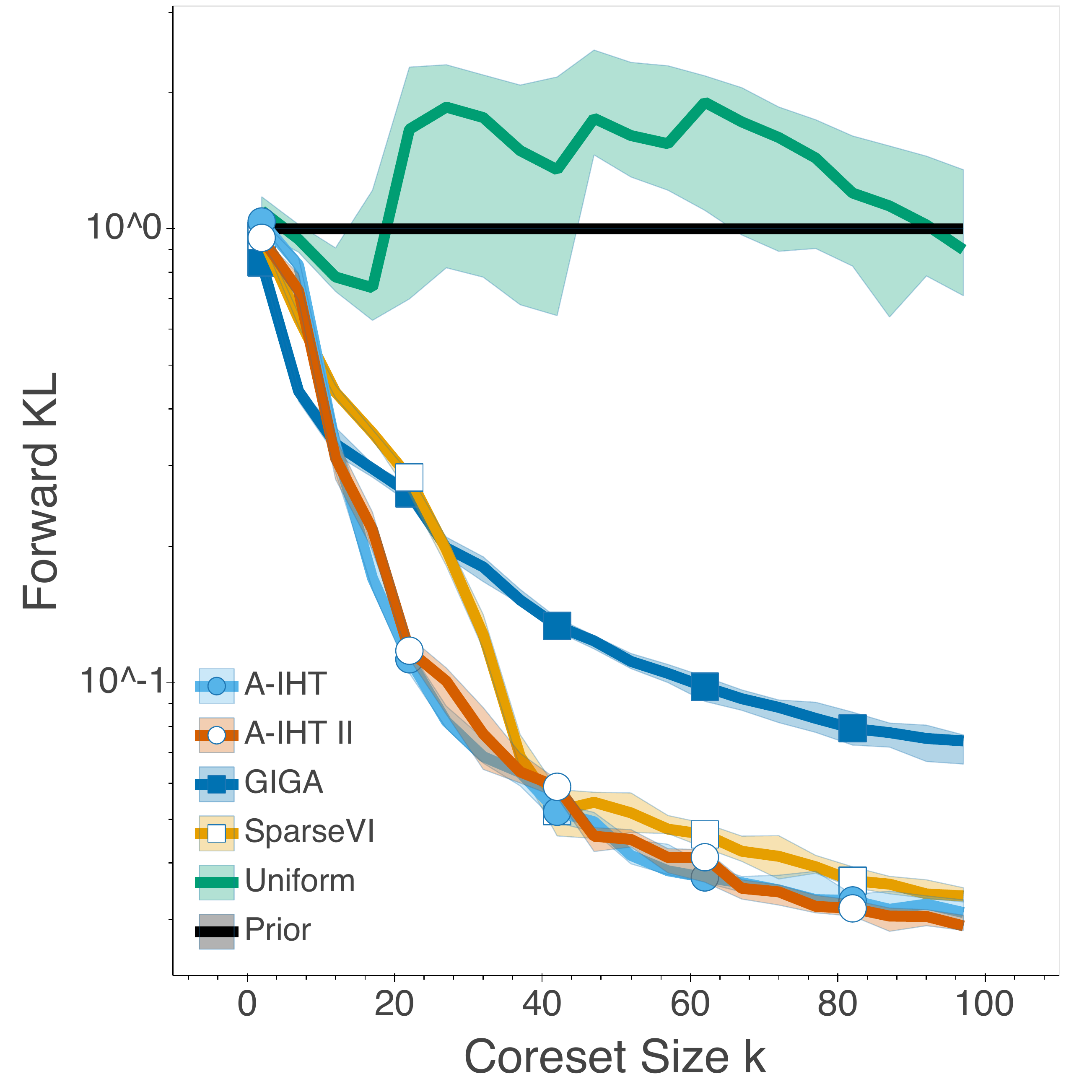}
		\includegraphics[width=0.32\linewidth]{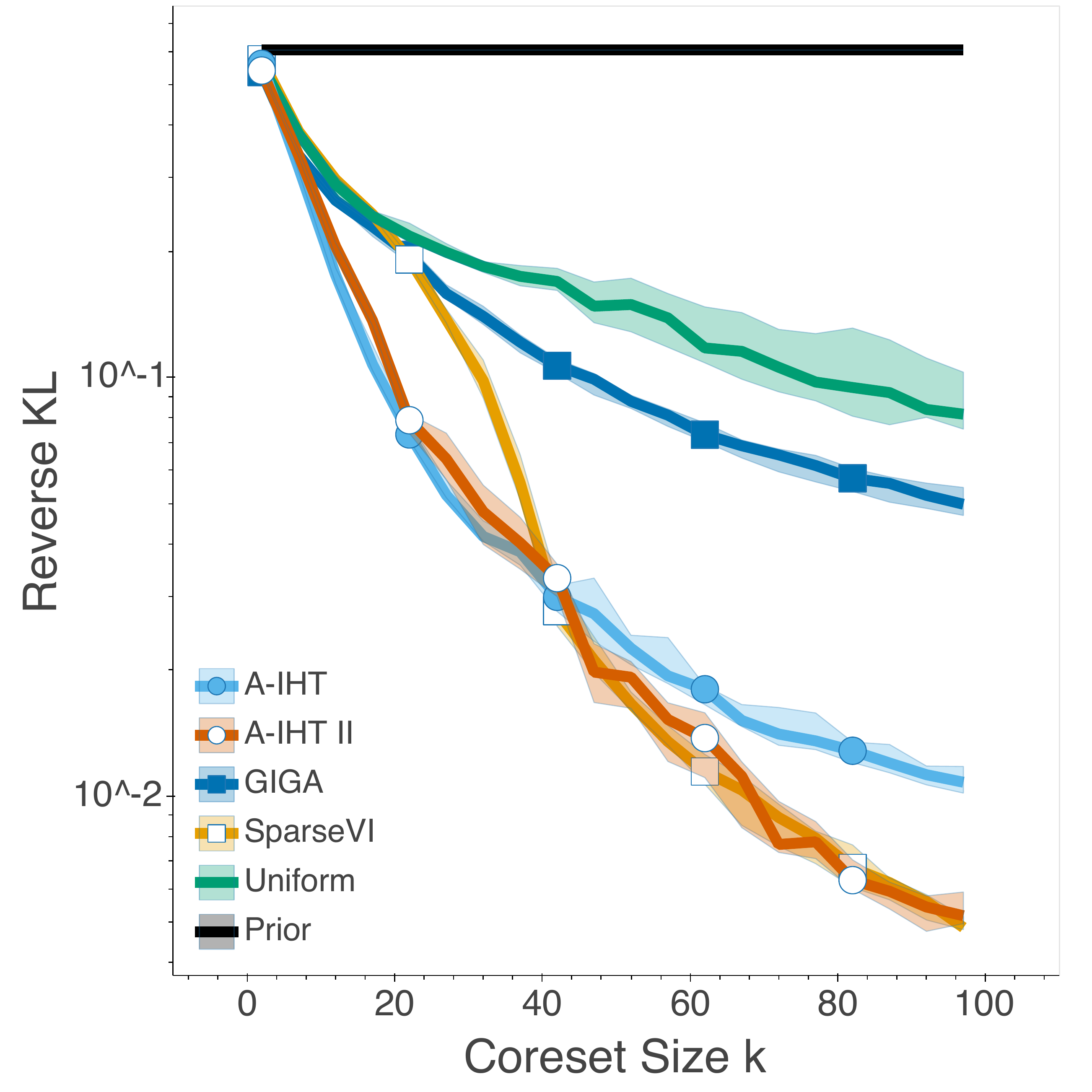}
		\includegraphics[width=0.32\linewidth]{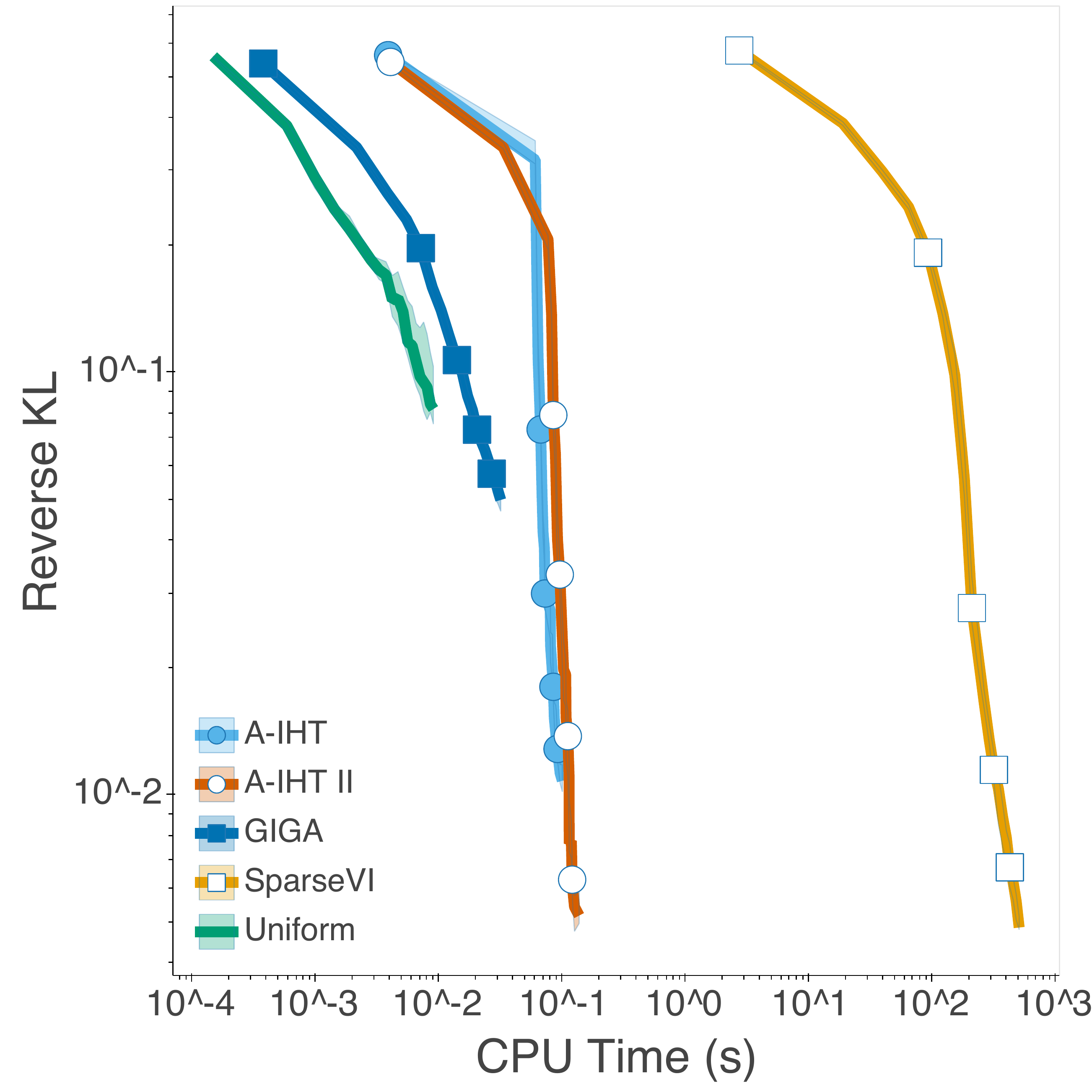}
		\vspace{-0.5em}
		{\small (c) \texttt{chemical reactivities dataset} for logistic regression}
	\end{minipage}

	\caption{Bayesian coreset construction for logistic regression (LR) using the three different datasets. All the algorithms are run 20 times, and the median as well as the interval of $35^{th}$ and $65^{th}$ percentile, indicated as the shaded area, are reported. Different maximal coreset size $k$ is tested from $1$ to $100$. Forward KL (left) and reverse KL (middle) divergence between estimated true posterior and coreset posterior indicate the quality of the constructed coreset. The smaller the KL divergence, the better the coreset is. The running time for each algorithm is also recorded (right).} \label{fig:exp-3-lr}
\end{figure}

\begin{figure}[t!]\centering
	%\vspace{-2em}
	\begin{minipage}[c]{0.9\linewidth}\centering
		%\vspace{-2em}
		\includegraphics[width=0.32\linewidth]{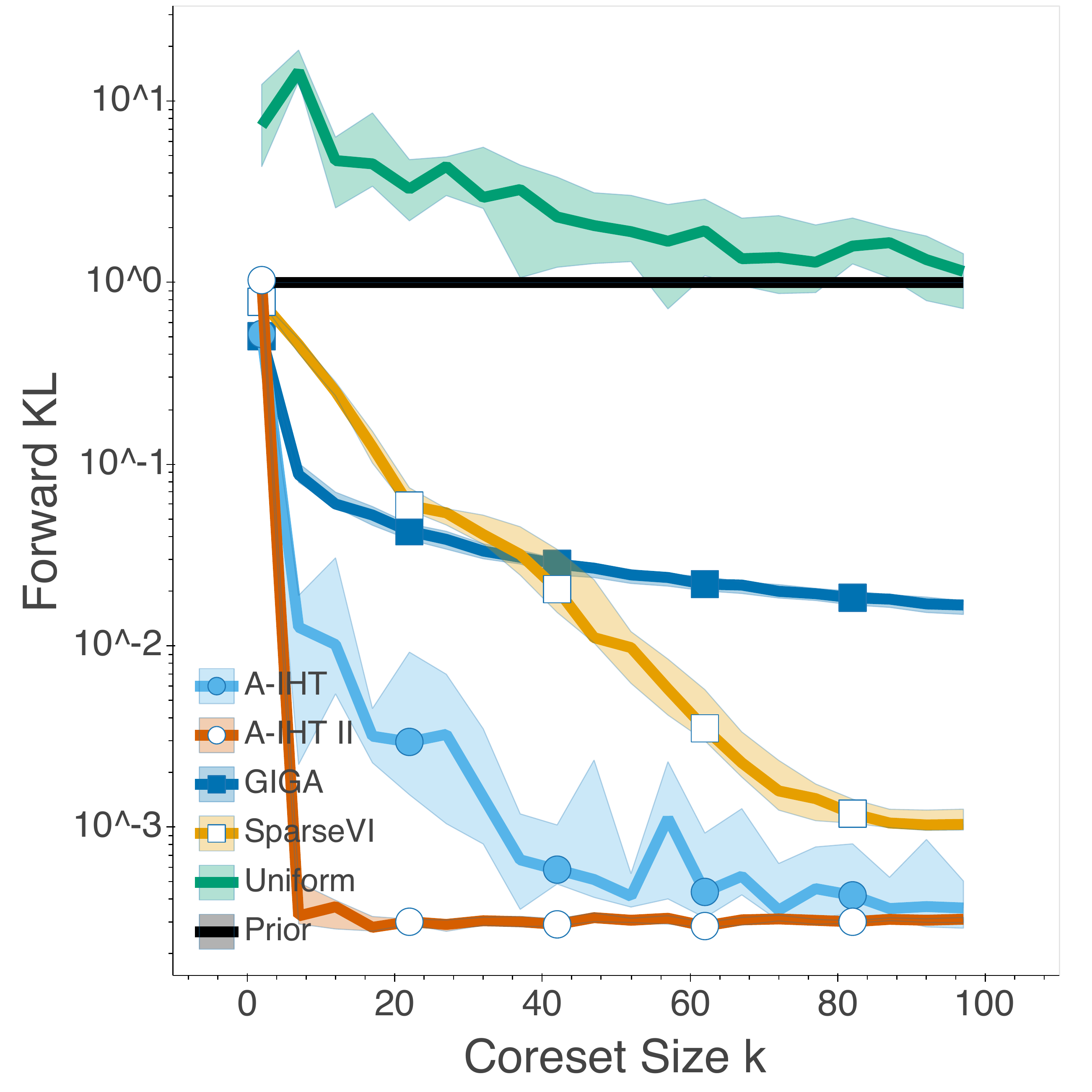}
		\includegraphics[width=0.32\linewidth]{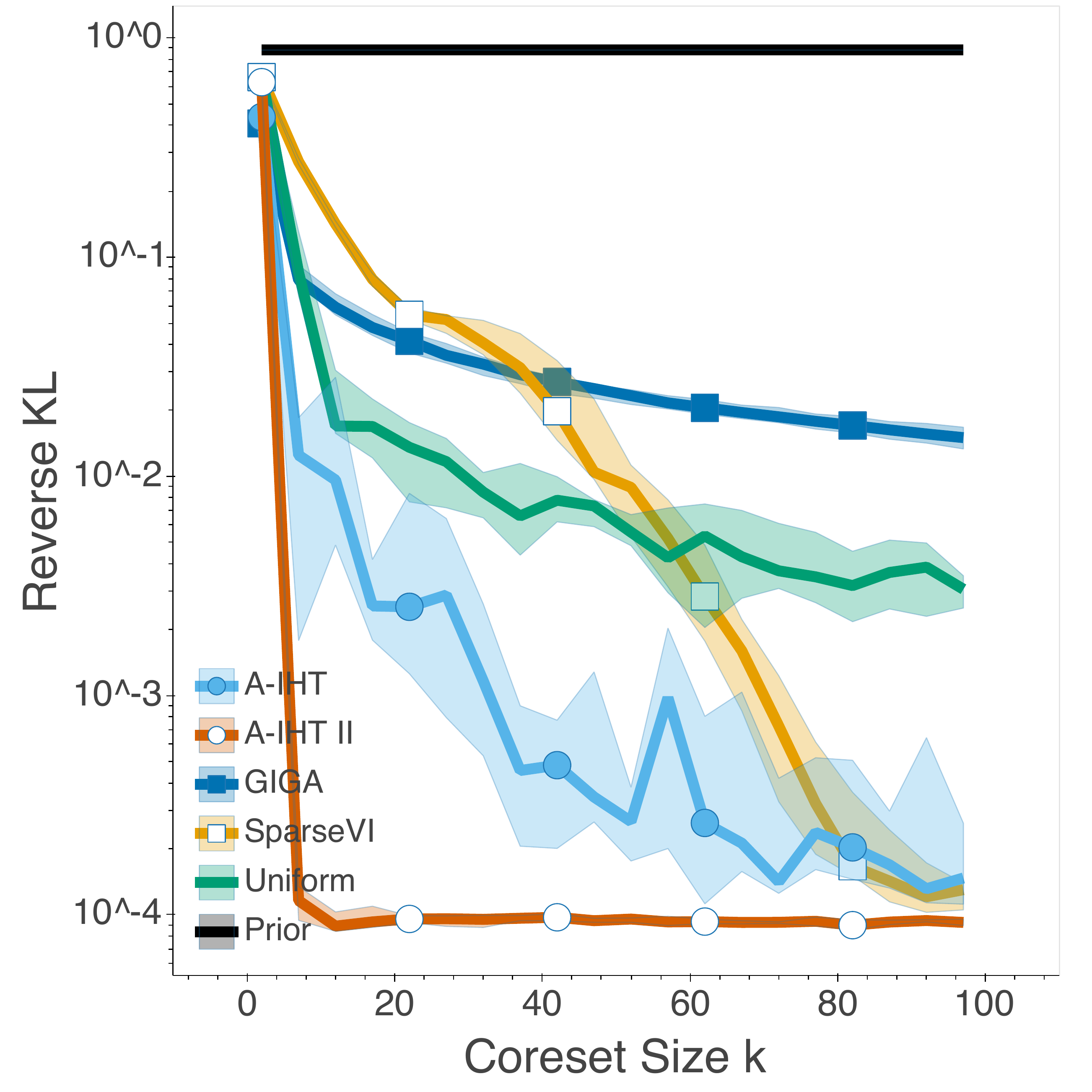}
		\includegraphics[width=0.32\linewidth]{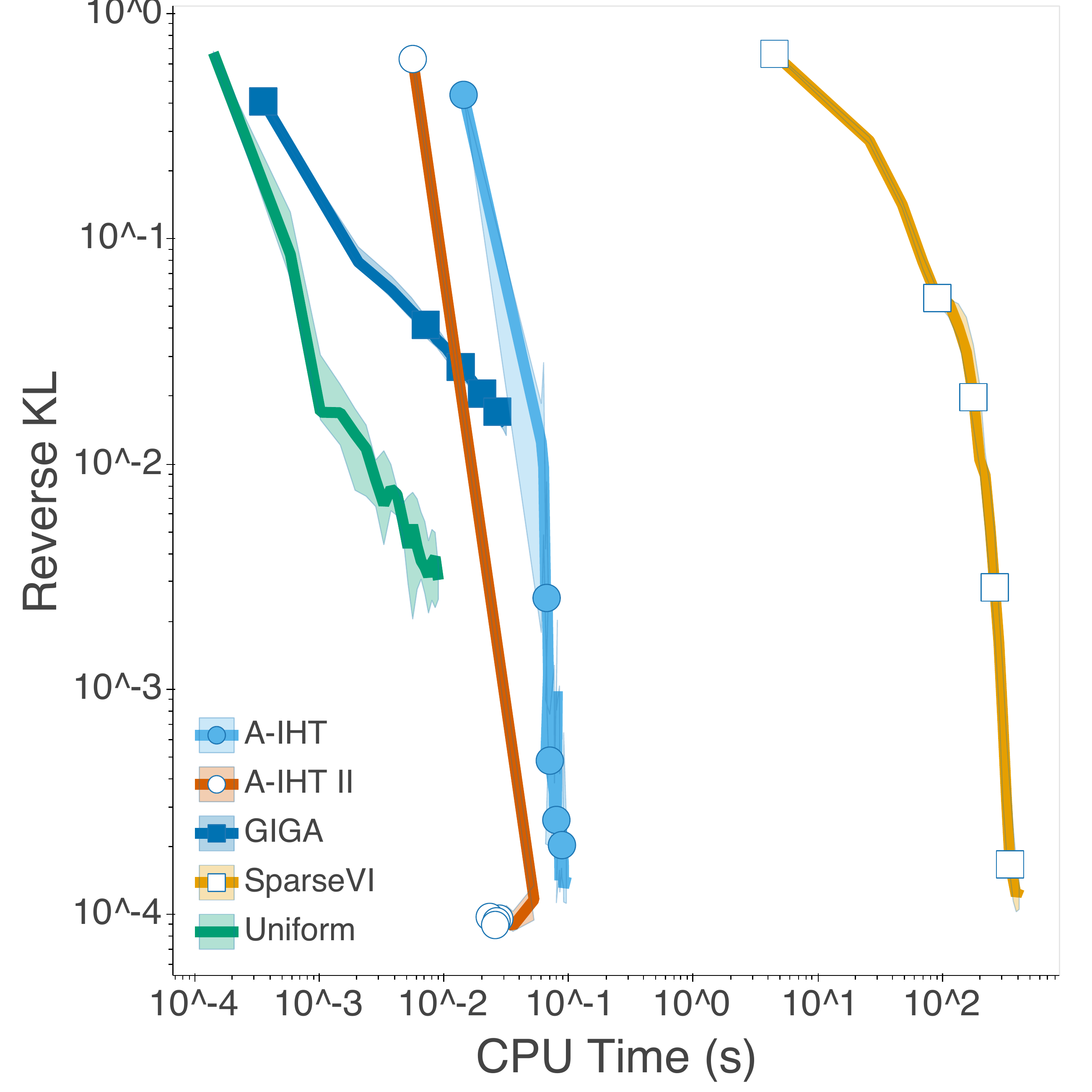}
		\vspace{-0.5em}
		{\small (a) \texttt{synthetic dataset for Poisson regression}}
	\end{minipage}
	\begin{minipage}[c]{0.9\linewidth}\centering
		\vspace{1em}
		\includegraphics[width=0.32\linewidth]{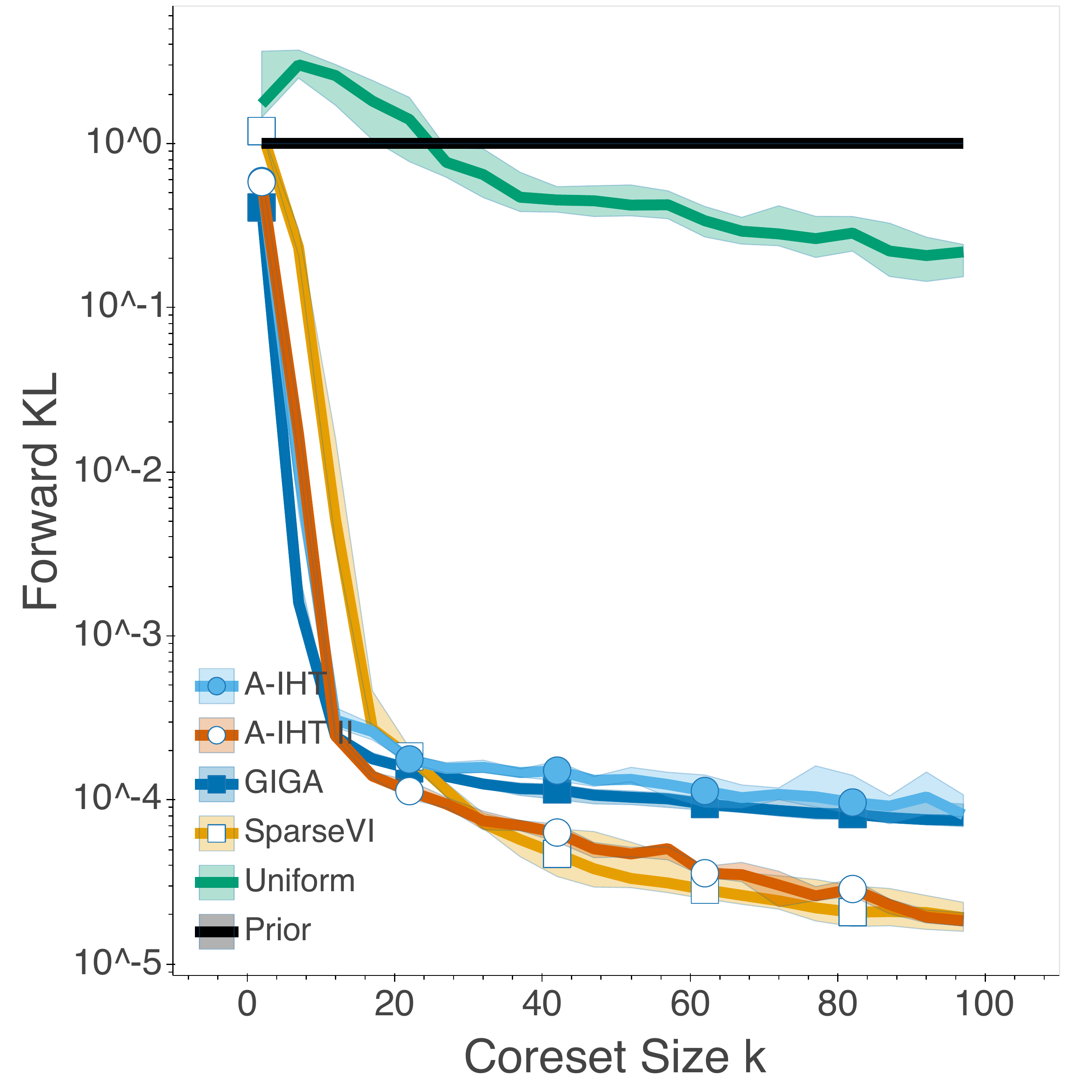}
		\includegraphics[width=0.32\linewidth]{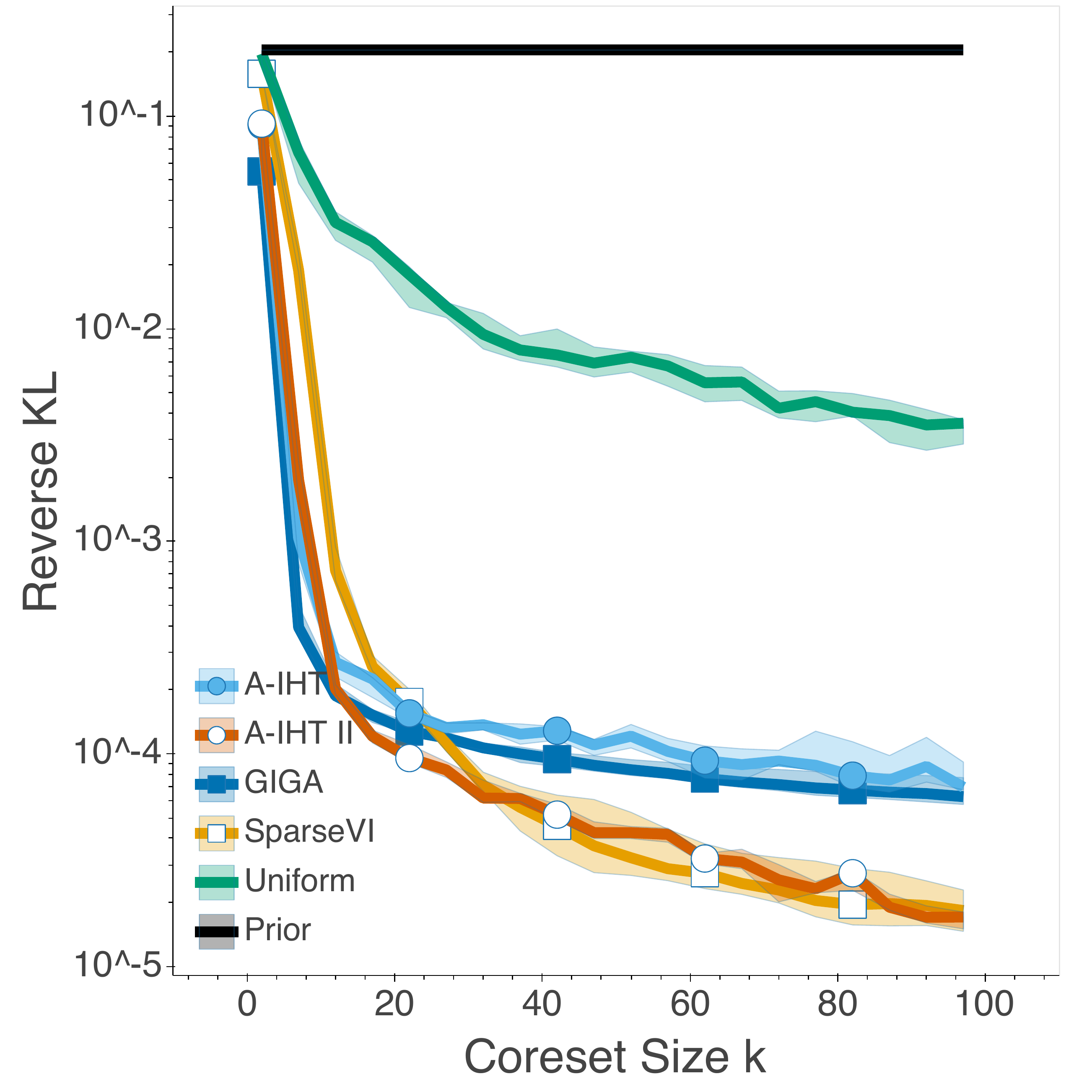}
		\includegraphics[width=0.32\linewidth]{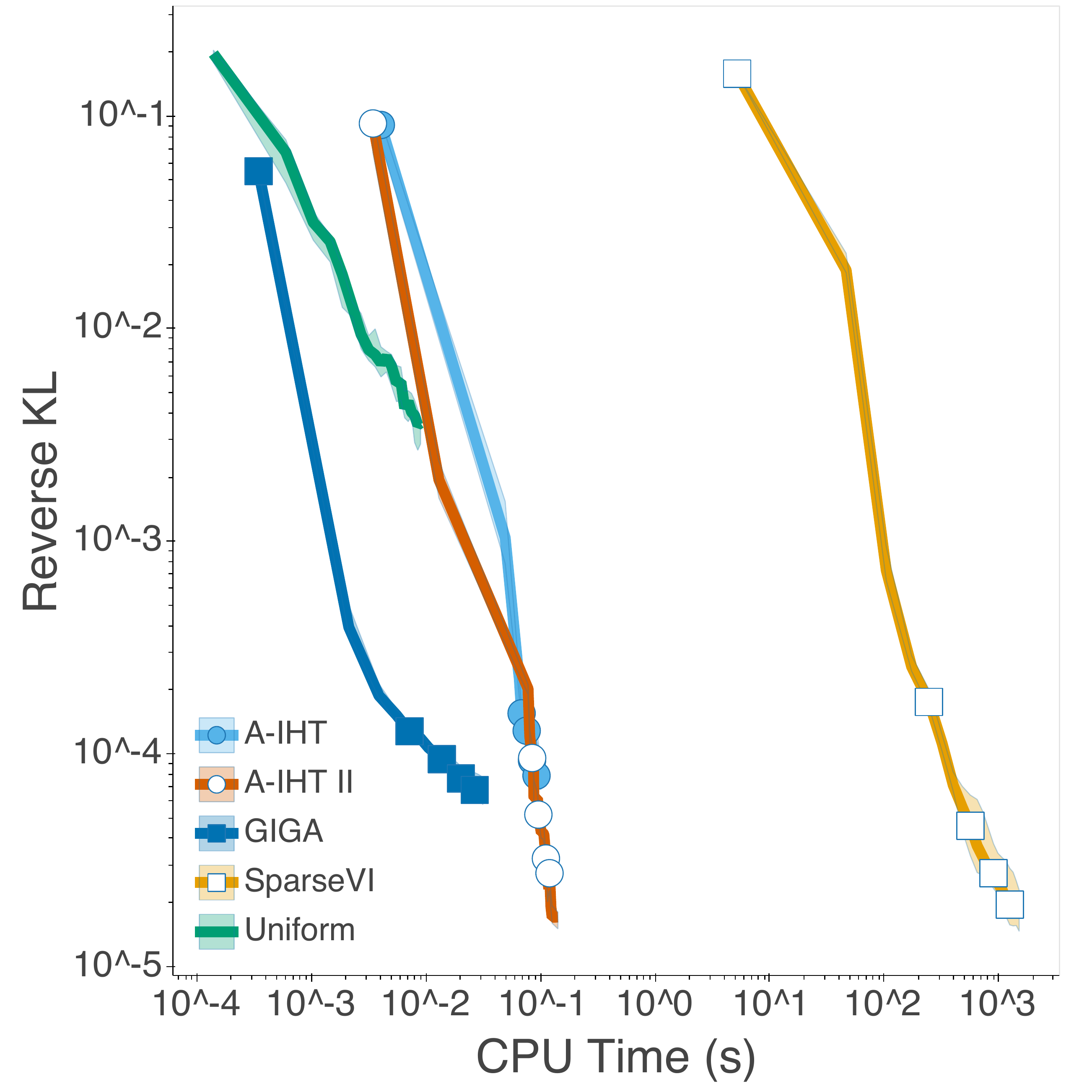}
		\vspace{-0.5em}
		{\small (b) \texttt{biketrips dataset} for Poisson regression}
	\end{minipage}
	\begin{minipage}[c]{0.9\linewidth}\centering
		\vspace{1em}
		\includegraphics[width=0.32\linewidth]{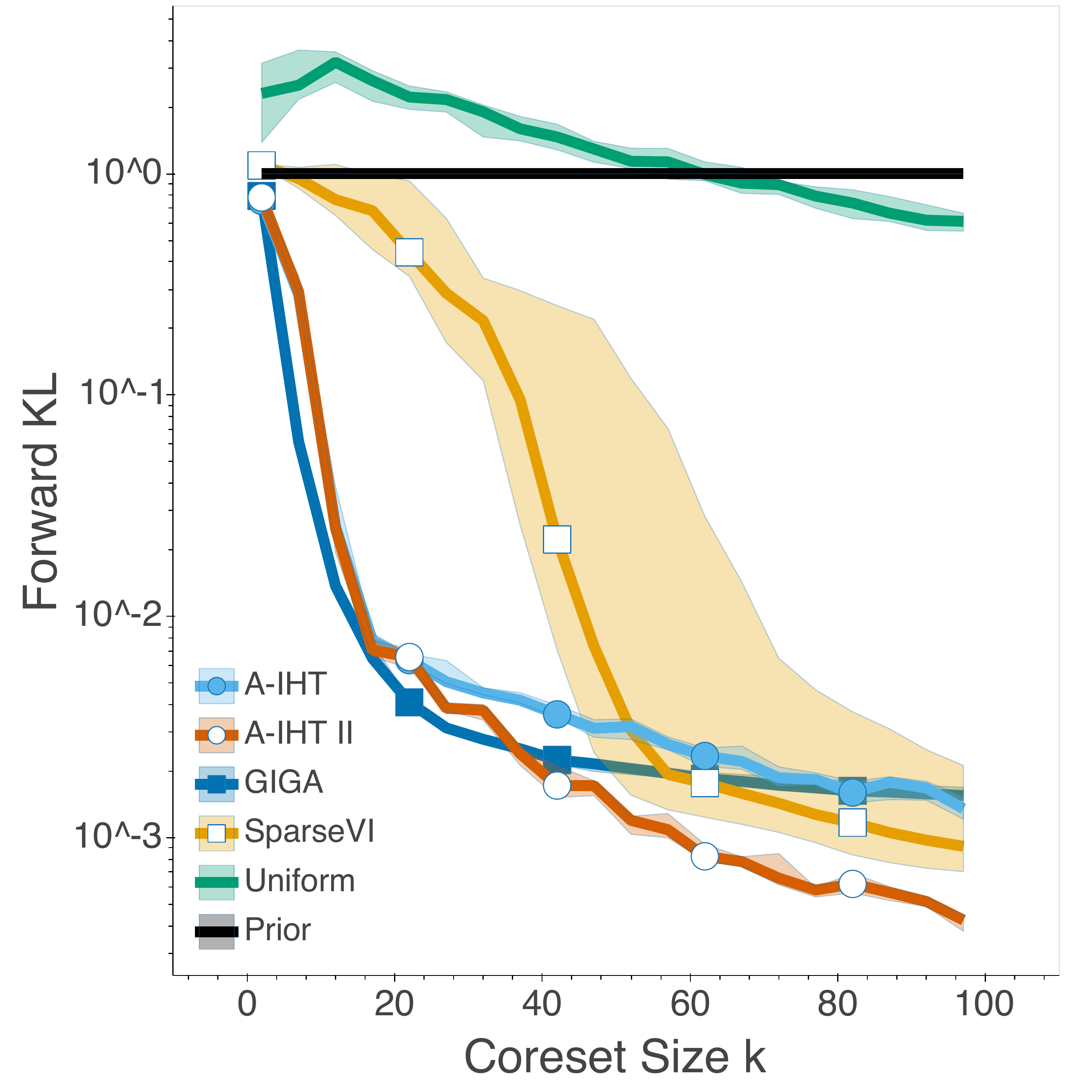}
		\includegraphics[width=0.32\linewidth]{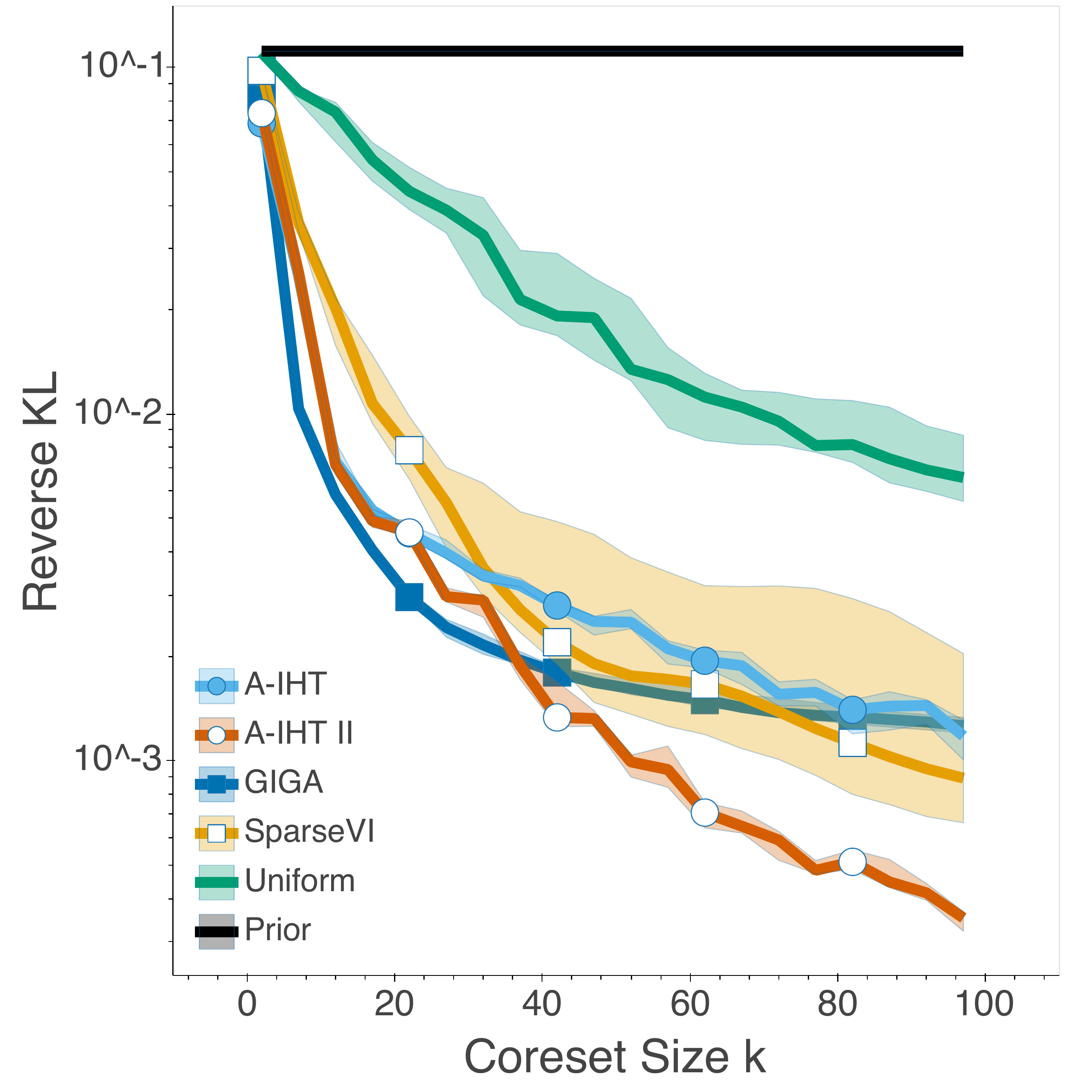}
		\includegraphics[width=0.32\linewidth]{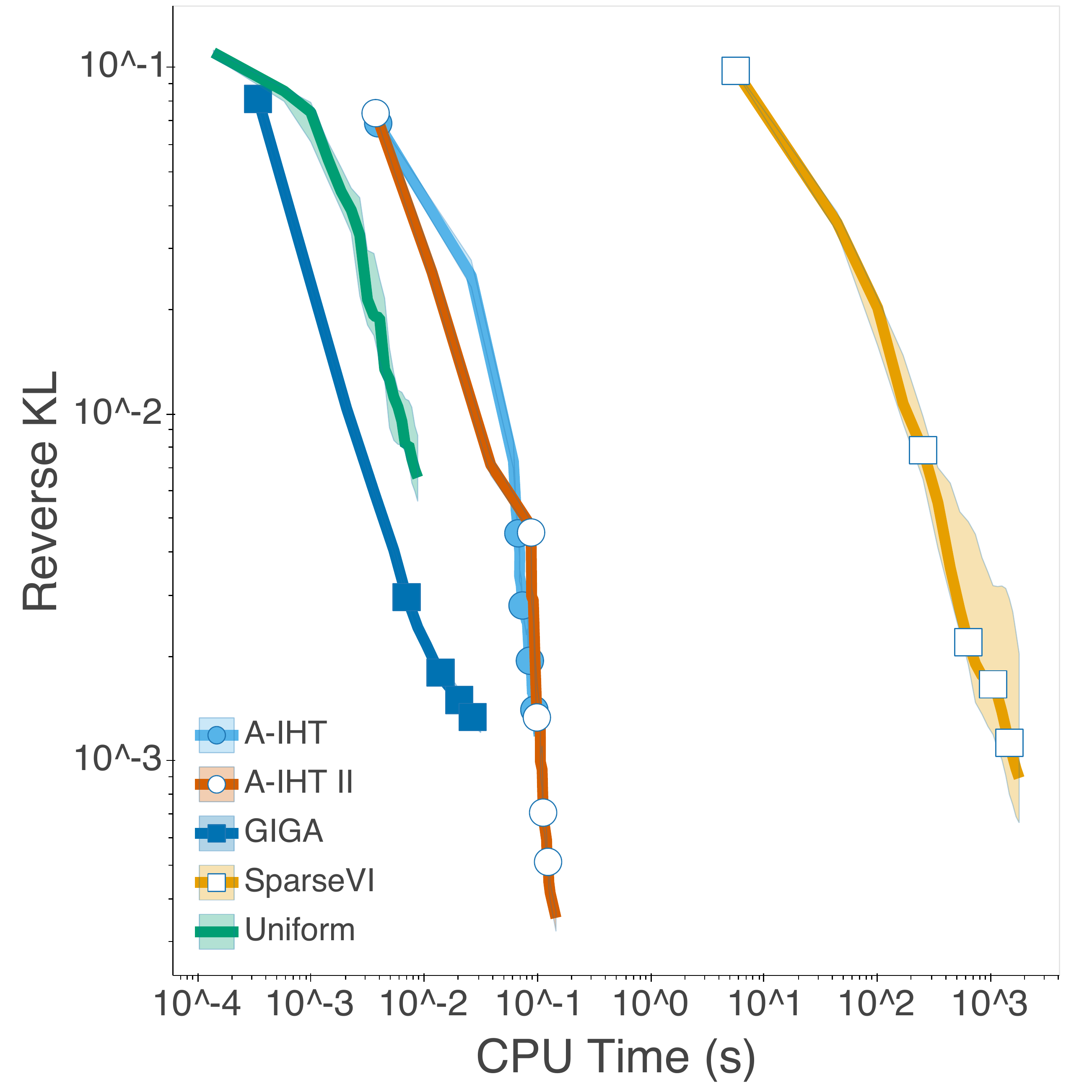}
		%\vspace{-4em}
		{\small (c) \texttt{airportdelays dataset} for Poisson regression}
	\end{minipage}
	\caption{Bayesian coreset construction for Poisson regression (PR) using the three different datasets. All the algorithms are run 20 times, and the median as well as the interval of $35^{th}$ and $65^{th}$ percentile, indicated as the shaded area, are reported. Different maximal coreset size $k$ is tested from $1$ to $100$. Forward KL (left) and reverse KL (middle) divergence between estimated true posterior and coreset posterior indicate the quality of the constructed coreset. The smaller the KL divergence, the better the coreset is. The running time for each algorithms is also recorded (right).} \label{fig:exp-3-pr}
\end{figure}

We consider how IHT performs when used in real applications where the closed-form expressions are unattainable. As the true posterior is unknown, a Laplace approximation is used for GIGA and IHT to derive the finite projection of the distribution, \emph{i.e.}, $\hat{g}_i$. Further, Monte Carlo sampling is needed to derive gradients of $D_{\text{KL}}$ for SparseVI. We compare different algorithms estimating the posterior distribution for logistic regression and Poisson regression. The reverse KL and forward KL between the coreset posterior and true posterior are estimated using another Laplace approximation. The experiment was proposed by \cite{campbell2019automated}, and is used in \citep{campbell2018bayesian} (GIGA) and \citep{campbell2019sparse} (SparseVI). The experimental settings for each baseline algorithms are set following their original settings for this experiment. In addition, we conduct additional experiments using a stochastic gradient estimator or using an alternative evaluation for coreset quality.

For logistic regression, given a dataset $\{(x_n, y_n)\in \mathbb{R}^D \times \{1, -1\} \mid i\in [N]\}$, we aim to infer $\theta \in \mathbb{R}^{D+1}$ based on the model:
\begin{align}
    y_n \mid x_n, \theta \sim \text{Bern}\left(\frac{1}{1+e^{-z_n^\top \theta}}\right),
\end{align}
where $z_n = [x_n^\top, 1]^\top$. Three datasets are used for logistic regression. The \texttt{synthetic dataset for logistic regression} consists of data $x_n$ sampled i.i.d. from standard normal distribution $\gN(0,I)$, and label $y_n$ sampled from Bernoulli distribution conditioned on $x_n$ and $\theta=[3,3,0]^\top$. The original \texttt{phishing} dataset\footnote{\url{https://www.csie.ntu.edu.tw/~cjlin/libsvmtools/datasets/binary.html}} consists of $N=11055$ data points with dimension $D=68$. The \texttt{phishing} dataset used in this experiment is preprocessed~\citep{campbell2019sparse} via principle component analysis to project each data points to dimension of $D=10$ to mitigate high computation by SparseVI. The original \texttt{chemical reactivities} dataset\footnote{\url{http://komarix.org/ac/ds}} has $N=26733$ data points with dimension $D=10$. We uniformly sub-sample $N=500$ data points from each datasets for this experiment, due to the high computation cost of SparseVI.

For Poisson regression, given $\{(x_n, y_n)\in \mathbb{R}^D \times \mathbb{N} \mid i\in [N]\}$, we aim to infer $\theta \in \mathbb{R}^{D+1}$ from model
\begin{align}
    y_n \mid x_n, \theta \sim \text{Poiss}\left(\log \left( 1+e^{-z_n^\top \theta}\right)\right),
\end{align}
where $z_n = [x_n^\top, 1]^\top$. Three other datasets are used for Poisson regression: the \texttt{synthetic dataset for Poisson regression} consists of data $x_n$ sampled i.i.d. from a standard normal distribution $\gN(0,1)$, and target $y_n$ sampled from Poisson distribution conditioned on $x_n$ and $\theta=[1,0]^\top$. The \texttt{biketrips} dataset\footnote{\url{http://archive.ics.uci.edu/ml/datasets/Bike+Sharing+Dataset}} consists of $N=17386$ data points with dimension $D=8$. The \texttt{airportdelays} dataset\footnote{The \texttt{airportdelays} dataset was constructed~\citep{campbell2019automated} by combining flight delay data (\url{http://stat-computing.org/dataexpo/2009/the-data.html}) and weather data (\url{https://www.wunderground.com/history/.}).} has $N=7580$ data points with dimension $D=15$. Same as logistic regression, we uniformly sub-sample $N=500$ data points from each datasets for this experiment.

The comparison of the algorithms for Bayesian coreset construction for logistic regression are shown in Figure~\ref{fig:exp-3-lr},  and Bayesian coreset construction for Poisson regression are shown in Figure~\ref{fig:exp-3-pr}. The left column shows forward KL divergence given sparsity setting $k$, the middle column shows reverse KL divergence, and the right column presents the running time for corset construction for each algorithm.  

It is observed that A-IHT and A-IHT \RomanNumeralCaps{2} achieve state-of-the-art performance. The IHT algorithms often obtain coresets with smaller KL than GIGA and SparseVI, with computing time comparable to GIGA, significantly less than SparseVI. 
The experiments indicate that IHT outperforms the previous methods, improving the trade-off between accuracy and performance.

The results on large-scale datasets have been presented in the Figure~\ref{fig:exp-3-large} in the main paper. Next, we present two additional sets of experiments that are omitted in the main paper.

\textbf{Stochastic Gradient Estimator.} For large-scale datasets, it is often necessary to ``batch" the algorithms. IHT can be easily batched by replacing the gradient with a stochastic gradient estimator that only a batch of data in each iteration. 

Recall that for IHT the gradient of the objective function $f(w)=\|y-\Phi w\|^2$ is $\nabla f(w)=2\Phi^\top (\Phi w - y)$, where $\Phi\in \sR^{S\times n}$. As we introduced in section~\ref{sec:background}, $S$ is the number of samples $\theta\sim \hat \pi$, and $n$ is the number of data. Thus, we can form a unbiased gradient estimator as 
\begin{align}
    \Tilde{g}(w) = 2\Gamma_1^\top\Phi^\top (\Phi \Gamma_2 w - y),
\end{align}
where $\Gamma_1, \Gamma_2\in \sR^{n\times n}$ are \emph{i.i.d} sampled from a distribution $\pi_\Gamma$ with $\E_{\Gamma_1 \sim \pi_\Gamma} [\Gamma_1] = \E_{\Gamma_2 \sim \pi_\Gamma} [\Gamma_2]=I$, where $I\in \sR^{n\times n}$ is the identity matrix. Therefore,
\begin{align}
     \E\Tilde{g}(w) = 2(\E\Gamma_1)^\top\Phi^\top (\Phi \E\Gamma_2 w - y)= 2I\Phi^\top (\Phi I w - y) = \nabla f(w),
\end{align}
showing that $\Tilde{g}(w)$ is an unbiased estimator of $\nabla f(w)$. 

For example, we can form the estimator using a batch of data with batch size $B$ by letting $\Gamma_1, \Gamma_2$ be random matrices as randomly setting $n-B$ rows of $\tfrac{n}{B}I$ be zero. Equivalently, it is the same as randomly picking $B$ columns of $\Phi$, setting the rest columns be zero, and scale the matrix by $n/B$. Noting that each column of $\Phi$ corresponds to each of the $n$ data points, this operation is essentially to approximate $\Phi$ using a batch of data with batch size $B$, and thus it approximates the gradient using a batch of a data.

We test how Algorithm~\ref{alg:a-iht} performs on the Bayesian logistic regression and Poisson regression using the stochastic estimator with batch size $B=n/5$. All of the experimental settings are the same as what we have introduced in this section. As a summary of both forward FL and reverse KL, we use the symmetrized  KL (\emph{i.e.}, the sum of forward KL and reverse KL) as the evaluation metric for coreset quality. The results are shown in Figure~\ref{fig:exp-3-stoc}. It is observed that A-IHT with the stochastic gradient estimator (A-IHT batch grad.) performs comparably to the A-IHT. We note that the batched version of A-IHT can be improved by increasing its maximal number of iterations, \emph{i.e.}, optimization with stochastic gradient needs more iterations to converge, or using a better batch gradient estimator. Theoretical study on accelerated IHT with approximated gradients is still an open question to the best of our knowledge. Further research on accelerated IHT with stochastic gradients is an interesting future work.

\begin{figure}[t!]\centering
	%\vspace{-2em}
	\begin{minipage}[c]{0.9\linewidth}\centering
		%\vspace{-2em}
		\begin{minipage}[c]{0.32\linewidth}\centering
		    \includegraphics[width=\linewidth]{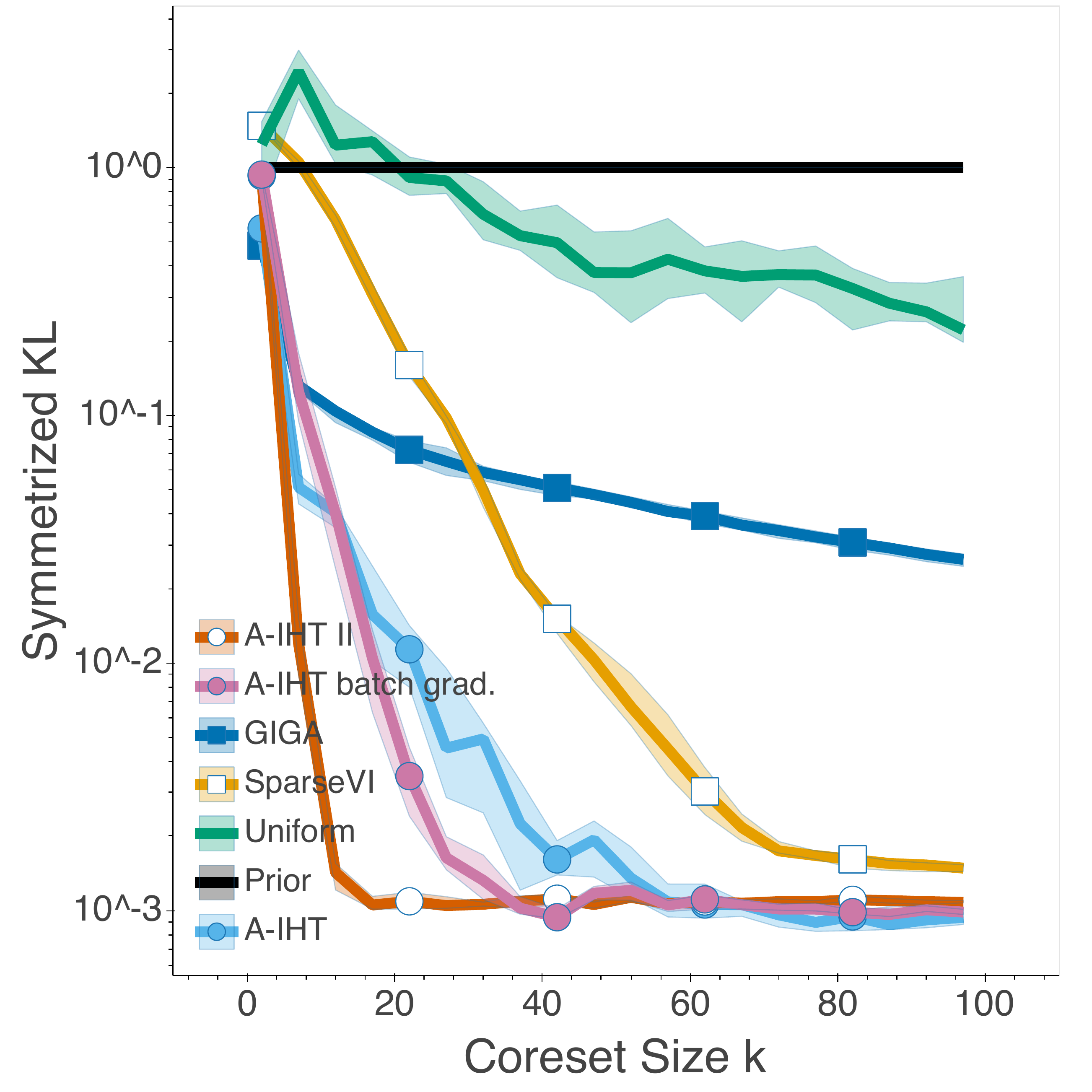}
		    {\small {\texttt{synthetic dataset for logistic regression} }}
		\end{minipage}
		\begin{minipage}[c]{0.32\linewidth}\centering
		    \includegraphics[width=\linewidth]{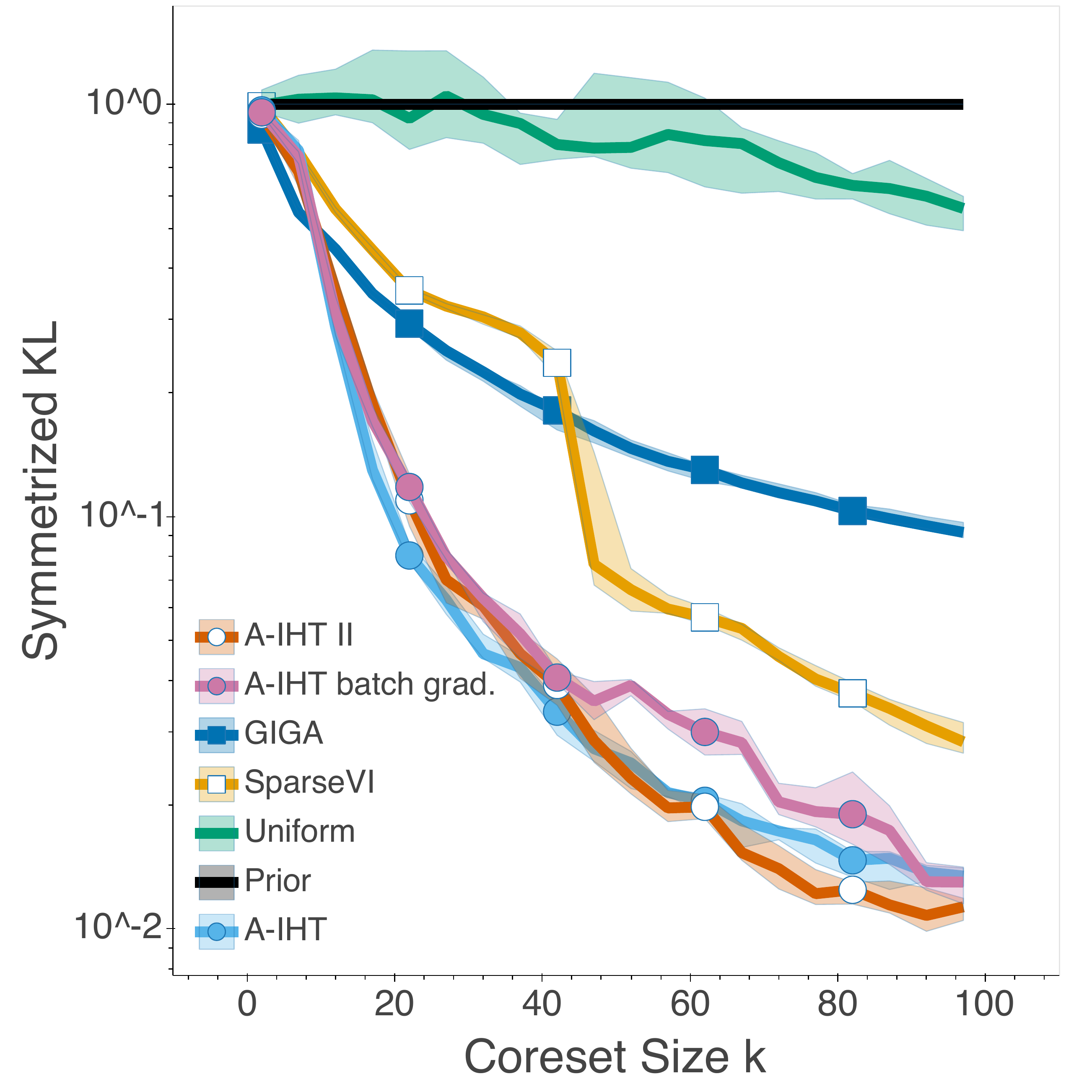}
		    {\small {\texttt{phishing dataset} for logistic regression }}
		\end{minipage}
		\begin{minipage}[c]{0.32\linewidth}\centering
		    \includegraphics[width=\linewidth]{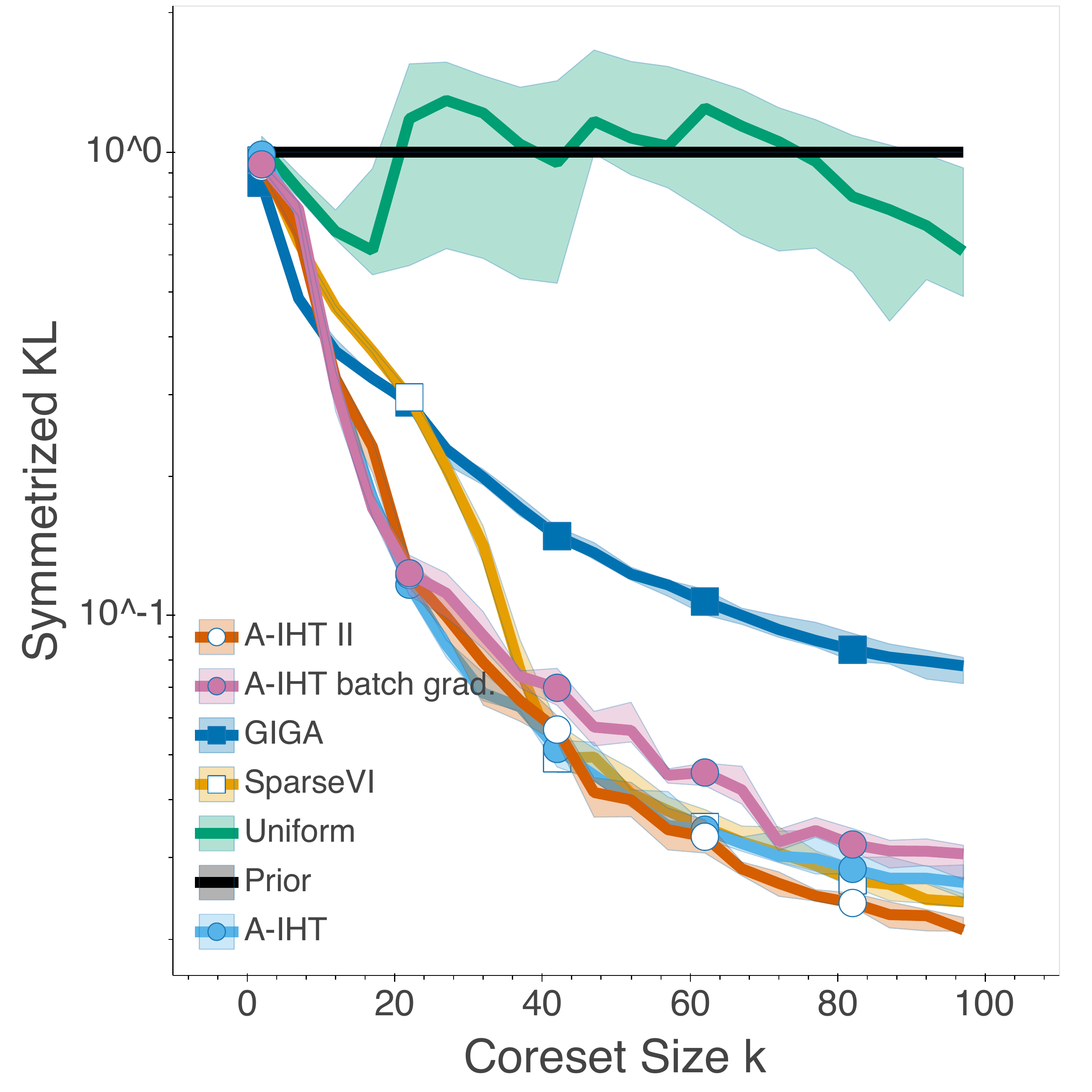}
		    {\small {\texttt{chemical reactivities dataset} for logistic regression}}
		\end{minipage}
		%\vspace{-0.5em}
		%{\small (a) {Logistic regression}}
	\end{minipage}
	
	\begin{minipage}[c]{0.9\linewidth}\centering
		\vspace{0.7em}
		\begin{minipage}[c]{0.32\linewidth}\centering
		    \includegraphics[width=\linewidth]{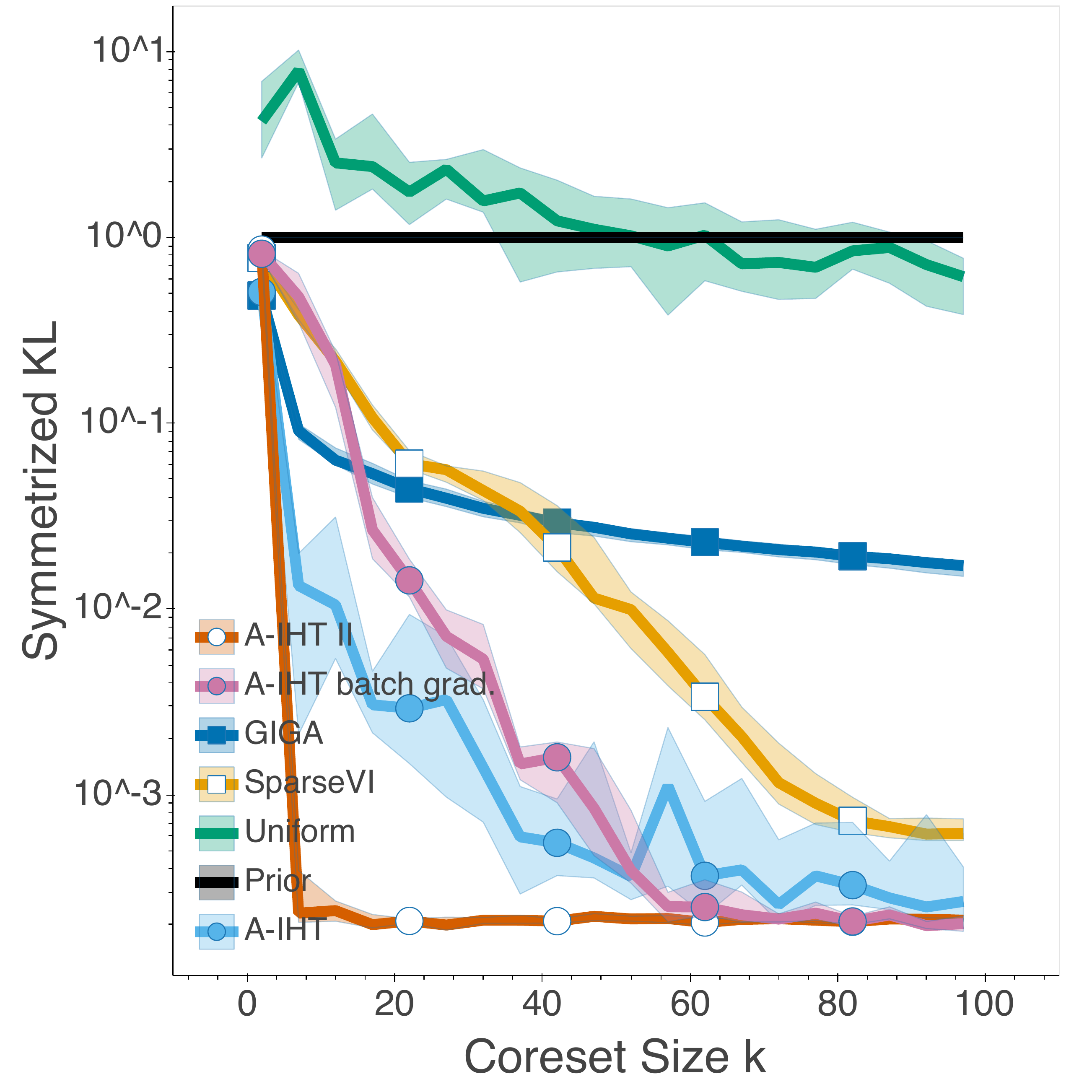}
		    {\small {\texttt{synthetic dataset for Poisson regression} }}
		\end{minipage}
		\begin{minipage}[c]{0.32\linewidth}\centering
		\vspace{0.7em}
		    \includegraphics[width=\linewidth]{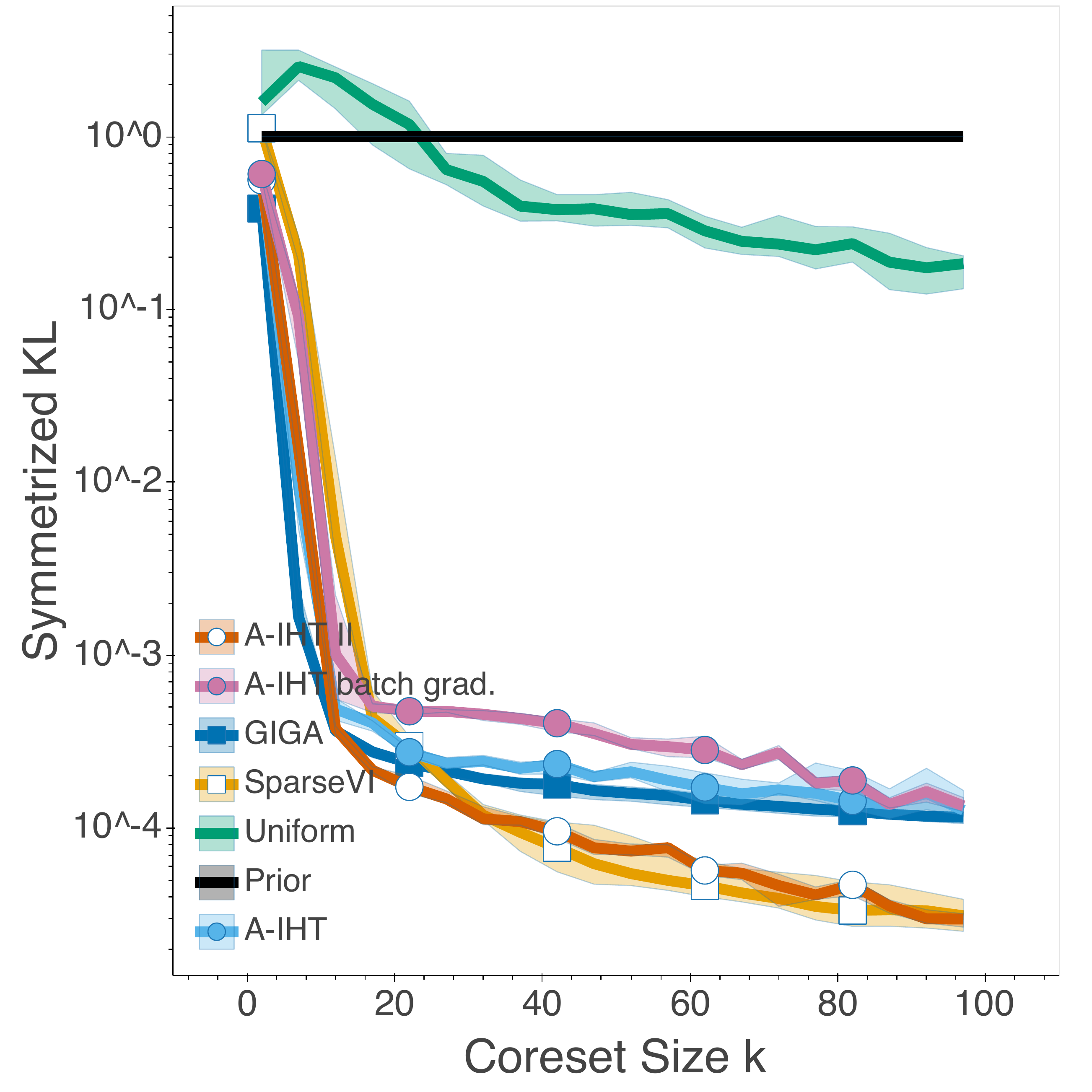}
		    {\small {\texttt{biketrips dataset} for Poisson regression }}
		\end{minipage}
		\begin{minipage}[c]{0.32\linewidth}\centering
		\vspace{0.7em}
		    \includegraphics[width=\linewidth]{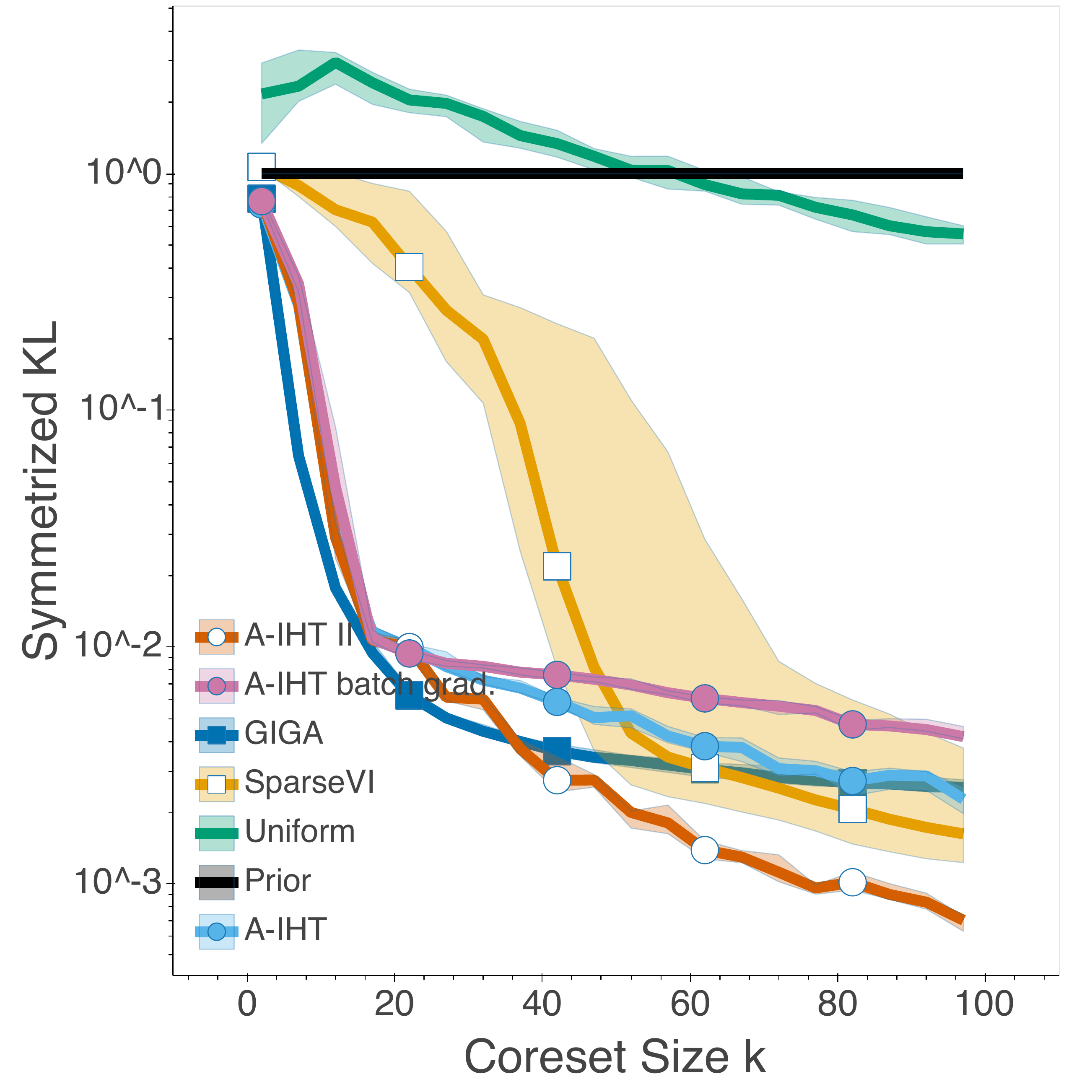}
		    {\small {\texttt{airportdelays dataset} for Poisson regression}}
		\end{minipage}
		%\vspace{-0.5em}
		%{\small (b) {Poisson regression}}
	\end{minipage}
	\caption{Bayesian coreset construction for logistic regression and  Poisson regression using the six different datasets. All the algorithms are run 20 times, and the median as well as the interval of $35^{th}$ and $65^{th}$ percentile, indicated as the shaded area, are reported. Different maximal coreset size $k$ is tested from $1$ to $100$. Symmetrized  KL divergence between estimated true posterior and coreset posterior indicate the quality of the constructed coreset. The smaller the KL divergence, the better the coreset is.} \label{fig:exp-3-stoc}
\end{figure}

\textbf{$\ell_2$-distance Evaluation of Coreset Quality.} In the previous experiments in the subsection, the coreset quality is evaluated by approximating the KL divergence between the full-dataset posterior and coreset posterior. As an alternative way to measure the coreset quality, we measure the $\ell_2$-distance between the maximum-a-posteriori (MAP) estimation of the full-dataset posterior and coreset posterior. The results are shown in Figure~\ref{fig:exp-3-distance}. It is observed that the two IHT algorithms usually achieve the best results, except that SparseVI achieves the lowest $\ell_2$-distance on two datasets. However, SparseVI costs $\times10^4$ more time than IHT and GIGA.

\begin{figure}[t!]\centering
	%\vspace{-2em}
	\begin{minipage}[c]{0.9\linewidth}\centering
		%\vspace{-2em}
		\begin{minipage}[c]{0.32\linewidth}\centering
		    \includegraphics[width=\linewidth]{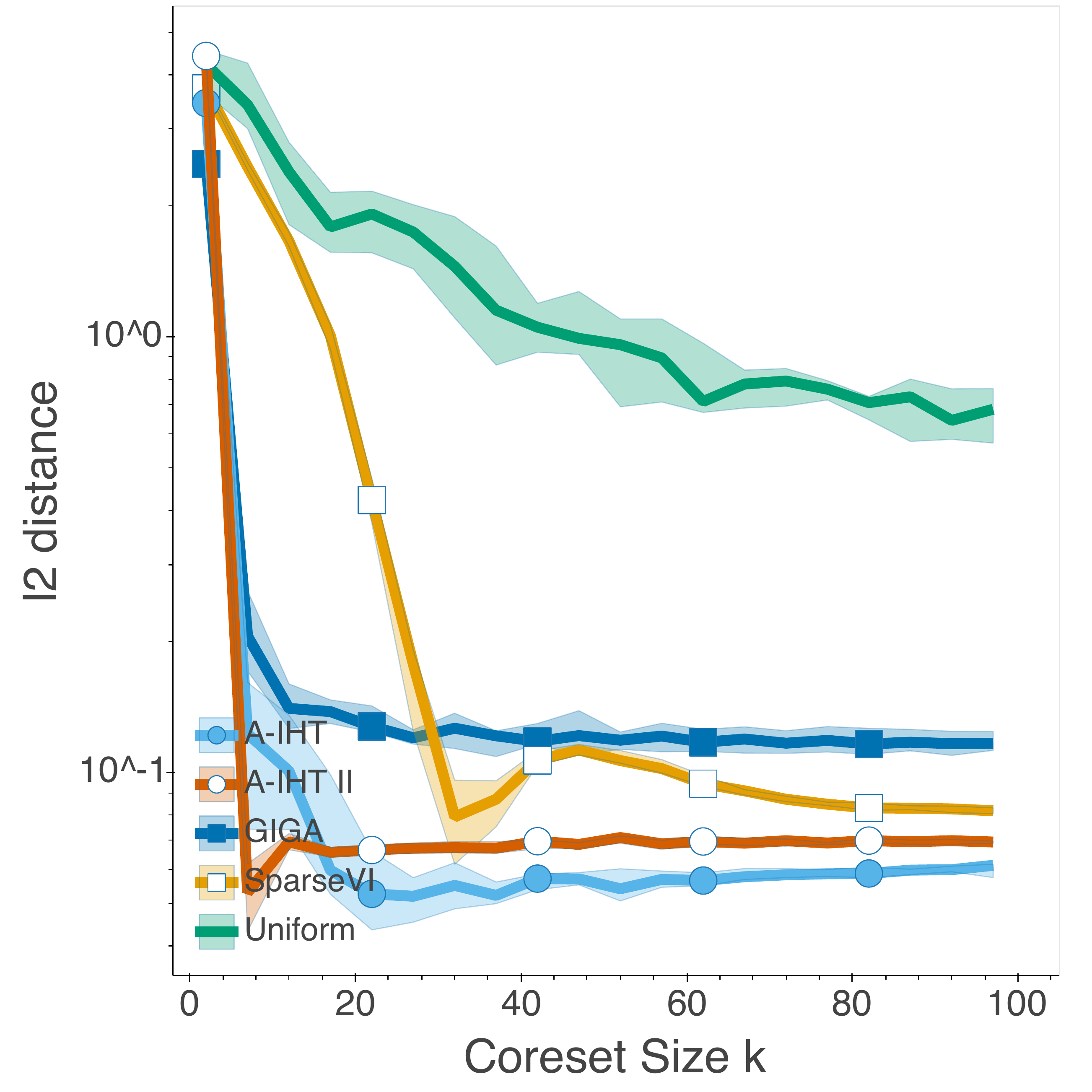}
		    {\small {\texttt{synthetic dataset for logistic regression} }}
		\end{minipage}
		\begin{minipage}[c]{0.32\linewidth}\centering
		    \includegraphics[width=\linewidth]{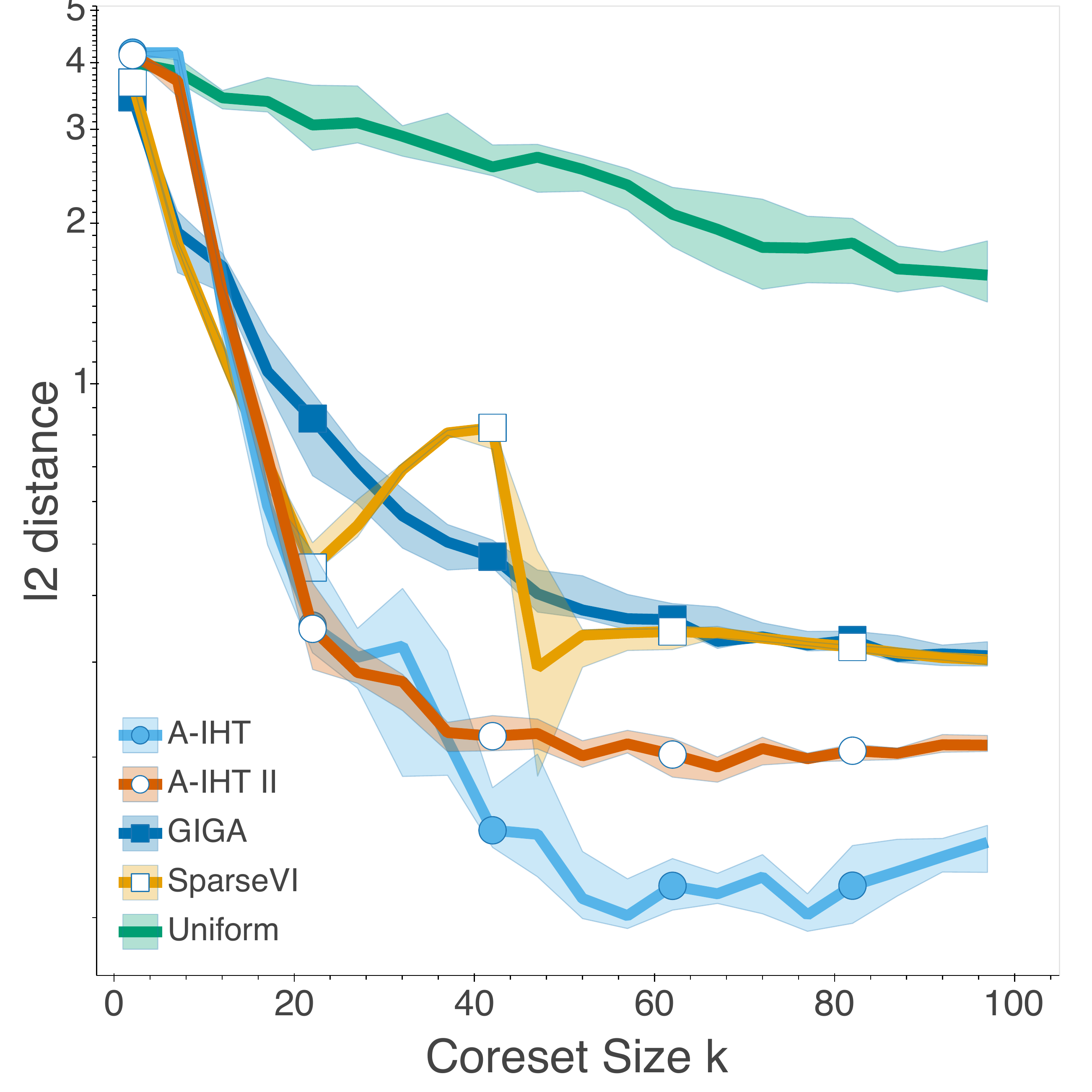}
		    {\small {\texttt{phishing dataset} for logistic regression }}
		\end{minipage}
		\begin{minipage}[c]{0.32\linewidth}\centering
		    \includegraphics[width=\linewidth]{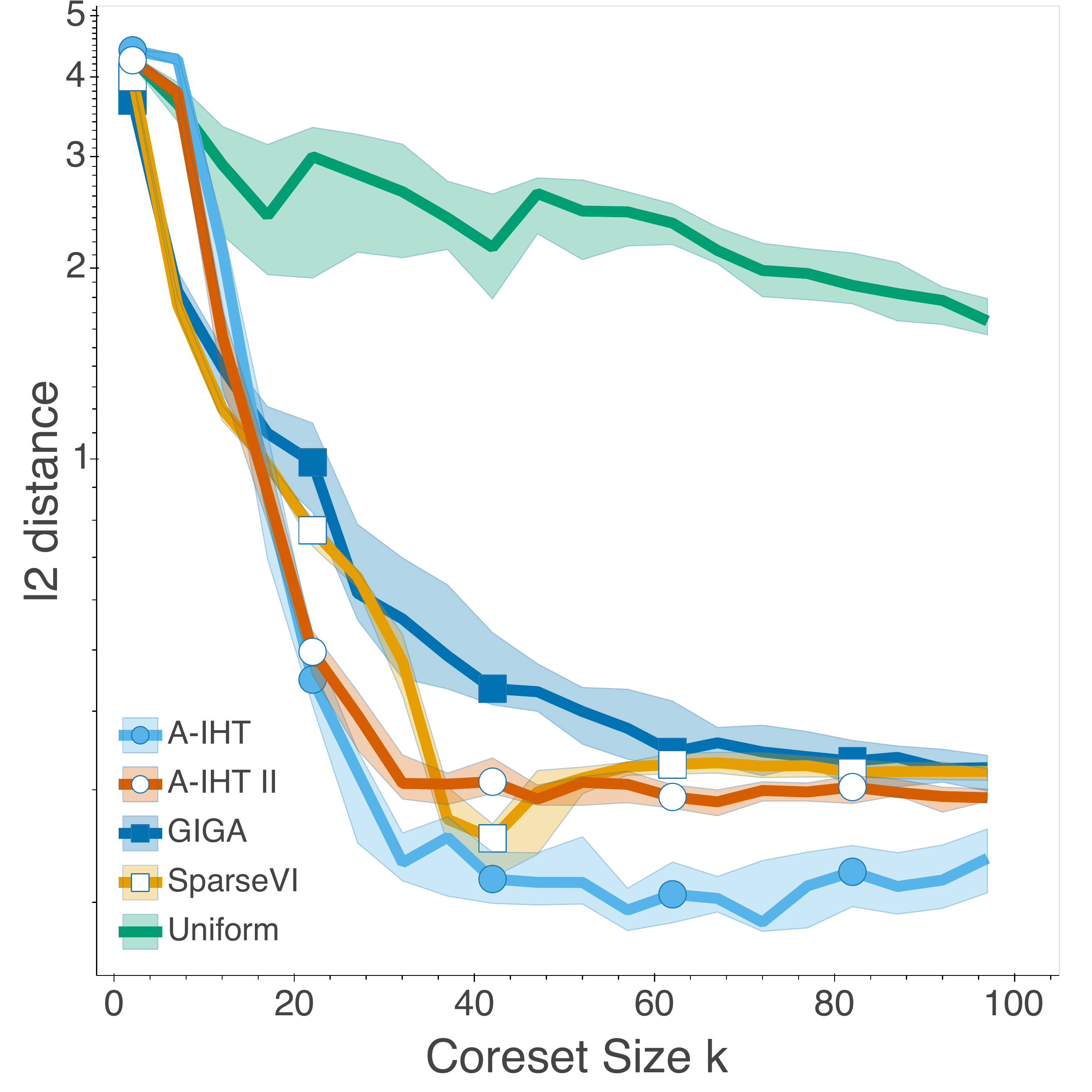}
		    {\small {\texttt{chemical reactivities dataset} for logistic regression}}
		\end{minipage}
		%\vspace{-0.5em}
		%{\small (a) {Logistic regression}}
	\end{minipage}
	
	\begin{minipage}[c]{0.9\linewidth}\centering
		\vspace{0.7em}
		\begin{minipage}[c]{0.32\linewidth}\centering
		    \includegraphics[width=\linewidth]{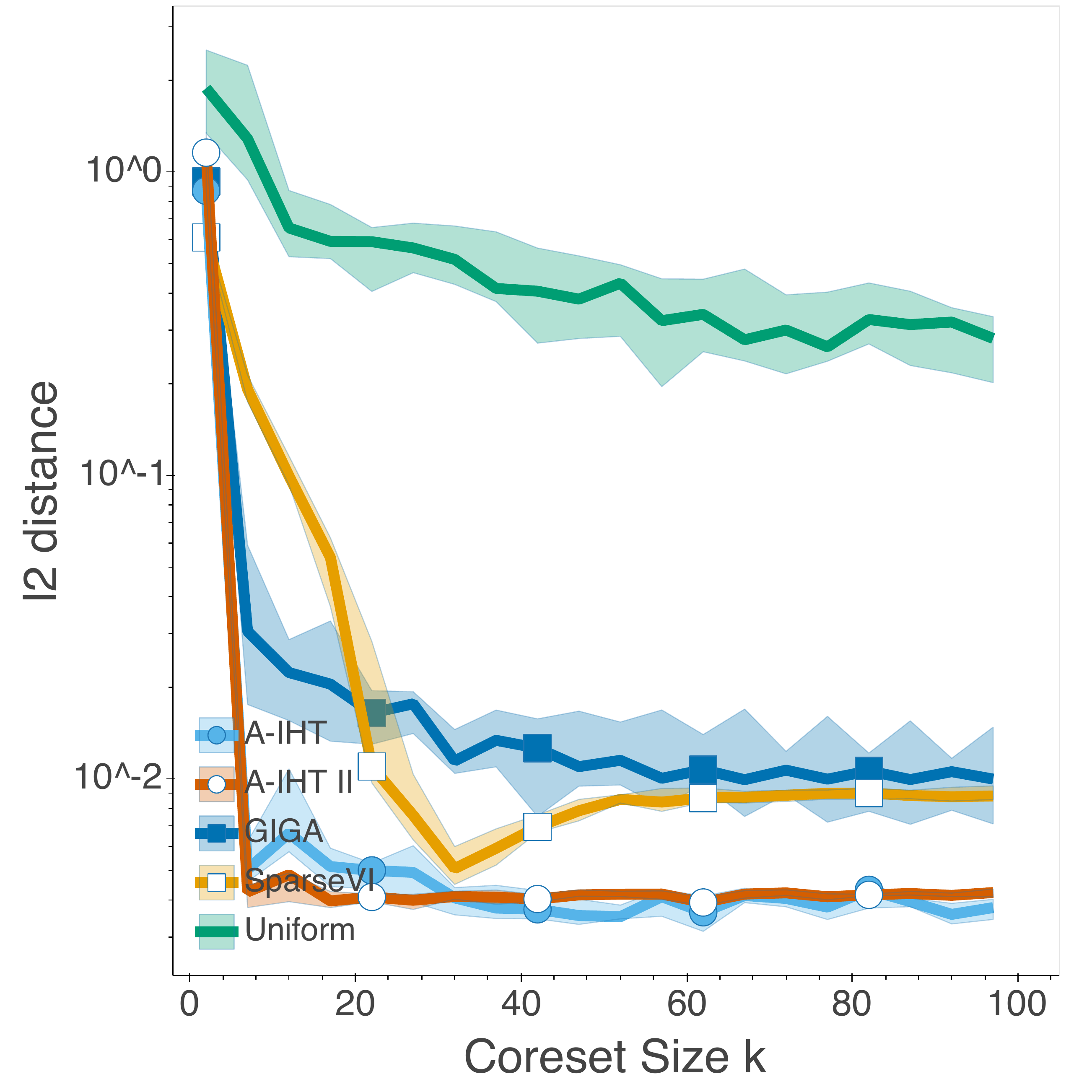}
		    {\small {\texttt{synthetic dataset for Poisson regression} }}
		\end{minipage}
		\begin{minipage}[c]{0.32\linewidth}\centering
		\vspace{0.7em}
		    \includegraphics[width=\linewidth]{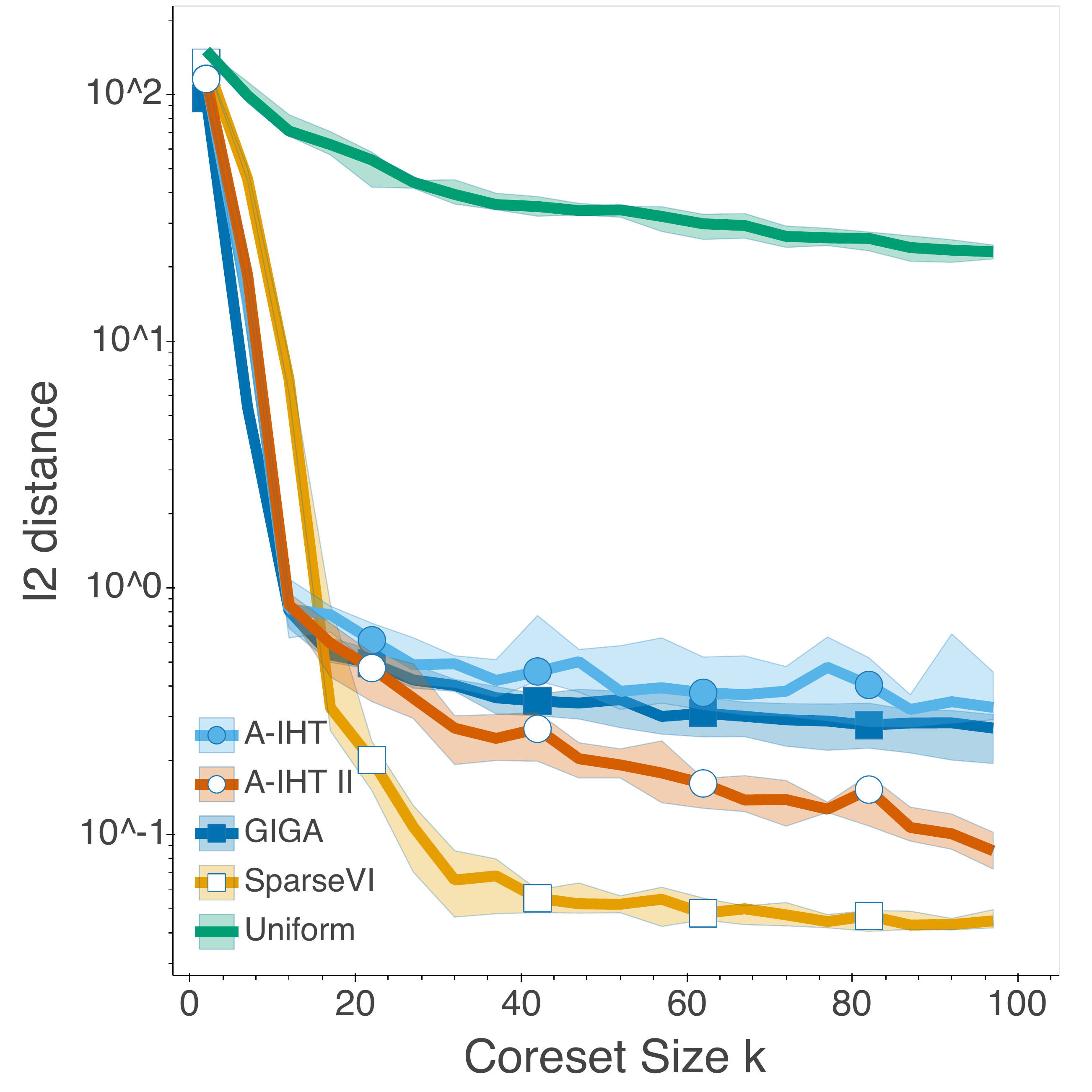}
		    {\small {\texttt{biketrips dataset} for Poisson regression }}
		\end{minipage}
		\begin{minipage}[c]{0.32\linewidth}\centering
		\vspace{0.7em}
		    \includegraphics[width=\linewidth]{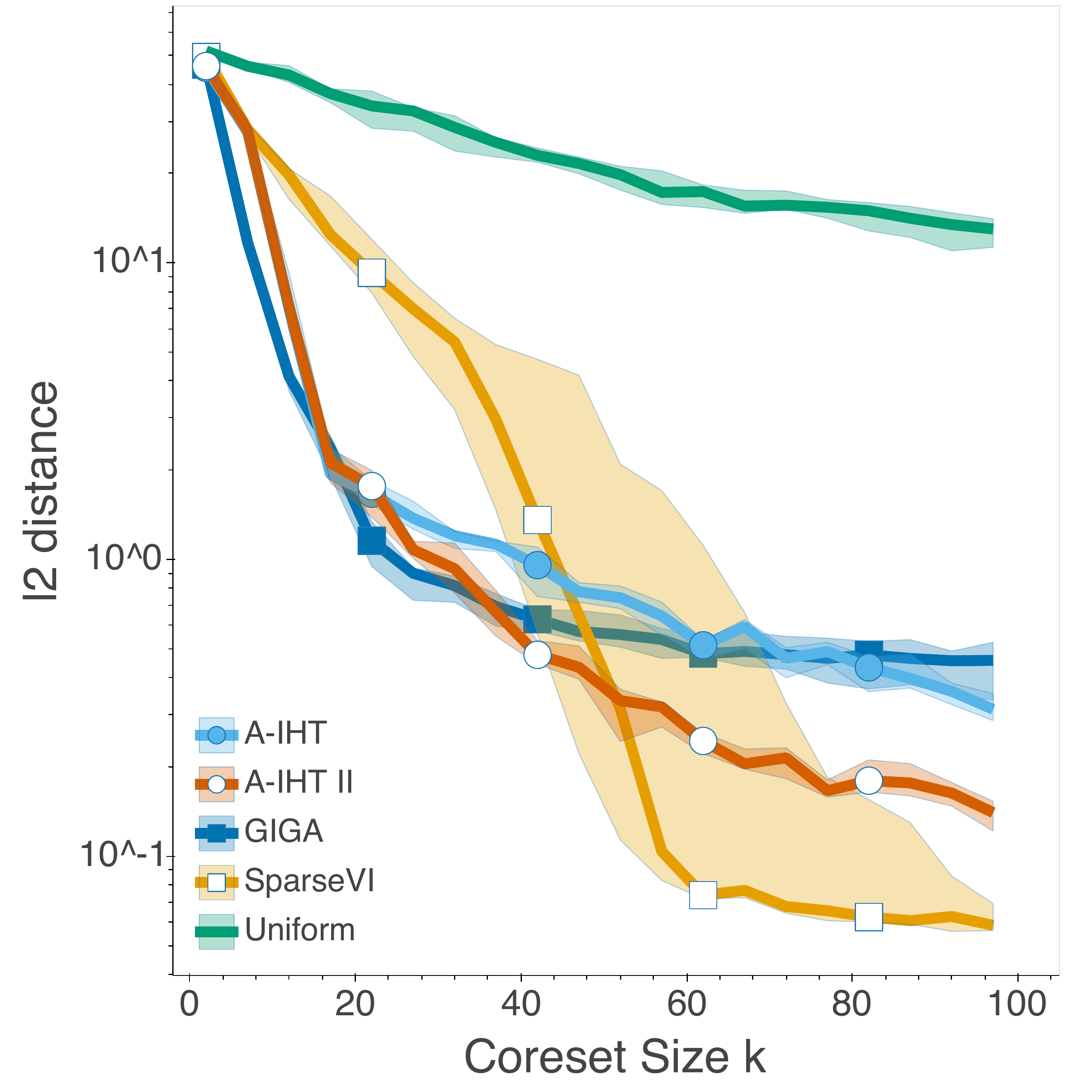}
		    {\small {\texttt{airportdelays dataset} for Poisson regression}}
		\end{minipage}
		%\vspace{-0.5em}
		%{\small (b) {Poisson regression}}
	\end{minipage}
	\caption{Bayesian coreset construction for logistic regression and  Poisson regression using the six different datasets. All the algorithms are run 20 times, and the median as well as the interval of $35^{th}$ and $65^{th}$ percentile, indicated as the shaded area, are reported. Different maximal coreset size $k$ is tested from $1$ to $100$. $\ell_2$-distance between the MAP estimators of the full-dataset posterior and coreset posterior indicate the quality of the constructed coreset. The smaller the $\ell_2$-distance, the better the coreset is.} \label{fig:exp-3-distance}
\end{figure}

\end{document}